\pdfoutput=1
\documentclass[11pt, letterpaper]{article}

\usepackage[english]{babel}
\usepackage[T1]{fontenc}
\usepackage{parskip}
\usepackage[margin=1in]{geometry}
\usepackage{bbm,bm}
\usepackage{bbold}
\usepackage{parskip}
\usepackage{amsmath}
\usepackage{amssymb}
\usepackage{amsthm}
\usepackage{thmtools, thm-restate}
\usepackage{booktabs}
\usepackage{array}
\usepackage[lined,boxed,ruled,norelsize,linesnumbered]{algorithm2e}
\usepackage{graphicx} 
\usepackage{enumitem}
\usepackage{float}

\usepackage{hyperref}
\hypersetup{
	colorlinks=true,
	linkcolor=blue!70!black,
	citecolor=blue!70!black,
	urlcolor=blue!70!black
}
\usepackage{graphicx}
\usepackage[english]{babel}
\usepackage[utf8x]{inputenc}
\usepackage[T1]{fontenc}
\usepackage{mathtools}
\usepackage{cleveref}
\usepackage{parskip}
\usepackage{xcolor}
\newtheorem{theorem}{Theorem}
\newtheorem{lemma}{Lemma}
\newtheorem{fact}{Fact}

\newtheorem{corollary}{Corollary}
\newtheorem{proposition}{Proposition}
\newtheorem{model}{Model}
\newtheorem{assumption}{Assumption}
\newtheorem{problem}{Problem}

\newtheorem{remark}{Remark}

\newcommand{\defeq}{:=}

\newcommand{\norm}[1]{\left\lVert#1\right\rVert}
\newcommand{\norms}[1]{\lVert#1\rVert}
\newcommand{\normop}[1]{\left\lVert#1\right\rVert_{\textup{op}}}
\newcommand{\normf}[1]{\left\lVert#1\right\rVert_{\textup{F}}}

\newcommand{\normsop}[1]{\lVert#1\rVert_{\textup{op}}}
\newcommand{\normsf}[1]{\lVert#1\rVert_{\textup{F}}}

\newcommand{\inprod}[2]{\left\langle#1, #2\right\rangle}

\newcommand{\eps}{\epsilon}

\newcommand{\R}{\mathbb{R}}

\newcommand{\N}{\mathbb{N}}
\newcommand{\diag}[1]{\textbf{\textup{diag}}\left(#1\right)}
\newcommand{\half}{\frac{1}{2}}

\newcommand{\E}{\mathbb{E}}

\newcommand{\Var}{\textup{Var}}

\newcommand{\Tr}{\textup{Tr}}

\newcommand{\id}{\mathbf{I}}

\newcommand{\dd}{\textup{d}}

\newcommand{\Par}[1]{\left(#1\right)}
\newcommand{\Brack}[1]{\left[#1\right]}
\newcommand{\Brace}[1]{\left\{#1\right\}}
\newcommand{\Abs}[1]{\left|#1\right|}

\newcommand{\ind}{\mathbb{I}}
\newcommand{\mzero}{\mathbf{0}}
\newcommand{\Sym}{\mathbb{S}}
\newcommand{\PSD}{\Sym_{\succeq \mzero}}

\newcommand{\range}{\mathrm{range}}

\newcommand{\vth}{\boldsymbol{\theta}}

\definecolor{burntorange}{rgb}{0.8, 0.33, 0.0}

\newcommand{\twonorm}[1]{\norm{#1}_2}

\newcommand{\msig}{\mathbf{\Sigma}}
\newcommand{\vmu}{\boldsymbol{\mu}}
\newcommand{\vlam}{\boldsymbol{\lambda}}
\newcommand{\vzero}{\mathbf{0}}

\newcommand{\TV}{\mathrm{TV}}
\newcommand{\KL}{\mathrm{KL}}

\newcommand{\simiid}{\sim_{\textup{i.i.d.}}}

\newcommand{\Prob}{\mathbb{P}}
\renewcommand{\Pr}{\Prob}
\newcommand{\prob}[1]{\mathbb{P}\left[#1\right]}

\newcommand{\0}{\mathbf{0}}

\newcommand{\ma}{{\mathbf{A}}}
\newcommand{\mb}{{\mathbf{B}}}
\newcommand{\mc}{{\mathbf{C}}}
\newcommand{\md}{{\mathbf{D}}}
\newcommand{\me}{{\mathbf{E}}}
\newcommand{\mf}{{\mathbf{F}}}
\newcommand{\mg}{{\mathbf{G}}}
\newcommand{\mh}{{\mathbf{H}}}

\newcommand{\mk}{{\mathbf{K}}}

\newcommand{\mm}{{\mathbf{M}}}
\newcommand{\mn}{{\mathbf{N}}}

\newcommand{\mr}{{\mathbf{R}}}
\newcommand{\ms}{{\mathbf{S}}}
\newcommand{\mt}{{\mathbf{T}}}
\newcommand{\muu}{{\mathbf{U}}}
\newcommand{\mv}{{\mathbf{V}}}
\newcommand{\mw}{{\mathbf{W}}}
\newcommand{\mx}{{\mathbf{X}}}
\newcommand{\my}{{\mathbf{Y}}}
\newcommand{\mz}{{\mathbf{Z}}}

\newcommand{\mlam}{\mathbf{\Lambda}}

\newcommand{\va}{{\mathbf{a}}}

\newcommand{\vc}{{\mathbf{c}}}

\newcommand{\ve}{{\mathbf{e}}}

\newcommand{\vg}{{\mathbf{g}}}

\newcommand{\vn}{{\mathbf{n}}}

\newcommand{\vr}{{\mathbf{r}}}
\newcommand{\vs}{{\mathbf{s}}}
\newcommand{\vt}{{\mathbf{t}}}

\newcommand{\vv}{{\mathbf{v}}}
\newcommand{\vw}{{\mathbf{w}}}
\newcommand{\vx}{{\mathbf{x}}}
\newcommand{\vy}{{\mathbf{y}}}
\newcommand{\vz}{{\mathbf{z}}}

\newcommand{\calA}{{\mathcal{A}}}

\newcommand{\calE}{{\mathcal{E}}}

\newcommand{\calH}{{\mathcal{H}}}
\newcommand{\calI}{{\mathcal{I}}}

\newcommand{\calM}{{\mathcal{M}}}
\newcommand{\calN}{{\mathcal{N}}}

\newcommand{\calP}{{\mathcal{P}}}
\newcommand{\calQ}{{\mathcal{Q}}}

\newcommand{\calW}{{\mathcal{W}}}

\newcommand{\bbE}{{\mathbb{E}}}

\newcommand{\paren}[1]{\left(#1\right)}
\newcommand{\bra}[1]{\left[#1\right]}
\newcommand{\cbra}[1]{\left\{#1\right\}}

\crefname{assumption}{assumption}{assumptions}

\newcommand{\innp}[1]{\left \langle #1 \right\rangle}

\newcommand{\rank}{\mathrm{rank}}

\definecolor{burntorange}{rgb}{0.8, 0.33, 0.0}
\definecolor{skyblue}{HTML}{75aadb}

\newcommand{\BW}{\dd_{\mathrm{BW}}}
\newcommand{\FID}{\dd_{\mathrm{FID}}}
\newcommand{\VA}{FID-VALE}
\newcommand{\vX}{\boldsymbol{X}}
\newcommand{\vY}{\boldsymbol{Y}}
\newcommand{\vnu}{\boldsymbol{\nu}}
\newcommand{\hvth}{\widehat{\vth}}
\newcommand{\sfG}{\mathsf{G}}
\newcommand{\tmh}{\widetilde{\mh}}
\newcommand{\hvnu}{\widehat{\vnu}}

\newcommand{\hmb}{\widehat{\mb}}

\newcommand{\Wish}{\textup{Wish}}
\newcommand{\hdd}{\widehat{\dd}}
\newcommand{\hvv}{\widehat{\vv}}
\newcommand{\fidinf}{\mathrm{FID}_{\infty}}
\newcommand{\hD}{\widehat{D}}
\newcommand{\tmb}{\widetilde{\mb}}
\newcommand{\tmc}{\widetilde{\mc}}
\newcommand{\bmy}{\overline{\my}}
\newcommand{\hQ}{\widehat{Q}}
\newcommand{\cyc}{\mathrm{cyc}}
\newcommand{\mxi}{\boldsymbol{\Xi}}
\newcommand{\tmv}{\widetilde{\mv}}
\newcommand{\tmx}{\widetilde{\mx}}
\newcommand{\tmr}{\widetilde{\mr}}

\title{Sample-Optimal Estimation of the Fr\'echet Inception Distance}
\date{}

\author{Ziyun Chen\thanks{
University of Washington,
\texttt{ziyuncc@cs.washington.edu}
} \and Jerry Li\thanks{
University of Washington,
\texttt{jerryzli@cs.washington.edu}
} \and Kevin Tian\thanks{
University of Texas at Austin,
\texttt{kjtian@cs.utexas.edu}
}\and Yusong Zhu\thanks{
University of Texas at Austin,
\texttt{zhuys@utexas.edu}
}}
\begin{document}

\maketitle

\begin{abstract}
The Fréchet Inception Distance (FID) is widely used to evaluate generative models, but its empirical plug-in estimator suffers from finite-sample bias \cite{bińkowski2018demystifying, chong2019effectively}. We study the sample complexity $n$ of estimating FID to error $\epsilon$ between $d$-dimensional Gaussians with bounded mean distance and covariances, when one distribution is known. Our contributions are threefold.\vspace{1mm}
\begin{enumerate}
    \item We establish tight finite-sample $\Theta(\frac{d^2}{n})$ bias and $\Theta(\frac{d}{n} + \frac {d^2} {n^2})$ variance bounds for the empirical plug-in estimator, establishing a $\gtrsim d^2$ sample complexity.
    \item To debias the empirical plug-in estimator, we generalize the $\fidinf$ estimator of \cite{chong2019effectively} to extrapolation methods of arbitrary order $k$. We further prove tight bias and variance bounds of $\Theta(\frac{d^{k + 2}}{n^{k + 1}})$ and $\Theta(\frac d n + \frac{d^2}{n^2})$ for any order-$k$ extrapolation under our framework.
    \item We introduce Relative Taylor Debiasing (RTD), a new, computationally efficient FID estimation algorithm using debiasing techniques inspired by U-statistics. We show that RTD achieves an $O(\frac d {\eps^2})$ sample complexity, and prove that this is optimal.
\end{enumerate}\vspace{1mm}
We provide a complementary empirical evaluation of our new estimators.
Our experiments on synthetic Gaussians validate the predicted residual bias and support the tightness of our bounds. On ImageNet with Inception embeddings, RTD achieves the lowest mean estimation error at the standard 50K sample budget, while our second-order variance-aware extrapolation estimator (VALE$_2$) uses only 10K samples to achieve accuracy comparable to FID$_\infty$ at 50K samples.\footnote{Authors are ordered alphabetically.}\footnote{Our code is available at {\urlstyle{same}\url{https://github.com/zys996/sample-optimal-fid-estimation}}.}
\end{abstract}
\thispagestyle{empty}
\newpage
\begingroup
\setlength{\parskip}{0pt}
\tableofcontents
\endgroup
\newpage

\pagenumbering{arabic}
\section{Introduction}
\label{sec:intro}

A central challenge in generative modeling scholarship is \emph{evaluation}: that is, how can we quantitatively judge the quality of a model's generations?
Without evaluation, it is difficult to make concrete comparisons between different methods, algorithms, and architectures.
Downstream performance can sometimes be used as a proxy for quality, but many settings, notably image generation, where the ultimate goal is to produce high-quality natural images, require metrics that serve as direct proxies for generative quality. While there is no perfect metric for this task, the \emph{Fr\'{e}chet inception distance} (FID), introduced in~\cite{heusel2017gans}, has become a standard tool for this purpose.

The FID is a measure of distance between distributions $P, Q$ over image embeddings.
Formally, let $P$, $Q$ be distributions over images passed through a fixed embedding (often the Inception network \cite{szegedy2016rethinking}). When $P = \calN(\vmu, \ma)$ and $Q = \calN(\vnu, \mb)$ are multivariate Gaussians, their FID is exactly the squared \emph{$2$-Wasserstein} distance $W_2^2(P, Q)$, and enjoys the parametric form
\begin{equation}
\label{eq:frechet_dist}
\FID(P,Q)
\defeq
\twonorm{\vmu-\vnu}^2
+\Tr(\ma)+\Tr(\mb)
-2\Tr\paren{\paren{\ma^{1/2}\mb\ma^{1/2}}^{1/2}},
\end{equation}
see \cite{gelbrich1990formula}. For non-Gaussian $P, Q$, it is typical to take $\FID(P, Q)$ to be the FID between the multivariate Gaussians with matched means and covariances to $P$ and $Q$.

In practice, the FID is estimated from a finite dataset of samples. The \emph{empirical plug-in estimator}, which simply evaluates \eqref{eq:frechet_dist} on empirical means and covariances, is perhaps the most commonly used estimator. Unfortunately, this estimator exhibits a substantial finite-sample bias, 
particularly in high-dimensional spaces. For example, \cite{bińkowski2018demystifying} construct a $2048$-dimensional, exactly Gaussian example in which FID estimates constructed using 50,000 samples consistently give the wrong ordering between distributions. Recent studies of modern generative models likewise report that FID-based conclusions depend substantially on the sample size \cite{chong2019effectively, jayasumana2024rethinking}.


Thus, finite-sample bias is a well-documented issue when benchmarking with the FID. 
It can materially affect comparisons between generative models, even at practical scales, and removing it often requires the user to generate many images, which becomes extremely costly, especially if evaluation is repeated throughout model development. This motivates our central question.
\begin{center}
    {\it Can we characterize the sample complexity of estimating the FID?}
\end{center}

A precursor to our work, \cite{chong2019effectively}, also investigated this question, and proposed an alternative estimator called $\fidinf$. They further empirically demonstrated that $\fidinf$ attains an improved sample complexity compared to the empirical plug-in estimator, at a fixed error level. However, \cite{chong2019effectively} does not prove non-asymptotic sample complexity bounds for either the plug-in estimator or $\fidinf$, leaving open the possibility of improved estimators in both theory and practice.

\subsection{Our results}

In this paper, we develop new estimators for FID estimation, alongside a complementary non-asymptotic sample complexity theory for the problem. 
We focus on the setting where image embedding distributions $P$ and $Q$ are exactly Gaussian, a canonical setting for high-dimensional statistical inference. 
This setting is particularly natural for FID, which by its very definition as a squared Wasserstein distance between moment-matched Gaussians, assumes the image embeddings are close to a Gaussian (else measuring this distance is ill-posed). Under this assumption, we abstract FID estimation into the following natural parametric statistical estimation problem.

\begin{problem}
\label{prob:main}
Let $\vmu, \vnu \in \R^d$ and let $\ma, \mb \in \R^{d \times d}$ be positive semidefinite.
Let $P = \calN(\vmu, \ma)$ be a reference distribution, where $\vmu$, $\ma$ are known, and let $\eps > 0$.
    Given $n$ i.i.d.\ samples $\{\vX_i\}_{i \in [n]} \sim \calN (\vnu, \mb)$, where $\vnu$, $\mb$ are unknown, estimate~\eqref{eq:frechet_dist} with high probability up to additive error $\eps$.
\end{problem}

We give a formal statement of Problem~\ref{prob:main} in Section~\ref{ssec:prob_def}. Throughout the paper, we also impose a problem scale (Assumption~\ref{assum:normalization}), an assumption which is necessary because the FID is not scale invariant (i.e., $\eps$ grows with the size of $\ma$, $\mb$, $\norm{\vmu - \vnu}_2$).

We now overview our main theoretical findings, with a detailed side-by-side comparison in~\Cref{tab:estimator-comparison}.
\begin{enumerate}
    \item We give the first tight minimax convergence rates for the empirical plug-in estimator (Theorem~\ref{thm:first-order-plugin-bias-variance}) and $\fidinf$ \cite{chong2019effectively} (Corollary~\ref{cor:fidinf}).
    Our rates demonstrate that the empirical plug-in estimator requires $O(d^2)$ samples to achieve constant error, and that $\mathrm{FID}_\infty$ decreases this to $O(d^{3/2})$. However, this is not entirely ``for free'': in particular, for fixed $N$, the error of $\mathrm{FID}_\infty$ scales as $\Theta(d^3)$ with the dimension as opposed to $\Theta(d^2)$ for the empirical plug-in estimator. All of these rates are accompanied by matching lower bounds (Proposition~\ref{prop:extrapolation-bias-lower-bound}).
    \item We give a new generalization of $\mathrm{FID}_\infty$, which we call \VA~(FID with Variance-Aware Linear Extrapolation, see~\Cref{sec:extrapolation}), that uses variance-aware higher-order extrapolation to obtain a sample complexity $\Theta (d^{(k + 2) / (k + 1)})$ at constant error, for any integer $k \ge 1$.
    \item Finally, we introduce a novel estimator based on explicit higher-order debiasing, which we call \emph{Relative Taylor Debiasing} (FID-RTD, see~\Cref{sec:optimal-estimation}). We show FID-RTD achieves the minimax optimal sample complexity of $n = \Theta(d\eps^{-2})$ for FID estimation (Theorem~\ref{thm:rtd-estimation}, Proposition~\ref{prop:gaussian-w2-lower-bound}).
\end{enumerate}

\begin{table*}[t]
    \centering
    \small
    \caption{
        Sample complexity and runtime comparison of FID estimators. We report the sample complexity necessary to achieve constant failure probability for Gaussian data, and suppress constant factors;
        $n$ is the sample size, $\omega < 2.373$ is the matrix multiplication exponent, and $\gamma > 0$ is an arbitrarily small constant. All results come with matching lower bounds (Propositions~\ref{prop:gaussian-w2-lower-bound},~\ref{prop:extrapolation-bias-lower-bound}).
    }
    \label{tab:estimator-comparison}
    \begin{tabular}{@{}lcc@{}}
        \toprule
        Method
        & Error at sample size $n$
        & Runtime at sample size $n$ \\
        \midrule

        Empirical plug-in~\cite{heusel2017gans}
        &
        $\displaystyle
        
            \sqrt{\frac dn}+\frac{d^2}{n}$
        (Theorem~\ref{thm:first-order-plugin-bias-variance})
        &
        $\displaystyle
        
            nd^{\omega-1}
            +d^\omega$
        \\[2ex]

        $\mathrm{FID}_\infty$~\cite{chong2019effectively}
        &
        $\displaystyle
        
            \sqrt{\frac dn}+\frac{d^3}{n^2}
       $
        (Corollary~\ref{cor:fidinf})
        &
        $\displaystyle
        
            nd^{\omega-1}
            +d^\omega
        $
        \\[2ex]

        \VA, constant order $k$ (\Cref{ssec:choosing-extrapolation-weights})
        &
        $\displaystyle
        
            \sqrt{\frac dn}
            +
            \frac{d^{k+2}}{n^{k+1}}
        $
        (Theorem~\ref{thm:general-order-extrapolation})
        &
        $\displaystyle
        
            nd^{\omega-1}
            +d^\omega
        $

        \\[2ex]

        FID-RTD (Section~\ref{sec:optimal-estimation})
        &
        $\displaystyle
        
            \sqrt{\frac dn}
        $
        (Theorem~\ref{thm:rtd-estimation})
        &
        $\displaystyle n^{1 + \gamma}d^{\omega -1} + d^{\omega + \gamma}$
        \\
        \midrule
        Information-theoretic lower bound
        &
        $\displaystyle
        
            \sqrt{\frac dn}
        $
        (Proposition~\ref{prop:gaussian-w2-lower-bound})
        &
        N/A
        \\
        \bottomrule
    \end{tabular}
\end{table*}

\begin{figure}[t]
\centering
\includegraphics[width=0.8\textwidth]{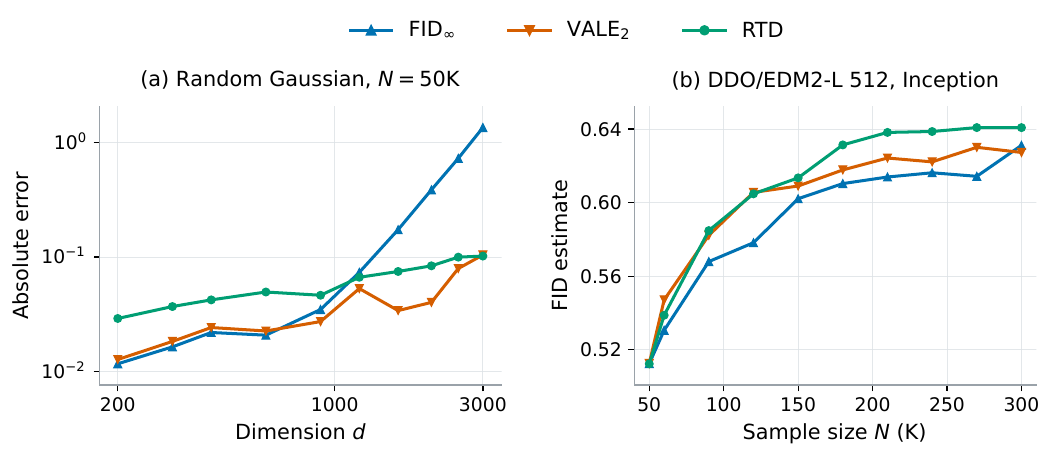}
\caption{Representative Gaussian and ImageNet comparisons.
(a) Haar log-uniform covariances at $N=50{,}000$; each point is the mean across ten covariance pairs of the median absolute error over five trials.
(b) DDO/EDM2-L 512 with Inception features; each point is the median of five estimates from the first $N$ development samples.}
\label{fig:intro-estimator-comparison}
\end{figure}

\textbf{Experimental evaluations.}
We complement our theoretical results with evaluations on both synthetic and real-world data to demonstrate that our new FID estimators substantially improve upon existing FID estimators such as the plug-in estimator and $\fidinf$ in both settings.

We first validate the correctness of our theoretical findings in a variety of synthetic experiments on Gaussian data (see~\Cref{sec:gaussian_exp} for more details).
Our key findings here are that the factor of $d^3$ in the statistical rate for $\mathrm{FID}_\infty$ can be witnessed in experimentally relevant settings, and that in such settings, RTD has substantially milder dimension dependence than $\mathrm{FID}_\infty$ (see~\Cref{fig:intro-estimator-comparison}(a)).
However, our alternative \VA~estimator at order $2$ performs comparably to RTD in both the Gaussian and real-data examples in~\Cref{fig:intro-estimator-comparison}, despite its theoretically worse convergence rate.

We next perform experiments evaluating our method on real generative models trained on ImageNet (Section~\ref{sec:imagenet_exp}).
We consider seven generator checkpoints, and three different embeddings, namely, Inception, DINOv2, and CLIP ViT-L/14.
Because we do not have ground-truth FID scores here, we compare the quality of our estimates against those of proxies computed with significantly more samples.
Namely, we evaluated our methods using $n \in \cbra{10\mathrm{K}, 20\mathrm{K}, 30\mathrm{K}, 40\mathrm{K}, 50\mathrm{K}}$ samples against proxies which used $n = 300\mathrm{K}$ samples.
One subtle question is what proxy to use to compute the FID using 300K samples.
Naively, it would be most natural to use the plug-in estimator to compute the proxy.
Interestingly, in some settings, the proxy computed by the plug-in estimator remained noticeably higher than the debiased estimates even at $n = 300,000$ samples --- indeed, this is exactly the problem this paper addresses!
To rectify this, we also compared our methods' FID estimates against proxies computed with $n = 300,000$ samples using both $\mathrm{FID}_\infty$ and our new FID-RTD estimator, which agree closely at large $n$ (see~\Cref{fig:imagenet-proxy-example}).

At a high level, our finding (see~\Cref{tab:imagenet-proxy-median-error}) is that at 50K samples, either \VA~at order $k=2$ or FID-RTD achieves the lowest mean error for each embedding and reference proxy.
This yields strong evidence for the real-world utility of our methods as alternatives to existing estimators for FID.
We give more detailed experimental results as well as additional ablations in~\Cref{sec:imagenet_exp}.
\subsection{Our techniques}
We now describe our new algorithms for Problem~\ref{prob:main} and our proof ideas in more detail.
In a slight abuse of notation, let $\FID (\calN(\vmu, \ma), \calN(\vnu, \mb))$ denote the quantity of interest, i.e., the expression in~\eqref{eq:frechet_dist}.
We define the squared \emph{Bures-Wasserstein} distance between positive semidefinite $\ma, \mb$ by \[\BW^2(\ma, \mb) \defeq \Tr(\ma) + \Tr(\mb) - 2\Tr((\ma^{1/2}\mb\ma^{1/2})^{1/2}),\] so that $\FID (\calN(\vmu, \ma), \calN(\vnu, \mb)) = \twonorm{\vmu-\vnu}^2 + \BW^2 (\ma, \mb).$ The difficulty of FID estimation comes from the $\BW^2(\ma, \mb)$ term, so for simplicity in this overview, we assume $\vmu = \vnu = \0_d$.

\textbf{Tight rates for the empirical plug-in estimator.} 
At a sample size of $n$, the empirical plug-in estimator applies the formula in~\eqref{eq:frechet_dist} to the empirical mean and covariance:
\begin{equation*}\hvnu_n \defeq \frac 1 n \sum_{i \in [n]} \vX_i,\quad \hmb_n \defeq \frac 1 {n} \sum_{i \in [n]} (\vX_i - \hvnu_n)(\vX_i - \hvnu_n)^\top. \end{equation*}
More precisely, the usual plug-in estimator outputs
\begin{equation*}
\widehat{\dd}_n \defeq \norm{\vmu - \hvnu_n}_2^2 + \BW^2\Par{\ma, \frac n {n - 1}\hmb_n}.
\end{equation*}
The slight adjustment to the empirical covariance is justified by classical statistical considerations (see e.g.~\cite{muirhead2009aspects}, Chapter 3), but the estimator is otherwise the ``naive'' estimator.

We analyze the plug-in estimator $\hdd_n$, its extension $\mathrm{FID}_\infty$, as well as our generalizations thereof, using a shared analytic framework.
The main difficulty in understanding the behavior of $\widehat{\dd}_n$ is in understanding the behavior of the squared Bures-Wasserstein distance between $\ma$ and our empirical covariance $\hmb_n$, specifically, dealing with the matrix square root therein.
We do so by appealing to the \emph{resolvent} representation of the matrix square root (Fact~\ref{fact:sqrt-resolvent}), which allows us to expand the expression for the matrix square root into an infinite series. Expressing $\hmb_n = \mb^{1/2}(\id_d + \mh) \mb^{1/2}$ as a perturbation of the unknown covariance $\mb$, where $\mh$ arises due to finite-sample error in the empirical covariance, \Cref{prop:trace-resolvent-expansion} decomposes the estimation error as
\begin{equation}\label{eq:bias_decomposition}\BW^2(\ma, \mb_{\mh}) - \BW^2(\ma, \mb) = \inprod{\mb - \mc}{\mh} + \sum_{k \ge 2} Q_{k, \mc}(\mh),\text{ where } \mc \defeq \Par{\mb^{1/2}\ma\mb^{1/2}}^{1/2}, \end{equation}
where $Q_{k, \mc}(\mh)$ captures the $k^{\text{th}}$-order effect of the resolvent expansion. 
We exactly characterize the lowest-order term $Q_{2, \mc}(\mh)$ in Lemma~\ref{lem:first_two_terms}, and then Proposition~\ref{prop:trace-resolvent-expansion} relates each higher-order term to $Q_{2, \mc}(\mh)$ directly. Our decomposition lets us leverage standard identities for Gaussian perturbations $\mh$ to bound $Q_{2, \mc}(\mh)$, and derive non-asymptotic bounds on the overall bias and variance.

Interestingly, Theorem~\ref{thm:first-order-plugin-bias-variance} diagnoses the finite-sample inaccuracy of $\hdd_n$ as primarily arising due to its bias. In particular, its variance is $\approx \frac d n$, which suffices to match the lower bound in Proposition~\ref{prop:gaussian-w2-lower-bound}, whereas its worst-case bias is $\gtrsim \frac{d^2}{n}$, precluding convergence unless $n \gtrsim d^2$.

\textbf{Tight rates for $\fidinf$.} $\mathrm{FID}_\infty$ is an elegant extension of the plug-in estimator $\hdd_n$, proposed by \cite{chong2019effectively}.
As detailed in~\Cref{sec:extrapolation}, the key observation is that $\widehat{\dd}_n$'s bias depends in a predictable way on $n$, the sample size. 
Specifically, by leveraging the expansion in~\eqref{eq:bias_decomposition}, one can show that the plug-in estimator's bias has the exact characterization
\begin{equation}\label{eq:fid_bias_intro}\E[\widehat{\dd}_n] - \FID (\calN(\vmu, \ma), \calN(\vnu, \mb)) = \frac{\va_1}{n} + O_d \left( \frac{1}{n^2} \right),\end{equation}
for some fixed $\va_1$.
Thus, a linear combination of independent instances of $\hdd_n$, computed at different sample sizes $n$, can exactly offset the low-degree term in \eqref{eq:fid_bias_intro}.
Perhaps the simplest variant of $\mathrm{FID}_\infty$ is to take $N$ samples, split them into subsets of size $\frac N 3$ and $\frac{2N}{3}$, and form
$
\mathrm{FID}_\infty \defeq 2 \widehat{\dd}_{\frac{2N}{3}} - \widehat{\dd}_{\frac N 3}$
as an overall estimate.
The coefficients and sample sizes are chosen to exactly remove the $\frac{\va_1}{n}$ term in the bias \eqref{eq:fid_bias_intro}, regardless of the value of $\va_1$, because $2 \cdot \frac{3}{2N} - \frac{3}{N} = 0$.

Proposition~\ref{prop:uniform-psd-expansion} gives a tighter characterization of $\hdd_n$'s bias that justifies \eqref{eq:fid_bias_intro}, as well as a fine-grained high-order extrapolation \eqref{eq:higher-order-expansion} to be discussed. We then repurpose our analysis of the plug-in estimator in Theorem~\ref{thm:first-order-plugin-bias-variance} to derive a tight $\frac{d^{3/2}}{\eps^{1/2}} + \frac{d}{\eps^2}$ sample complexity bound for $\fidinf$ (Corollary~\ref{cor:fidinf}).

\textbf{\VA: variance-aware high-order extrapolation.} 
The previous discussion motivates a natural generalization of $\mathrm{FID}_\infty$: cancel additional terms in the bias \eqref{eq:fid_bias_intro}, beyond the term linear in $\frac 1 n$.
Indeed, Proposition~\ref{prop:uniform-psd-expansion} uses our expansion \eqref{eq:bias_decomposition} to show that the plug-in estimator's bias satisfies
\begin{equation}
\label{eq:higher-order-expansion}
\E[\widehat{\dd}_n] \approx \FID (\calN(\vmu, \ma), \calN(\vnu, \mb)) + \frac{\va_1}{n} + \frac{\va_2}{n^2} + \ldots + \frac{\va_k}{n^k} + O_d\Par{\frac{1}{n^{k + 1}}},
\end{equation}
for any $k \geq 1$. Given such an expansion, it is natural to split the samples into $\ell \geq k + 1$ subsamples, and use a linear combination of these estimates to exactly cancel out all degree $\leq k$ terms in this approximation, yielding a better overall approximation to the true FID.

This approach yields a whole class of estimators.
Let $N$ denote the total number of samples in an extrapolation scheme, and let $k \geq 1$.
Then, for any integer $\ell \geq k + 1$, any sequence of $\ell$ sample sizes $\vn = (\vn_1, \ldots, \vn_\ell)$  summing to $N$, and deterministic weights $\vw \in \R^\ell$, we can define an estimator
\begin{equation}
\label{eq:extrapolation-estimator-intro}\hD\Par{\vn, \vw} \defeq \sum_{j \in [\ell]} \vw_j \widehat{\dd}_{\vn_j}, \end{equation}
where each $\widehat{\dd}_{\vn_j}$ is computed independently from $\vn_j$ samples. For the extrapolation \eqref{eq:extrapolation-estimator-intro} to succeed in canceling out the first $k$ terms in the bias expansion \eqref{eq:higher-order-expansion}, our framework requires
\begin{equation}
\label{eq:order-k-exact-weights}
 \vw \in   \calW_k (\vc)
    \defeq
    \left\{
        \vw\in\R^\ell:
        \sum_{j\in[\ell]}\vw_j=1,
        \;
        \sum_{j \in [\ell]}\frac{\vw_j}{\vc_j^i}=0
        \text{ for all } i \in [k]
    \right\} ,
\end{equation}
where $\vc \defeq \frac 1 N \vn$ are the relative sample sizes. In Lemma~\ref{lem:extrapolation-bias-cancellation}, we show that the estimator \eqref{eq:extrapolation-estimator-intro} applied with \emph{any} $\vw, \vc \in \R^\ell$ satisfying $\vw \in \calW_k(\vc)$ incurs a bias scaling as $O_d(\frac 1 {N^{k + 1}})$.

The main design choice left is then to choose the interpolation points $\vc$ (or equivalently $\vn$) and weights $\vw$.
For $\vc$, we consider three different choices of $\vc$ motivated by considerations from classical approximation theory, namely, uniformly spaced, inverse uniformly spaced, and spacing based on the \emph{Chebyshev-Lobatto} nodes (see Section~\ref{ssec:choosing-extrapolation-weights}) for more details.

In our experiments, the most important optimization empirically was to obtain good choices of $\vw$.
The choice of $\vw$ can be recast as a linear regression problem: given noisy estimates of~\eqref{eq:higher-order-expansion} at different values of $\vn_j$, the goal is to recover a good estimate of the true FID, the constant term in the expansion.
Our key observation is that the variance of each estimate depends on $\vn_j$, or equivalently on $\vc_j$, and thus one should choose weights $\vw$ which are a solution to a \emph{variance-aware} weighted regression, where the weights depend on $\vc_j$.
We call this method FID with Variance-Aware Linear Extrapolation (\VA); for a full description, see Section~\ref{ssec:choosing-extrapolation-weights}.
We note that the weights $\vw$ depend only on $\vc$, and thus can be precomputed independently of the data.

\textbf{FID-RTD: sample-optimal estimation via high-order debiasing.} While \VA~works well in most practical settings, it fails to match the lower bound (Proposition~\ref{prop:gaussian-w2-lower-bound}) at any fixed $k$, as $d \to \infty$. This motivates FID-RTD, our estimator with the strongest theoretical guarantees.

Our FID-RTD is based on an alternative strategy for debiasing $\hdd_n$: 
directly estimating each high-order term $Q_{k, \mc}$ in the expansion \eqref{eq:bias_decomposition} for an initial fixed perturbation $\mh$, using independent batches of samples. This is inspired by the literature on \emph{U-statistics}, which has developed tools for estimating parameters expressible as symmetric functions of independent random variables.

More formally, our FID-RTD estimator first takes a small set of $O(d)$ pilot samples to obtain $\tmb$, a rough spectral approximation to $\mb$, to center the expansion \eqref{eq:bias_decomposition}. It then 
truncates the expansion at degree $L = \Theta(\log \frac d \eps)$, incurring a negligible bias. The remaining task is to estimate each $Q_{k, \mc}(\mh)$ for $2 \le k \le L$, where $\mh = \tmb^{-1/2}\mb\tmb^{-1/2} - \id_d$ is the residual error of our pilot sample.\footnote{Formally, the value of $\mc$ changes slightly from \eqref{eq:bias_decomposition} when using this strategy, see \eqref{eq:optimal-relative-localization}.}

The key idea is to observe that $Q_{k, \mc}$ has an unbiased estimator.
This is because we can associate $Q_{k, \mc}$ with a symmetric $k$-linear form (order-$k$ tensor) in $\mh$, i.e., $Q_{k,\mc}(\mh) = \calQ_{k,\mc} [\mh, \ldots, \mh]$, where $\calQ_{k, \mc}: (\R^{d \times d})^k \to \R$.
Thus, we can estimate $Q_{k,\mc}(\mh)$ by breaking up our samples into $k$ disjoint batches, using each batch to form an independent unbiased estimator $\my_1, \ldots, \my_k$ of $\mh$, and approximating $Q_{k,\mc}(\mh) \approx \calQ_{k,\mc} [\my_1, \ldots, \my_k]$.
Combining across all levels $2 \le k \le L$ yields an estimator for the right-hand side of \eqref{eq:bias_decomposition}, and thus, an estimator for $d_{\mathrm{BW}}^2(\ma,\mb)$ and the FID.

To show the correctness of FID-RTD, the main technical work is in bounding its variance. Indeed, our estimator reuses the same batch of samples at every level $k$ to avoid an $L = \Omega(\log d)$-factor blowup in the sample complexity, which causes correlated estimates. 
In Section~\ref{ssec:optimal-direct-debiasing} we demonstrate these correlations can all be accounted for using judicious weighted applications of the Cauchy-Schwarz inequality, due to the exponential decay of our higher-order bounds, ultimately showing that FID-RTD achieves an optimal statistical rate up to constant factors.

The final difficulty we must address is the runtime of the FID-RTD estimator.
Indeed, even writing down the order-$k$ tensor $\calQ_{k, \mc}$ would take $d^{O(k)}$ time, which, for $k = \Theta (\log d)$, is super-polynomial. However, it turns out that $\calQ_{k, \mc}$ has a convenient structure, based on a recursive representation of the matrix square root derived by \cite{del2018taylor}. This structure admits a recurrence for evaluating $\calQ_{k, \mc}$ on any repeated input $\mh^{\otimes k}$ in time $O(k^2 d^\omega)$, as we show in Lemma~\ref{lem:implement_homogeneous}. By applying the \emph{polarization identity} of \cite{MazurO34, Thomas14}, we then apply Lemma~\ref{lem:implement_homogeneous} to evaluate $\calQ_{k, \mc}$ on mixed inputs at a $2^k \le 2^L$ blowup; see Algorithm~\ref{alg:evaluate-q}. By choosing the constant in $L = O(\log \frac d \eps)$ sufficiently small, FID-RTD attains the optimal sample complexity in time $O(d^{\omega+\gamma}/\eps^{2+\gamma})$ for any fixed $\gamma > 0$; see Section~\ref{ssec:optimal-efficient-evaluation}.

\subsection{Related work}

\textbf{$W_2$ estimation with provable guarantees.}
The statistical properties of empirical plug-in estimators for Gaussian $W_2$ and $W_2^2$ distances have been studied in several prior works. For example,
\cite{rippl2016limit} establish asymptotic limit laws for $W_2^2$, without finite-sample guarantees, and
\cite{zhu2020deconstructing} obtain an $O(\frac{d^2}{\eps^2})$ sample complexity for $W_2$ estimation, assuming that the covariance spectra are bounded away from zero. This requirement is avoided by
\cite{hu25a}, who give a sample complexity of $O(\frac{d^5}{\eps^4})$ for $W_2^2$ estimation.
Beyond the empirical plug-in estimator, alternative estimators aim to improve sample efficiency:
\cite{tiomoko2019estimation,tiomoko2019random} develop random matrix corrections for $W_2^2$ estimation, but establish only asymptotic guarantees.
For general smooth nonlinear functionals, extrapolation and Taylor debiasing methods have been analyzed by~\cite{koltchinskii2023estimation,koltchinskii2026functional}.
However, smoothness assumptions do not hold uniformly for $W_2^2$ estimation when singular covariances are allowed, and the resulting sample complexities incur a polynomial factor in the inverse-smallest eigenvalue, which we avoid.

\textbf{FID evaluation and debiasing.}
Since its introduction by~\cite{heusel2017gans}, FID has become a standard metric for evaluating generative models~\cite{peebles2023scalable,tian2024visual,jayasumana2024rethinking}. Despite its popularity, subsequent works identify limitations and propose improvements along two directions.
One concerns \emph{finite-sample error}, e.g., \cite{chong2019effectively} propose $\mathrm{FID}_{\infty}$, which extrapolates FID scores in inverse sample size using a Taylor expansion heuristic. Their experiments support the underlying assumption and demonstrate empirical improvements, but they do not establish a rigorous finite-sample guarantee of $\fidinf$'s improvement over the empirical plug-in estimator. 
Another direction questions \emph{the representational power} of Inception embeddings. Prior work explores alternative representations such as CLIP~\cite{kynk2023the} and DINOv2~\cite{stein2023exposing} to address limitations of Inception features in perceptual evaluation. \cite{jayasumana2024rethinking} empirically show that image distributions in Inception feature space are multimodal and violate the Gaussian assumption. They also expose FID's vulnerability to manipulation: deliberately constructed image distortions can improve FID scores while degrading image quality. To mitigate this issue, MMD-based alternatives that avoid the Gaussian assumption have been proposed, including KID~\cite{bińkowski2018demystifying} and CMMD~\cite{jayasumana2024rethinking}.
\section{Preliminaries}
\label{sec:prelim}

\subsection{Notation}\label{ssec:notation}

\textbf{General notation.} We use $\lesssim$, $\gtrsim$, and $\asymp$ to hide universal constants, and $\lesssim_\delta$, etc.\ to denote that the universal constant depends on a parameter $\delta$ (also treated as a constant). We define $[d] \defeq \{i \in \N: i \le d\}$. We denote vectors in either lowercase boldface or uppercase (italicized) boldface, and matrices in uppercase (upright) boldface. The identity matrix in $\R^d$ is denoted $\id_d$, and the all-zeroes $m \times n$ matrix is denoted $\0_{m \times n}$. We define the set of symmetric $d \times d$ matrices by $\Sym^{d \times d}$ with inner product $\inprod{\ma}{\mb} \defeq \Tr(\ma\mb)$, and the positive semidefinite (PSD) subset by $\PSD^{d \times d}$, equipped with the usual Loewner partial ordering, denoted by $\preceq$. We use $\Tr(\cdot)$, $\normop{\cdot}$, and $\normf{\cdot}$ to denote the trace, $\ell_2$ operator norm, and Frobenius norm of a matrix argument.

We use $\omega < 2.373$ to denote the exponent of $d \times d$ matrix multiplication \cite{AlmanDWXXZ25}. By standard reductions \cite{Strassen69, PanC99}, we also assume inversion and eigendecomposition take $O(d^\omega)$ time.

\textbf{Probability.} We let $\ind_{\calE}$ denote the $0$-$1$ indicator function of an event $\calE$. We use $\overset{d}{=}$ to denote equality in distribution. We let $\calP(\Omega)$ denote the set of probability measures over a sample space $\Omega$, which we identify with their density functions when applicable. We let $\calN(\vmu, \msig)$ be the multivariate Gaussian with mean $\vmu \in \R^d$ and covariance $\msig \in \PSD^{d \times d}$. We use $W_2(P, Q)$ to denote the $2$-Wasserstein distance between $(P, Q)$ in $\calP(\R^d)$ with finite second moments, denoted $\calP_2(\R^d)$. For an arbitrary $P \in \calP_2(\R^d)$ with mean $\vmu$ and covariance $\msig$, we let $P^\sfG \defeq \calN(\vmu, \msig)$ denote the moment-matched Gaussian.

\textbf{Properties of FID.} We review some important properties of FID.
First, it is symmetric, because
\begin{equation}\label{eq:gram_trace}
\Tr\Par{(\ma^{1/2}\mb\ma^{1/2})^{1/2}} = \Tr\Par{(\mb^{1/2}\ma\mb^{1/2})^{1/2}} \text{ for all } \ma, \mb \in \PSD^{d \times d},
\end{equation}
which follows because $\ma^{1/2}\mb\ma^{1/2}$ and $\mb^{1/2}\ma\mb^{1/2}$ have the same eigenvalues.

We recall a classical fact about $\BW$ (see Theorem 1.1, \cite{bhatia2019bures}): for $\ma, \mb \in \PSD^{d \times d}$ and any $r \ge d$,
\begin{equation}
\label{eq:bures-frobenius}
    d_{\mathrm{BW}}(\ma,\mb)
    =
    \min_{\substack{\mm\in\R^{d\times r}:
                     \mm\mm^\top=\ma}}
    \norm{\mm-\mn}_{\mathrm F}, \text{ for any } \mn \in \R^{d \times r},\; \mn\mn^\top=\mb.
\end{equation}

\subsection{Problem definition}\label{ssec:prob_def}
In this paper, we study finite-sample estimation of the FID against a reference distribution.

\begin{model}
\label{model:one-side}
We are given $\vmu \in \R^d$, $\ma \in \PSD^{d \times d}$ inducing $P \defeq \calN(\vmu, \ma)$, and observe i.i.d.\ samples $
    \cbra{\vX_i}_{i\in[n]}
    \simiid
    Q\defeq\calN(\vnu,\mb)$, where $\vnu \in \R^d$, $\mb \in \PSD^{d \times d}$ are unknown.
\end{model}

Model~\ref{model:one-side} is particularly common in practice, where the feature statistics of a public reference dataset are often precomputed once and reused across generator evaluations. In principle, one could also investigate FID under a two-sided model to compare two independently collected datasets without reference statistics. For simplicity, we focus on FID estimation under Model~\ref{model:one-side}.

Our paper works under the following convenient normalization.
\begin{assumption}
\label{assum:normalization}
Let $P=\calN(\vmu,\ma)$ and $Q=\calN(\vnu,\mb)$. We assume that
\begin{equation*}
    \vzero_{d \times d} \preceq\ma,\mb\preceq\id_d,
    \qquad
    \twonorm{\vmu-\vnu}\le1.
\end{equation*}
\end{assumption}

Assumption~\ref{assum:normalization} fixes the scale of the problem (so we may take the error parameter to be scale-invariant), and can be enforced by rescaling the samples, given bounds on $\max\{\normsop{\ma}, \normsop{\mb}\}$ and $\norm{\vmu - \vnu}_2$. 
The estimates in turn can be obtained with high probability given $O(d)$ pilot samples, using standard techniques which we defer to Chapter 4 of~\cite{Vershynin_2026}. 

We finally state the main estimation problem studied in this work. We fix the failure probability $\delta$ to a constant, which can be boosted using the median of $\log(\frac 1 \delta)$ independent repetitions.

\begin{problem}
\label{prob:FD_est}
Under Model~\ref{model:one-side},
suppose Assumption~\ref{assum:normalization} holds for $(P, Q)$. Given an accuracy parameter $\eps \in (0, 1)$, construct an estimator $\widehat \dd$ such that
\begin{equation}
\label{eq:fid_estimation_goal}
    \prob{\Abs{\widehat \dd- \FID(P,Q)}>\eps}
    \le \frac 1 3.
\end{equation}
\end{problem}

In Appendix~\ref{app:lb}, we prove the following information-theoretic lower bound for Problem~\ref{prob:FD_est}.

\begin{restatable}{proposition}{restatelb}
\label{prop:gaussian-w2-lower-bound}
For all $d \ge 1$ and $\eps \in (0, \frac 1 8)$, if $n < \frac{d}{128\eps^2}$, there is no estimator that solves Problem~\ref{prob:FD_est} under Model~\ref{model:one-side} and Assumption~\ref{assum:normalization}.
\end{restatable}

\subsection{Bounds on Wishart matrices}
\label{ssec:gaussian-wishart-prelim}
We briefly collect classical estimates on \emph{Wishart matrices} used throughout the
analysis. We defer a more detailed discussion to \cite{bishop2018introduction}. Recall that for $\msig \in \PSD^{d \times d}$ and $m \in \N$, the \emph{Wishart distribution} $\Wish(\msig, m)$ is a distribution over PSD matrices with mean $m\msig$.

Let $\mh_m$ denote a rescaled-and-centered Wishart matrix, i.e., so that $\mh_m + \id_d \sim \Wish(\frac 1 m\id_d, m)$:
\begin{equation}\label{eq:wishart_m}
    \mh_m
    \defeq
    \frac1m\sum_{i\in[m]}
    \paren{\vg_i\vg_i^\top-\id_d},
    \qquad
    \{\vg_i\}_{i \in [m]}\simiid\calN(\vzero,\id_d).
\end{equation}
\begin{fact}\label{fact:wishart_1}
For every fixed $\mk \in \Sym^{d \times d}$,
\begin{equation*}
    \bbE\norm{\mh_m}_{\mathrm F}^2
    =
    \frac{d(d+1)}m,
    \qquad
    \bbE\innp{\mk,\mh_m}^2
    =
    \frac{2}{m}\norm{\mk}_{\mathrm F}^2.
\end{equation*}
Moreover, for any $\delta \in (0, 1)$, with probability $\ge 1-\delta$,
\begin{equation*}
    \norm{\mh_m}_{\mathrm F}^2
    \lesssim_\delta\frac{d^2}{m},
    \qquad
    \Abs{\innp{\mk,\mh_m}}
    \lesssim_\delta
    \frac{\norm{\mk}_{\mathrm F}}{\sqrt m}.
\end{equation*}
\end{fact}
\begin{fact}
\label{fact:wishart_2}
For every fixed $q\ge0$, if $m \gtrsim_q d$,
\begin{equation*}
    \bbE\left[
        \norm{\mh_m}_{\mathrm F}^2
        \norm{\mh_m}_{\mathrm{op}}^{2q}
    \right]
    \lesssim_q
    \frac{d^{q+2}}{m^{q+1}}.
\end{equation*}
\end{fact}

\subsection{Multilinear forms}
\label{ssec:multilinear-hs}

In Section~\ref{sec:optimal-estimation}, we frequently manipulate multilinear forms (tensors), so we set up some notation here.

Let $\calH$ be a linear subspace of
$\R^{d\times d}$ equipped with the Frobenius inner product
$\inprod{\mm}{\mn}\defeq\Tr(\mm^\top\mn)$; our applications only use $\calH = \R^{d \times d}$ and $\Sym^{d \times d}$. For $j\ge1$, a $j$-linear form is an operator $\calA:\calH^j\to\R$ that is linear in
each argument separately. Given an orthonormal basis of $\calH$,
\[\{\me_r\}_{r \in [q]},\quad q \defeq \dim(\calH),\]
we define the \emph{Frobenius norm} of a $j$-multilinear form $\calA$ by
\begin{equation}
\label{eq:multilinear-hs-definition}
    \normf{\calA}^2
    \defeq
    \sum_{r_1,\ldots,r_j\in[q]}
    \Abs{\calA(\me_{r_1},\ldots,\me_{r_j})}^2.
\end{equation}
Identifying $\calA$ with its coefficient tensor in the basis $\{\me_r\}_{r \in [q]}$, \eqref{eq:multilinear-hs-definition} is simply the entrywise Euclidean norm of $\calA$. Note that \eqref{eq:multilinear-hs-definition} agrees with our definition on $\normf{\cdot}$ for matrices, which can be viewed as a $1$-form on $\calH$. For a symmetric $k$-linear form $\calA$, $j \in [k]$, and $\me \in \calH$, we denote the symmetric $j$-linear form obtained by fixing the last $k - j$ arguments to $\me$ by
\[\calA\Brack{\cdot^{\otimes j}, \me^{\otimes k - j}}.\]
Finally, we state one helper bound used in our analysis.

\begin{lemma}
\label{lem:multilinear-hs-second-moment}
Let $\mz_1,\ldots,\mz_j$ be independent random elements of $\calH$.
Suppose that for every $a\in[j]$,
\begin{equation}
\label{eq:multilinear-input-second-moment}
    \bbE\Abs{\inprod{\mm}{\mz_a}}^2
    \le v\normf{\mm}^2 \text{ for every deterministic }\mm\in\calH.
\end{equation}
Then every deterministic $j$-linear form $\calA:\calH^j\to\R$
satisfies
\begin{equation}
\label{eq:multilinear-hs-second-moment}
    \bbE\Abs{\calA(\mz_1,\ldots,\mz_j)}^2
    \le v^j\normf{\calA}^2.
\end{equation}
\end{lemma}
\begin{proof}
Conditionally on the last $j - 1$ inputs $\mz_2,\ldots,\mz_j$, the resulting operator on $\mz_1$ is
\[\sum_{r \in [q]} \calA(\me_r,\mz_2,\ldots,\mz_j)\me_r.\]
Independence and~\eqref{eq:multilinear-input-second-moment} then give
\begin{equation*}
    \bbE\Abs{\calA(\mz_1,\ldots,\mz_j)}^2
    \le
    v\sum_{r\in[q]}
    \bbE\Abs{\calA(\me_r,\mz_2,\ldots,\mz_j)}^2.
\end{equation*}
Repeating this argument for the remaining slots gives $v^j$ times the
sum in~\eqref{eq:multilinear-hs-definition}.
\end{proof}

\section{Bias-Variance Decomposition of the Plug-In Estimator}
\label{sec:emp_est}

In this section, we give a general framework for analyzing estimators for Problem~\ref{prob:FD_est}, based on a high-order expansion of the Bures-Wasserstein distance (Section~\ref{ssec:relative-bures-series}). We then specialize this framework to the standard empirical plug-in estimator as a warmup application (Section~\ref{ssec:first-order-plugin-expansion}). 

\subsection{High-order expansion of $\BW$}
\label{ssec:relative-bures-series}

We develop a general high-order perturbation expansion of $\BW$. In this subsection, the expansion is stated for a deterministic perturbation, and in Section~\ref{ssec:first-order-plugin-expansion}, we specialize it to empirical covariance matrices. For $\ma, \mb \in \PSD^{d \times d}$ and a perturbation $\mh \in \Sym^{d \times d}$, define
\begin{equation}\label{eq:bh_def}\mb_{\mh} \defeq \mb^{1/2}\Par{\id_d + \mh} \mb^{1/2},\qquad \mc \defeq \Par{\mb^{1/2}\ma\mb^{1/2}}^{1/2}.\end{equation}
We begin by decomposing $\BW^2(\ma, \mb_{\mh})$ in a convenient way, centered at $\mb$.

\begin{lemma}\label{lem:first_order_expand}
Let $\ma, \mb \in \PSD^{d \times d}$ and $\mh \in \Sym^{d \times d}$ satisfy $\id_d + \mh \in \PSD^{d \times d}$. Then following the notation \eqref{eq:bh_def},
\[\BW^2\Par{\ma, \mb_{\mh}} = \BW^2\Par{\ma, \mb} + \inprod{\mb - \mc}{\mh} + \Phi_{\mc}(\mh), \]
where
\begin{equation}
\label{eq:nonlinear-covariance-term}\begin{aligned}
    \Phi_{\mc}(\mh)
    &\defeq \BW^2(\mc, \mc^{1/2}(\id_d + \mh)\mc^{1/2}) \\
    &= 
    \inprod{\mc}{\mh}
    -2\Par{
        \Tr\paren{
            \paren{\mc(\id_d+\mh)\mc}^{1/2}
        }
        -\Tr(\mc)}.
        \end{aligned}
\end{equation}
\end{lemma}
\begin{proof}
Directly by using the definition of $\BW$, and applying \eqref{eq:gram_trace},
\begin{align*}\BW^2\Par{\ma, \mb_{\mh}} - \BW^2\Par{\ma, \mb} &= \Tr\Par{\mb_{\mh} - \mb} - 2\Par{\Tr\Par{(\ma^{1/2}\mb_{\mh}\ma^{1/2})^{1/2}} - \Tr\Par{\mc}} \\
&= \inprod{\mb}{\mh} + 2\Tr(\mc) - 2 \Tr\Par{(\ma^{1/2}\mb_{\mh}\ma^{1/2})^{1/2}}.\end{align*}
Thus it is enough to show that
\[\Tr\Par{(\ma^{1/2}\mb_{\mh}\ma^{1/2})^{1/2}} = \Tr\paren{
            \paren{\mc(\id_d+\mh)\mc}^{1/2}
        }.\]
        This follows from $\mx^\top\mx = \my^\top\my = (\id_d + \mh)^{1/2}\mc^2(\id_d + \mh)^{1/2}$ for $(\mx, \my) \defeq (\ma^{1/2}\mb^{1/2}(\id_d + \mh)^{1/2}, \mc(\id_d + \mh)^{1/2})$, so $\mx\mx^\top = \ma^{1/2}\mb_{\mh}\ma^{1/2}$ and $\my\my^\top = \mc(\id_d + \mh)\mc$ have equal eigenvalues.
\end{proof}

Lemma~\ref{lem:first_order_expand} decomposes $\BW^2(\ma, \mb_{\mh}) - \BW^2(\ma, \mb)$ into a linear term in $\mh$, and a nonlinear quantity $\Phi_{\mc}(\mh)$. To handle the matrix square root in $\Phi_{\mc}(\mh)$, we recall its standard \emph{resolvent} representation.

\begin{fact}[Theorem 1.1, \cite{del2018taylor}]
    \label{fact:sqrt-resolvent}
    For any $\mm \in \PSD^{d \times d}$, we have 
\[
\label{eq:sqrt-resolvent}
    \mm^{1/2}
    =
    \frac1\pi\int_0^\infty
    \mm(t\id_d+\mm)^{-1}t^{-1/2}\mathrm dt.
\]
\end{fact}
To ease use of Fact~\ref{fact:sqrt-resolvent}, we define for all $t > 0$:
\begin{equation}\label{eq:rtdef}
    \mr_t
    \defeq
    (t\id_d+\mc^2)^{-1},
    \qquad
    \mt_t
    \defeq
    \mc\mr_t^{1/2}\mh\mr_t^{1/2}\mc.
\end{equation}
These definitions are chosen so that
\begin{equation}
\label{eq:resolvent-factorization}
    t\id_d+\mc(\id_d+\mh)\mc
    =
    \mr_t^{-1/2}(\id_d+\mt_t)\mr_t^{-1/2}.
\end{equation}

We are now ready to state the main export of this subsection.

\begin{proposition}
\label{prop:trace-resolvent-expansion}
Let $\ma, \mb \in \PSD^{d \times d}$ and $\mh \in \Sym^{d \times d}$. For every $k \in \N$, define
\begin{equation}\label{eq:order_k}Q_{k, \mc}(\mh) \defeq \frac{2(-1)^k}{\pi} \int_0^\infty \Tr\Par{\mr_t \mt_t^k}t^{1/2}\dd t,\end{equation}
following the notation \eqref{eq:bh_def}, \eqref{eq:rtdef}. Then for all $k \ge 2$,
\[|Q_{k, \mc}(\mh)| \le Q_{2, \mc}(\mh)\normsop{\mh}^{k - 2}.\]
Further, if $\normop{\mh} < 1$,
\begin{equation}
\label{eq:trace-resolvent-series}
\begin{aligned}
\BW^2\Par{\ma, \mb_{\mh}} - \BW^2\Par{\ma, \mb} = \inprod{\mb-\mc}{\mh} + \sum_{k \ge 2} Q_{k, \mc}(\mh),
\end{aligned}
\end{equation}
and
\[\Abs{R_{k, \mc}(\mh)} \le \frac{Q_{2, \mc}(\mh)\normop{\mh}^{k - 1}}{1 - \normop{\mh}},\text{ where } R_{k, \mc}(\mh) \defeq \sum_{k' > k} Q_{k', \mc}(\mh). \]
\end{proposition}

Before proving Proposition~\ref{prop:trace-resolvent-expansion}, we first characterize the first two terms of the series \eqref{eq:order_k}.

\begin{lemma}\label{lem:first_two_terms}
In the setting of Proposition~\ref{prop:trace-resolvent-expansion}, let $\mc$ have eigendecomposition $\muu\mlam\muu^\top$ where $\mlam = \diag{\vlam}$, and let $\tmh \defeq \muu^\top\mh\muu$. Then letting $\frac{ab}{a+b} \defeq 0$ if $a = b = 0$,
\begin{equation}\label{eq:second_order}
Q_{1, \mc}(\mh) = -\inprod{\mc}{\mh},\quad 
Q_{2, \mc}(\mh) = \frac 1 4\sum_{i \in [d]} \vlam_i \tmh_{ii}^2 + \sum_{1 \le i < j \le d} \frac{\vlam_i\vlam_j}{\vlam_i + \vlam_j}\tmh_{ij}^2.\end{equation}
Moreover, following the notation \eqref{eq:nonlinear-covariance-term},
\begin{equation}\label{eq:2c_bound}0 \le Q_{2, \mc}\Par{\mh} \le \frac{\normop{\mc}}{4}\normf{\mh}^2.\end{equation}
\end{lemma}
\begin{proof}
We first directly expand
\begin{equation}\label{eq:first_order_term}
\begin{aligned}\Tr\Par{\mr_t \mt_t} &= \inprod{\muu\Par{t\id_d + \mlam^2}^{-1}\muu^\top}{\muu\Par{\mlam\sqrt{\Par{t\id_d + \mlam^2}^{-1}}}\tmh\Par{\mlam\sqrt{\Par{t\id_d + \mlam^2}^{-1}}}\muu^\top} \\
&= \inprod{\mlam^2\Par{t\id_d + \mlam^2}^{-2}}{\tmh} = \sum_{i \in [d]} \tmh_{ii} \frac{\vlam_i^2}{(t + \vlam_i^2)^2}.
\end{aligned}
\end{equation}
Hence, exchanging the sum and integral, and using $\int_0^\infty \frac{t^{1/2}s^2}{(t + s^2)^2} \dd t = \frac{s\pi}{2}$, gives
\begin{align*}
Q_{1, \mc}(\mh) &= \frac{-2}{\pi} \int_0^\infty \Tr\Par{\mr_t \mt_t} t^{1/2}\dd t   \\
&= -\frac 2 \pi \sum_{i \in [d]} \tmh_{ii}\int_0^\infty \frac{t^{1/2}\vlam_i^2}{(t + \vlam_i^2)^2} \dd t = -\frac 2 \pi \sum_{i \in [d]} \tmh_{ii} \Par{\frac{\vlam_i\pi}{2}} = -\inprod{\mc}{\mh}.
\end{align*}
Next, a similar calculation to \eqref{eq:first_order_term} yields
\[\Tr\Par{\mr_t \mt_t^2} =  \sum_{(i, j) \in [d] \times [d]} \frac{\vlam_i^2 \vlam_j^2}{\Par{t + \vlam_i^2}^2\Par{t+\vlam_j^2}} \tmh_{ij}^2.  \]
Thus, applying the identities $\int_0^\infty \frac{s^4t^{1/2}}{(t + s^2)^3} \dd t = \frac{\pi s}{8}$ and $\int_0^\infty \frac{r^2s^2t^{1/2}}{(t + r^2)^2(t + s^2)} \dd t = \frac{\pi rs^2}{2(r + s)^2}$ gives 
\begin{align*}
Q_{2, \mc}(\mh) &= \frac 2 \pi \int_0^\infty \Tr(\mr_t\mt_t^2) t^{1/2}\dd t \\
&=  \frac 1 4 \sum_{i \in [d]} \vlam_i \tmh_{ii}^2 + \sum_{1 \le i < j \le d} \Par{\frac{\vlam_i\vlam_j^2 + \vlam_j\vlam_i^2}{(\vlam_i + \vlam_j)^2}  }\tmh_{ij}^2,
\end{align*}
proving \eqref{eq:second_order}. Finally, we have that $0 \le \vlam_i \le \normsop{\mc}$ and $\frac{\vlam_i\vlam_j}{\vlam_i + \vlam_j} \le \half\normsop{\mc}$. Applying these bounds termwise to \eqref{eq:second_order} and using $\normsf{\tmh}^2 = \normsf{\mh}^2$ gives the first claim in \eqref{eq:2c_bound}.
\end{proof}

We now complete our proof of Proposition~\ref{prop:trace-resolvent-expansion}.

\begin{proof}[Proof of Proposition~\ref{prop:trace-resolvent-expansion}]
Because $\normsop{\mc\mr_t^{1/2}} \le 1$,
\begin{equation}\label{eq:tt_bound}
\normop{\mt_t} \le \normop{\mh} \text{ for all } t \ge 0,
\end{equation}
so that for any $k \ge 2$,
\[\Abs{\Tr\Par{\mr_t \mt_t^k}} \le \normop{\mt_t}^{k - 2}\Tr\Par{\mr_t \mt_t^2 }\le \normop{\mh}^{k - 2} \Tr\Par{\mr_t\mt_t^2}.\]
Integrating the above display gives $|Q_{k, \mc}(\mh)| \le Q_{2, \mc}(\mh)\normsop{\mh}^{k - 2}$. Next, assume that $\normop{\mh} < 1$, so that \eqref{eq:tt_bound} also gives $\normsop{\mt_t} < 1$.
Then expanding $(\id_d + \mt_t)^{-1} = \id_d + \sum_{k \ge 1} (-1)^k \mt_t^k$,
\begin{align*}
\Tr\paren{
            \paren{\mc(\id_d+\mh)\mc}^{1/2}
        }
        -\Tr(\mc) &= \frac 1 \pi \sum_{k \ge 1} (-1)^{k + 1} \int_0^\infty \Tr(\mr_t \mt_t^k) t^{1/2} \dd t \\
        &= -\half\sum_{k \ge 1} Q_{k, \mc}(\mh).
\end{align*}
Because Lemma~\ref{lem:first_two_terms} shows $Q_{1, \mc}(\mh) = -\inprod{\mc}{\mh}$, $\Phi_{\mc}(\mh)$ defined in \eqref{eq:nonlinear-covariance-term} equals $\sum_{k \ge 2} Q_{k, \mc}(\mh)$. The claim \eqref{eq:trace-resolvent-series} now follows from Lemma~\ref{lem:first_order_expand}. Finally,
\[|R_{k, \mc}(\mh)| \le \sum_{k' > k} Q_{2, \mc}(\mh)\normsop{\mh}^{k' - 2} = \frac{Q_{2, \mc}(\mh)\normop{\mh}^{k - 1}}{1 - \normop{\mh}}.\]
\end{proof}

We conclude with a convenient bound on $\Phi_{\mc}(\mh)$ used in our later analysis. Importantly, this bound will hold for any $\mh$ with $\id_d + \mh \in \PSD^{d \times d}$ irrespective of an operator norm bound.

\begin{lemma}\label{lem:phi_bound}
Let $\mc \in \PSD^{d \times d}$ and $\mh + \id_d \in \PSD^{d \times d}$. Then following the notation \eqref{eq:nonlinear-covariance-term}, 
\[\Phi_{\mc}(\mh) \le \normop{\mc}\normf{\mh}^2. \]
\end{lemma}
\begin{proof}
Taking $\mm \gets \mc^{1/2}$ and $\mn \gets \mc^{1/2}(\id_d + \mh)^{1/2}$ in \eqref{eq:bures-frobenius} implies
\begin{align*}\Phi_{\mc}(\mh) &= \BW^2\Par{\mc, \mc^{1/2}(\id_d + \mh)\mc^{1/2}} \le \normf{\mc^{1/2} - \mc^{1/2}(\id_d + \mh)^{1/2}}^2 \\
&\le \normop{\mc}\normf{\id_d - \Par{\id_d + \mh}^{1/2}}^2 \le \normop{\mc}\normf{\mh}^2, \end{align*}
where the last line used that $\id_d - (\id_d + \mh)^{1/2} \in \Sym^{d \times d}$, and $| 1- \sqrt{1 + x}| \le |x|$ for all $x \ge -1$.
\end{proof}

\subsection{Analyzing the plug-in estimator}
\label{ssec:first-order-plugin-expansion}

We now specialize Proposition~\ref{prop:trace-resolvent-expansion} to the setting of the plug-in estimator defined in~\Cref{sec:intro}.
We recall the definition of the plug-in estimator here for convenience.
Let $\hvnu_n$ and $\hmb_n$ denote the empirical mean and covariance with $n$ samples:
\begin{equation}\label{eq:empirical}\hvnu_n \defeq \frac 1 n \sum_{i \in [n]} \vX_i,\quad \hmb_n \defeq \frac 1 {n} \sum_{i \in [n]} (\vX_i - \hvnu_n)(\vX_i - \hvnu_n)^\top. \end{equation}
As mentioned previously, in the setting of Model~\ref{model:one-side}, the plug-in estimator outputs
\begin{equation}\label{eq:plugin}
\widehat{\dd}_n \defeq \norm{\vmu - \hvnu_n}_2^2 + \BW^2\Par{\ma, \frac n {n - 1}\hmb_n}.
\end{equation}
The rescaling of $\hmb_n$ in \eqref{eq:plugin} is justified by the following classical fact.
\begin{fact}[\cite{muirhead2009aspects}, Chapter 3]\label{fact:rescaling}
In the setting of Model~\ref{model:one-side}, following the notation \eqref{eq:empirical}, we have $\hvnu_n \sim \calN(\nu, \frac 1 n \mb)$, $\hvnu_n$ and $\hmb_n$ are independent, and for $m \defeq n - 1$ and $\id_d + \mh_m \sim \Wish(\frac 1 m \id_d, m)$,
\[\frac{n}{n-1}\hmb_n \overset{d}{=} \mb^{1/2} \Par{\id_d + \mh_m}\mb^{1/2}.\]
\end{fact}
Recalling that $\mm \sim \Wish(\id_d, m)$ satisfies $\E[\mm] = m\id_d$, Fact~\ref{fact:rescaling} shows that $\frac{n}{n-1}$ is the correct scaling of the empirical covariance $\hmb_n$ for it to be unbiased for $\mb$. Further, after whitening by $\mb^{1/2}$, the random perturbation $\mh_m$ is exactly of the form considered by Section~\ref{ssec:gaussian-wishart-prelim}.

We can now determine the asymptotic bias and variance of the plug-in estimator \eqref{eq:plugin}.

\begin{theorem}
\label{thm:first-order-plugin-bias-variance}
Under Model~\ref{model:one-side} and Assumption~\ref{assum:normalization}, the plug-in estimator $\hdd_n$ in \eqref{eq:plugin} satisfies
\begin{equation}\label{eq:bias_variance_plugin}
    \Abs{\E [\hdd_n] - \FID(P, Q)}
    =
    O\Par{\frac{d^2}{n}},\quad \Var[\hdd_n] = O\Par{\frac{d}{n} + \frac{d^2}{n^2}}.
\end{equation}
Consequently, if $n \gtrsim \frac{d}{\eps^2} + \frac{d^2}{\eps}$, the plug-in estimator solves Problem~\ref{prob:FD_est}.
\end{theorem}

\begin{proof}
Throughout, let $m \defeq n - 1$, $\vv \defeq \vnu - \vmu$, and $\hvv_n \defeq \hvnu_n - \vmu$. By Lemma~\ref{lem:first_order_expand},
\[\hdd_n - \FID(P, Q) = 2\inprod{\vv}{\hvv_n - \vv} + \norm{\hvv_n - \vv}_2^2 + \inprod{\mb - \mc}{\mh_m} + \Phi_{\mc}(\mh_m). \]
Moreover, by Fact~\ref{fact:rescaling}, $\hvv_n \sim \calN(\vv, \frac 1 n \mb)$, so by taking expectations,
\begin{equation}\label{eq:decompose_diff}\E[\hdd_n] - \FID(P, Q) = \frac{\Tr(\mb)}{n} + \E[\Phi_{\mc}(\mh_m)] \le \frac d n + \E[\Phi_{\mc}(\mh_m)].\end{equation}
To bound the last term, Assumption~\ref{assum:normalization} implies $\normop{\mc} \le 1$, so Lemma~\ref{lem:phi_bound} and Fact~\ref{fact:wishart_1} now give
\begin{equation}\label{eq:phi_expect_bound}\E[\Phi_{\mc}(\mh_m)] \le \E \normf{\mh_m}^2 = \frac{d(d+1)}{m} = O\Par{\frac{d^2}{n}}.\end{equation}
We next bound the variance. By Fact~\ref{fact:rescaling}, $\hvv_n$ and $\mh_m$ are independent, so
\[\Var[\hdd_n] = \Var\Brack{2\inprod{\vv}{\hvv_n - \vv} + \norm{\hvv_n - \vv}_2^2} + \Var\Brack{\inprod{\mb - \mc}{\mh_m} + \Phi_{\mc}(\mh_m)}.\]
By standard Gaussian moment identities (e.g., \cite{muirhead2009aspects}, Chapter 1), we have
\[\Var\Brack{2\inprod{\vv}{\hvv_n - \vv} + \norm{\hvv_n - \vv}_2^2} = \frac{4\vv^\top \mb \vv}{n} + \frac{2\Tr(\mb^2)}{n^2} = O\Par{\frac{1}{n} + \frac d {n^2}}.\]
Moreover, using $(a + b)^2 \le 2a^2 + 2b^2$ and linearity of expectation gives
\begin{align*}\Var\Brack{\inprod{\mb - \mc}{\mh_m} + \Phi_{\mc}(\mh_m)} &\le 2\Var\Brack{\inprod{\mb - \mc}{\mh_m}} + 2\Var\Brack{\Phi_{\mc}(\mh_m)},
\end{align*}
and Fact~\ref{fact:wishart_1} shows
\[\Var\Brack{\inprod{\mb - \mc}{\mh_m}} = \frac 2 m \normf{\mb - \mc}^2 = O\Par{\frac d n}.\]
It remains to control $\Var[\Phi_{\mc}(\mh_m)]$. For $\mg \in \R^{d \times m}$ with i.i.d.\ entries $\sim \calN(0, 1)$, let
\[\mh(\mg) \defeq \frac 1 m \mg\mg^\top - \id_d,\quad \mn(\mg) \defeq \begin{pmatrix} \frac 1 {\sqrt m} \mc^{1/2}\mg & \0_{d \times \max(0, d-m)} \end{pmatrix}.\]
Then $\mh_m \overset{d}{=} \mh(\mg)$, and $\mn(\mg)\mn(\mg)^\top = \mc^{1/2}(\id_d + \mh(\mg))\mc^{1/2}$. Now by \eqref{eq:bures-frobenius},
\[\sqrt{\Phi_{\mc}(\mh(\mg))} = \BW(\mc, \mc^{1/2}(\id_d + \mh(\mg))\mc^{1/2})  = \min_{\mm \in \R^{d \times \max(m, d)}: \mm\mm^\top = \mc } \normf{\mm - \mn(\mg)}.\]
The triangle inequality then shows that for any $\mg, \mg' \in \R^{d \times m}$,
\[\Abs{\sqrt{\Phi_{\mc}(\mh(\mg))} - \sqrt{\Phi_{\mc}(\mh(\mg'))}} \le \normf{\mn(\mg) - \mn(\mg')} \le \frac 1 {\sqrt m}\normf{\mg - \mg'},\]
so $\mg \to \sqrt{\Phi_{\mc}(\mh(\mg))}$ has a Lipschitz constant bounded by $m^{-1/2}$. Thus,
\[\normf{\nabla_{\mg} \Phi_{\mc}(\mh(\mg))}^2 = 4\Phi_{\mc}(\mh_m)\normf{\nabla_{\mg} \sqrt{\Phi_{\mc}(\mh(\mg))}}^2 \le \frac 4 m \Phi_{\mc}(\mh(\mg)).\]
The Gaussian Poincar\'e inequality (Section 3.7, \cite{BLM13}) now yields
\begin{equation}\label{eq:apply_pi}\Var\Brack{\Phi_{\mc}(\mh(\mg))} \le \frac 4 m \E\Brack{ \Phi_{\mc}(\mh(\mg)) }. \end{equation}
Finally, applying $\mh_m \overset{d}{=}\mh(\mg)$ and our earlier bound \eqref{eq:phi_expect_bound} shows $\Var[\Phi_{\mc}(\mh_m)] = O(\frac{d^2}{n^2})$, concluding the proof of \eqref{eq:bias_variance_plugin}. The claim about Problem~\ref{prob:FD_est} follows from Chebyshev's inequality.\end{proof}

Theorem~\ref{thm:first-order-plugin-bias-variance} shows that when $d \gg \eps^{-1}$, the leading-order term in the sample complexity arises because of the \emph{bias} of the plug-in estimator $\hdd_n$, not its variance. In fact, \eqref{eq:phi_expect_bound} shows that the dominant bias term is the error in the nonlinear estimator $\Phi_{\mc}(\mh_m) \approx \Phi_{\mc}(\mh)$. In Section~\ref{sec:extrapolation}, we improve the sample complexity of FID estimation by using high-order corrections to debias the plug-in estimator.

\section{Estimation via Extrapolation}
\label{sec:extrapolation}

Theorem~\ref{thm:first-order-plugin-bias-variance} shows that the suboptimal sample complexity of the plug-in estimator arises from its large bias at sample sizes $n \ll d^2$. In this section, we study \emph{extrapolation} as a principled strategy for debiasing FID estimation, by generalizing 
the $\fidinf$ estimator of \cite{chong2019effectively} to higher orders. In Section~\ref{ssec:extrapolation}, we first give a generic framework for bounding the bias of estimators constructed via extrapolation. In Section~\ref{ssec:higher-order-extrapolation}, we specialize this framework to FID estimation, and in Section~\ref{ssec:choosing-extrapolation-weights}, we discuss strategies for choosing extrapolation weights in practice. Finally, we demonstrate sharpness of our bias bound from Section~\ref{ssec:higher-order-extrapolation} at every fixed order in Section~\ref{ssec:extrapolation-bias-lower-bound}.

\subsection{Extrapolation framework}\label{ssec:extrapolation}

In this subsection, we derive our generic framework for analyzing extrapolation methods. Let $D \in \R$ be an arbitrary scalar-valued parameter to be estimated, and let $\{\hD_n\}_{n \in \N}$ be a sequence of finite-sample estimators of $D$. 
Recall from~\Cref{sec:intro} that the premise of extrapolation is that, although each $\hD_n$ may be biased at a fixed $n$, if its bias can be approximated by a polynomial in $\frac 1 n$, i.e.,
\[\E[\hD_n] \approx D + \frac{\va_1}{n} + \frac{\va_2}{n^2} + \ldots + \frac{\va_k}{n^k},\]
then by evaluating $\hD_n$ at different sample sizes $n$ and taking an appropriate linear combination, we can cancel out low-order terms in $\frac 1 n$ and substantially improve the bias. 

We now formalize the finite-order approximation required by our framework.

\begin{assumption}
\label{assum:polynomial-expectation-expansion}
Let $\{\hD_n\}_{n \in \N}$ be a sequence of random variables. We assume that for every fixed $k \in \N$, there exist
$\va \in \R^k$, $A_{k+1}\ge0$, and
$n_0(k) \in \N$ independent of $n$, such that for every $n\ge n_0(k)$,
\begin{equation}
\label{eq:polynomial-expectation-assumption}
    \bbE[\hD_n]
    =
    D+\sum_{i \in [k]} \frac{\va_i}{n^i}+R_{k,n},
    \quad
    \Abs{R_{k,n}}
    \le\frac{A_{k+1}}{n^{k+1}}.
\end{equation}
\end{assumption}

In Section~\ref{ssec:higher-order-extrapolation}, we verify that Assumption~\ref{assum:polynomial-expectation-expansion} holds for each fixed order $k$, where we take $\hD_n \gets \hdd_n$ to be the plug-in estimator \eqref{eq:plugin}, and where $D \gets \FID(P, Q)$ under Model~\ref{model:one-side}. 

We now define our generic order-$k$ extrapolation method, which generalizes the setup described in~\Cref{sec:intro}. 
The method takes parameters $\ell \ge k + 1$, a sequence of sample sizes $\vn \in \N^{\ell}$, and deterministic weights $\vw \in \calW_k (\vc)$, where $\calW_k (\vc)$ is defined in~\Cref{eq:order-k-exact-weights}. 
Our estimator is simply
\begin{equation}
\label{eq:extrapolation-estimator}\hD\Par{\vn, \vw} \defeq \sum_{j \in [\ell]} \vw_j \hD_{n_j}, \end{equation}
where each $\hD_{n_j}$ is computed independently from $n_j$ samples.
As before, we let
\begin{equation}\label{eq:cdef}N \defeq \sum_{j \in [\ell]} \vn_j,\quad \vc_j \defeq \frac{\vn_j}{N} \text{ for all } j \in [\ell],\end{equation}

We now prove our main result on the bias of \eqref{eq:extrapolation-estimator} under the assumption $\vw \in \calW_k(\vc)$.

\begin{lemma}
\label{lem:extrapolation-bias-cancellation}
Fix $k \ge 1$, $\ell \ge k + 1$, $\vn \in \N^\ell$, and $\vw \in \R^\ell$, and define $\vc \in (0, 1)^\ell$ as in \eqref{eq:cdef}. Under Assumption~\ref{assum:polynomial-expectation-expansion} and following the notation \eqref{eq:extrapolation-estimator}, \eqref{eq:order-k-exact-weights}, if $\min_{j \in [\ell]}\vn_j \ge n_0(k)$ and $\vw \in \calW_k(\vc)$,
\begin{equation*}
    \Abs{\bbE\Brack{\hD(\vn, \vw)}-D}
    \le
    \frac{A_{k+1}}{N^{k+1}}
    \sum_{j\in[\ell]} \frac{\Abs{\vw_j}}{\vc_j^{k+1}}
    =
    O_{\vc,\vw,k}\paren{\frac{A_{k+1}}{N^{k+1}}}.
\end{equation*}
\end{lemma}

\begin{proof}
Applying \eqref{eq:polynomial-expectation-assumption} at every $\vn_i$ and summing gives
\begin{equation*}
\begin{aligned}
    \bbE\Brack{\hD(\vn, \vw)} -D
    &=
    D\paren{\sum_{j \in [\ell]}\vw_j-1}
    +\sum_{i\in[k]} \frac{\va_i}{N^i}
        \sum_{j \in [\ell]} \frac{\vw_j}{\vc_j^i}
    +\sum_{j \in [\ell]} \vw_jR_{k,\vn_j} =\sum_{j\in[\ell]} \vw_j R_{k,\vn_j}.
\end{aligned}
\end{equation*}
The remainder bound in Assumption~\ref{assum:polynomial-expectation-expansion} and the triangle inequality prove the claim.
\end{proof}

The extrapolation estimator \eqref{eq:extrapolation-estimator} analyzed in this section is with respect to independent samples at each level. In Appendix~\ref{ssec:shared-pool-extrapolation}, we extend our analysis to subsampling from a shared pool, an alternative strategy used in our experiments (as well as the official implementation of \cite{chong2019effectively}).

\subsection{Verifying Assumption~\ref{assum:polynomial-expectation-expansion} for the plug-in estimator}
\label{ssec:higher-order-extrapolation}

We now specialize the machinery of Section~\ref{ssec:extrapolation} to the setting of the plug-in estimator. Throughout the rest of this section, we fix $D \gets \FID(P, Q)$ under Model~\ref{model:one-side}, and $\hD_n \gets \hdd_n$ as defined in \eqref{eq:plugin}.

The main result of this subsection is the following.

\begin{proposition}
\label{prop:uniform-psd-expansion}
The plug-in estimator $\hdd_n$ in \eqref{eq:plugin} satisfies Assumption~\ref{assum:polynomial-expectation-expansion} with
\begin{equation}
\label{eq:general-expectation-expansion}
\begin{aligned}
|\va_j| \le C_k d^{j + 1} \text{ for all } j \in [k],\quad A_{k + 1} = C_k d^{k + 2},\quad n_0(k) = C_k d,
\end{aligned}
\end{equation}
where $C_k > 0$ is a constant that depends only on $k$.
\end{proposition}

Before proving Proposition~\ref{prop:uniform-psd-expansion}, we briefly use it to derive a sample complexity bound for (one instance of) the $\fidinf$ estimator \cite{chong2019effectively}, which corresponds to the $k = 1$ case of Lemma~\ref{lem:extrapolation-bias-cancellation}.

\begin{corollary}\label{cor:fidinf}
Under Model~\ref{model:one-side} and Assumption~\ref{assum:normalization}, taking $\ell = 2$ and
\[\vn = \Par{\frac N 3, \frac{2N}{3}},\quad \vw = \Par{-1, 2},\]
in \eqref{eq:extrapolation-estimator} yields $\hD(\vn, \vw)$ that solves Problem~\ref{prob:FD_est} if $N \gtrsim \frac d {\eps^2} + \frac{d^{3/2}}{\eps^{1/2}}$.
\end{corollary}
\begin{proof}
This setting corresponds to $\vc_1 = \frac 1 3$, $\vc_2 = \frac 2 3$, so  $\vw \in \calW_1(\vc)$. Thus Lemma~\ref{lem:extrapolation-bias-cancellation} applies, so Proposition~\ref{prop:uniform-psd-expansion} shows that the bias of $\hD(\vn, \vw)$ is $O(\frac{d^3}{N^2})$. Moreover, Theorem~\ref{thm:first-order-plugin-bias-variance} and independence of $\hdd_{\vn_1}$ and $\hdd_{\vn_2}$ in \eqref{eq:extrapolation-estimator} shows that the variance of $\hD(\vn, \vw)$ is $O(\frac{d}{N} + \frac{d^2}{N^2})$. Combining with our lower bound on $N$, and applying Chebyshev's inequality, gives the claim.
\end{proof}

To improve upon the sample complexity of Corollary~\ref{cor:fidinf}, we verify that Assumption~\ref{assum:polynomial-expectation-expansion} holds for $k > 1$. In Lemma~\ref{lem:wishart-collision-expansion}, we first give an exact polynomial characterization of each homogeneous term $\E Q_{r, \mc}(\mh_m)$ \eqref{eq:order_k} arising in the decomposition \eqref{eq:trace-resolvent-series}. Notably, a uniform bound of the form
\[\Abs{\bbE Q_{r,\mc}(\mh_m)}
    \lesssim_r
    d\paren{\frac dm}^{\lceil \frac r 2\rceil},\]
    while true, is not sufficient to apply Assumption~\ref{assum:polynomial-expectation-expansion}, which requires separating out the dependences on each inverse polynomial in $m$, with bounds on the coefficients at each order.
\begin{lemma}
\label{lem:wishart-collision-expansion}
Under Assumption~\ref{assum:normalization}, and following the notation of \eqref{eq:order_k} and Fact~\ref{fact:rescaling},
for every fixed $r\ge2$, there exist coefficients
$b_{r,j}(\mc)$, independent of $m$, such that
\begin{equation}
\label{eq:wishart-homogeneous-expansion}
    \bbE Q_{r,\mc}(\mh_m)
    =
    \sum_{j=\lceil r/2\rceil}^{r-1}
    \frac{b_{r,j}(\mc)}{m^j},
    \quad
    \Abs{b_{r,j}(\mc)}
    \lesssim_r d^{j+1}.
\end{equation}
\end{lemma}

\begin{proof}
Throughout, let $\{\vg_a \in \R^d\}_{a \in [m]}$ have i.i.d.\ $\calN(0, 1)$ entries, and let
\begin{equation*}
    \mh_m
    =
    \frac1m\sum_{a \in [m]} \my_a,
    \quad
    \my_a
    \defeq
    \vg_a\vg_a^\top-\id_d.
\end{equation*}
Since $Q_{r,\mc}$ is homogeneous of degree $r$, expanding the $r$
copies of $\mh_m$ gives
\begin{equation}
\label{eq:homogeneous-sample-expansion}
    Q_{r,\mc}(\mh_m)
    =
    \frac1{m^r}
    \sum_{a_1,\ldots,a_r\in[m]}
    \mathcal Q_{r,\mc}
    (\my_{a_1},\ldots,\my_{a_r}),
\end{equation}
where $\mathcal Q_{r,\mc}$ is multilinear in its arguments.
Concretely, it is obtained from the integral \eqref{eq:order_k} defining $Q_{r,\mc}$
by replacing the $r$ copies of $\mh$ in $\mt_t^r$ by the corresponding
arguments $\my_{a_1}, \ldots, \my_{a_r}$.

Each tuple $(a_1,\ldots,a_r)$ of sample labels induces a partition
$\pi$ of $[r]$: positions $s,t\in[r]$ belong to the same block
precisely when $a_s=a_t$. For example, when $r=4$, a tuple
$(a,a,b,b)$ with $a\ne b$ induces
$\pi=\cbra{\cbra{1,2},\cbra{3,4}}$, whereas $(a,a,a,b)$ induces
$\cbra{\cbra{1,2,3},\cbra{4}}$.

If a sample label occurs only once,
conditioning on all other samples leaves an expression linear in
the corresponding $\my_a$. Independence and
$\bbE\my_a=\vzero$ make this conditional expectation zero, so only partitions with every block having size at least two can contribute.

Let $\mathcal P_r^{\ge2}$ be the collection of such partitions of $[r]$.
For $\pi\in\mathcal P_r^{\ge2}$, let $\mathcal C_{\pi,\mc}$ be the
expectation associated with any fixed assignment of distinct sample
labels to its blocks. For example,
\begin{equation*}
    \mathcal C_{\cbra{\cbra{1,2},\cbra{3,4}},\mc}
    \defeq
    \bbE\mathcal Q_{4,\mc}
    (\my_1,\my_1,\my_2,\my_2).
\end{equation*}
Since the $\my_a$ are i.i.d., these quantities depend only on $\pi$
and $\mc$, rather than the labels or $m$. If $\pi$ has $s=\Abs{\pi}$ blocks, there are exactly
\begin{equation*}
    (m)_s\defeq m(m-1)\cdots(m-s+1)
\end{equation*}
ways to assign $s$ distinct sample labels to its blocks. Hence
\begin{equation}
\label{eq:collision-partition-identity}
    \bbE Q_{r,\mc}(\mh_m)
    =
    \frac1{m^r}
    \sum_{\pi\in\mathcal P_r^{\ge2}}
    (m)_{\Abs{\pi}}\,\mathcal C_{\pi,\mc}.
\end{equation}

Every index appears at least twice, so
$\Abs{\pi}\le\lfloor \frac r 2\rfloor$. Moreover, $(m)_s$ is a degree-$s$
polynomial with no constant term. Thus $\frac{(m)_s}{m^r}$ contains only
powers $m^{-j}$ with $r-s\le j\le r-1$. Collecting equal powers
in~\eqref{eq:collision-partition-identity} gives the first claim in \eqref{eq:wishart-homogeneous-expansion}. 
To control the magnitude of the coefficients, define
\begin{equation}\label{eq:poly_def}
    p(z)
    \defeq 
    \sum_{j= \lceil r/2 \rceil}^{r-1}
    \frac{b_{r,j}(\mc)}{d^{j+1}}z^{j},
\end{equation}
Then $z^{-\lceil \frac r 2 \rceil} p(z)$ is a degree-$(L - 1)$ polynomial, where $L \defeq \lfloor \frac r 2 \rfloor$. By construction,
\begin{equation}
\label{eq:normalized-homogeneous-polynomial}
    \bbE Q_{r,\mc}(\mh_m)
    =
    dp\paren{\frac dm}.
\end{equation}
Thus, bounds on the coefficients of $p$ directly translate to bounds on the $b_{r, j}(\mc)$. We bound the coefficients of $p$ by evaluating it at $L$ fixed interpolation points $\{m_i\}_{ i\in [L]}$. Choose an integer $c_r$ large enough such that if $m \ge c_r d$, Fact~\ref{fact:wishart_2} applies at $q \gets 0$ and $q \gets r$. Then for each $i \in [L]$, we set
\[m_i \defeq ic_r d,\quad z_i \defeq \frac{d}{m_i} = \frac 1 {ic_r}. \]
Then we have for each $i \in [L]$,
\begin{align*}
\Abs{\E Q_{r, \mc}(\mh_{m_i})} &\le \frac 1 4 \E\Brack{\normf{\mh_{m_i}}^2\normop{\mh_{m_i}}^{r - 2}} \\
&\le \frac 1 4 \E \normf{\mh_{m_i}}^2 + \frac 1 4 \E\Brack{\normf{\mh_{m_i}}^2\normop{\mh_{m_i}}^{2r}} \\
&\lesssim_r \frac{d^2}{m_i} + \frac{d^{r + 2}}{m_i^{r + 1}} \lesssim_r d,
\end{align*}
where the first inequality used Proposition~\ref{prop:trace-resolvent-expansion}, \eqref{eq:2c_bound}, and Assumption~\ref{assum:normalization}, the second inequality used $x^{r - 2} \le 1 + x^{2r}$ for all $x \ge 0$, the third inequality applied Fact~\ref{fact:wishart_2} at $q \gets 0$ and $q \gets r$, and the last substituted in our choice of $m_i$. Combined with \eqref{eq:normalized-homogeneous-polynomial}, this gives
\[\Abs{z_i^{-\lceil \frac r 2 \rceil}p(z_i)} \lesssim_r 1,\text{ for all } i \in [L].\]
Each interpolation point $z_i$ is a fixed nonzero constant depending only on $r$, so inverting the Vandermonde matrix in the $\{z_i\}_{i \in [L]}$ to recover coefficients incurs a blowup that is only a function of $r$. After shifting by $z^{\lceil \frac r 2\rceil}$, we conclude that every coefficient of $p$ is $\lesssim_r 1$ in magnitude.
Finally, \eqref{eq:poly_def} shows that each $b_{r, j}(\mc)$ is $d^{j + 1}$ times larger than the degree-$j$ coefficient of $p$, yielding \eqref{eq:wishart-homogeneous-expansion}.
\end{proof}

We have thus shown that each $\E Q_{r, \mc}(\mh_m)$ has an exact polynomial expansion in $\frac 1 m$, with bounded coefficients. We next prove that truncating to only include finitely many orders $r$ suffices to approximate the entire nonlinear term $\Phi_{\mc}(\mh_m)$, defined in the decomposition \eqref{eq:nonlinear-covariance-term}.

\begin{lemma}\label{lem:truncation_ols}
Under Assumption~\ref{assum:normalization}, and following the notation of \eqref{eq:nonlinear-covariance-term}, \eqref{eq:order_k}, and Fact~\ref{fact:rescaling}, for every fixed $k \ge 1$, there is a constant $c_k > 0$ such that whenever $m \ge c_k d$,
\[ \E\Phi_{\mc}(\mh_m) = \sum_{r = 2}^{2k + 1} \E Q_{r, \mc}(\mh_m) + O_k\Par{\frac{d^{k + 2}}{m^{k + 1}}}.\]
\end{lemma}
\begin{proof}
We first claim that the following pointwise estimate holds for all realizations of $\mh_m$:
\begin{equation}\label{eq:residual_event}
\Abs{\Phi_{\mc}(\mh_m) - \sum_{r = 2}^{2k + 1} Q_{r, \mc}(\mh_m)} \lesssim_k \normf{\mh_m}^2 \normop{\mh_m}^{2k}.
\end{equation}
We split our proof based on $\calE \defeq \{\normsop{\mh_m} \le \half\}$. Under $\calE$, \eqref{eq:residual_event} follows from Proposition~\ref{prop:trace-resolvent-expansion}, \eqref{eq:2c_bound}, and Assumption~\ref{assum:normalization}. 
Under $\calE^c$, by Lemma~\ref{lem:phi_bound} and Assumption~\ref{assum:normalization}, we first have
\[0 \le \Phi_{\mc}(\mh_m) \le \normf{\mh_m}^2 \le 2^{2k} \normf{\mh_m}^2 \normop{\mh_m}^{2k}.\]
Moreover, for every $2 \le r \le 2k + 1$, Proposition~\ref{prop:trace-resolvent-expansion}, \eqref{eq:2c_bound}, and Assumption~\ref{assum:normalization} give
\[\Abs{Q_{r, \mc}(\mh_m)} \le \frac 1 4 \normf{\mh_m}^2\normop{\mh_m}^{r - 2} \le 2^{2k} \normf{\mh_m}^2 \normop{\mh_m}^{2k}.\]
Thus combining the above two displays with the triangle inequality shows \eqref{eq:residual_event}. Finally, taking expectations over \eqref{eq:residual_event} and applying Fact~\ref{fact:wishart_2}  with $q \gets k$ concludes the proof.
\end{proof}

We now put together the pieces to prove Proposition~\ref{prop:uniform-psd-expansion}.

\begin{proof}[Proof of Proposition~\ref{prop:uniform-psd-expansion}]
We follow the notation of Lemmas~\ref{lem:wishart-collision-expansion} and~\ref{lem:truncation_ols}. 
For each $j \in [k]$, define
\[\beta_j(\mc) \defeq \sum_{r = j+1}^{2j} b_{r, j}(\mc).\]
Then directly comparing with the statements of Lemmas~\ref{lem:wishart-collision-expansion} and~\ref{lem:truncation_ols} shows
\[\E \Phi_{\mc}(\mh_m) = \sum_{j \in [k]} \frac{\beta_j(\mc)}{m^j} + O_k\Par{\frac{d^{k + 2}}{m^{k + 1}}},\quad |\beta_j(\mc)| \lesssim_k d^{j + 1}.\]
Next, the decomposition \eqref{eq:decompose_diff} gives
\[\E\hdd_n = \FID(P, Q) + \frac{\Tr(\mb)}{n} + \E\Phi_{\mc}(\mh_m).\]
We are now essentially done with the proof, except that Assumption~\ref{assum:polynomial-expectation-expansion} requires a polynomial in $\frac 1 n$ rather than in $\frac 1 m$. To fix this, we use the negative-binomial expansion identity:
\[\frac 1 {m^j} = \frac 1 {(n - 1)^j} =  \sum_{i=j}^\infty \binom{i-1}{j-1} \frac 1 {n^i}= \sum_{i = j}^k \binom{i-1}{j-1} \frac 1 {n^i} + O_k\Par{\frac 1 {n^{k + 1}}},\]
so that Assumption~\ref{assum:polynomial-expectation-expansion} now holds with
\[\va_i \defeq \ind_{i = 1}\Tr(\mb) + \sum_{j \in [i]} \binom{i-1}{j-1} \beta_j(\mc),\quad A_{k + 1} = O(d^{k + 2}).\]
At this point substituting our bounds $\Tr(\mb) \le d$, $|\beta_j(\mc)| \lesssim_k d^{j + 1}$ gives the conclusion.
\end{proof}

Combining Proposition~\ref{prop:uniform-psd-expansion} with Lemma~\ref{lem:extrapolation-bias-cancellation} now gives our high-order generalization of Corollary~\ref{cor:fidinf}.

\begin{theorem}
\label{thm:general-order-extrapolation}
Under Model~\ref{model:one-side} and Assumption~\ref{assum:normalization}, let $k \in \N$ be fixed. Taking $\ell \ge k + 1$, $\vc \in (0, 1)^\ell$ with $\norm{\vc}_1 = 1$, $\vn \defeq N\vc$, and $\vw \in \calW_k(\vc)$ in \eqref{eq:extrapolation-estimator} yields $\hD(\vn, \vw)$ that solves Problem~\ref{prob:FD_est} if
\[N \gtrsim_{k, \vc, \vw} \frac{d}{\eps^2} + d\Par{\frac d \eps}^{\frac 1 {k + 1}}.\]
\end{theorem}
\begin{proof}
The proof is completely analogous to Corollary~\ref{cor:fidinf}, taking a general $k$ in Proposition~\ref{prop:uniform-psd-expansion}.
\end{proof}

\subsection{Choosing extrapolation weights}
\label{ssec:choosing-extrapolation-weights}

Theorem~\ref{thm:general-order-extrapolation}, our high-order extension to $\fidinf$ (Corollary~\ref{cor:fidinf}), requires a sample complexity whose dependence on $d$ and $\eps$ is fixed for any choice of the extrapolation parameters $\vc, \vw$. In this subsection, we describe and motivate several concrete choices of $\vc, \vw$ used in our empirical evaluations.

\textbf{Polynomial regression.} Our starting point is the relationship
\begin{equation}\label{eq:fid_regress}
\begin{aligned}
    \E \hdd_{c N} &\approx \FID(P, Q) + \frac{\va_1}{cN} + \frac{\va_2}{c^2 N^2} + \ldots + \frac{\va_k}{c^k N^k}  \\
    &= \vth_0 + \vth_1 x + \vth_2 x^2 + \ldots + \vth_k x^k,
\end{aligned}
\end{equation}
derived in Assumption~\ref{assum:polynomial-expectation-expansion}, where $x \defeq \frac 1 c$ and $\vth_i \defeq \frac{\va_i}{N^i}$ for $i \ge 1$. Our goal is to estimate the constant coefficient term $\vth_0 = \FID(P, Q)$. To do so, a natural strategy is to solve a regression problem in observed estimates $\hdd_{\vc_i N}$ at $\ell \ge k + 1$ choices of $\vc_i$. Let $\vy \in \R^\ell$, $\mx \in \R^{\ell \times k + 1}$ have
\[\vy_i \defeq \hdd_{\vc_i N},\quad \mx_{ij} \defeq \vc_i^{-(j-1)}, \quad j = 1,2,...,k+1.\]
Then the relationship \eqref{eq:fid_regress} can be described by $\vy \approx \mx \vth$. To solve for $\vth$, we use a weighted regression problem, induced by a positive diagonal matrix $\mr = \diag{\vr} \in \R^{\ell \times \ell}$:
\[\hvth \defeq \arg\min_{\vth \in \R^{k + 1}} \Par{\vy - \mx\vth}^\top\mr\Par{\vy - \mx\vth}.\]
In principle, any $\mr \succ \0_{\ell \times \ell}$ can be used, but we focus on diagonal $\mr$ due to the independence of the $\hdd_{\vc_i N}$ calls.
When $\vc$ has at least $k + 1$ distinct entries, $\mx$ has full column rank, so the problem above has the solution
\[
\hvth_0 = \ve_1^\top\Par{\mx^\top\mr\mx}^{-1}\mx^\top\mr \vy.
\]
Finally, our extrapolation estimator \eqref{eq:extrapolation-estimator} has the form $\hD(\vn, \vw) = \vw^\top \vy$. Thus, to have $\hD(\vn, \vw)$ match our derived expression for $\hvth_0$ we take
\begin{equation}\label{eq:regress_weights}
\vw = \mr\mx\Par{\mx^\top\mr\mx}^{-1}\ve_1.
\end{equation}
Note that \eqref{eq:regress_weights} requires specification of the coefficient vector $\vc \in \R^\ell$ inducing $\mx$, and $\mr$ inducing the regression problem.
We conclude by showing that for these choices of $\vc$, $\mr$, Theorem~\ref{thm:general-order-extrapolation} applies to \eqref{eq:regress_weights}.

\begin{lemma}
\label{lem:ols-va-exact-order}
Fix $(k, \ell)$ with $\ell \ge k + 1$, $\vc \in (0, 1)^\ell$ with at least $k + 1$ distinct entries, and
$\mr = \diag{\vr} \succ \0_{\ell \times \ell}$. The weights~\eqref{eq:regress_weights}
satisfy
\begin{equation*}
    \vw\in\calW_k(\vc)\setminus\calW_{k+1}(\vc).
\end{equation*}
\end{lemma}
\begin{proof}
First, $\vw \in \calW_k(\vc)$ is equivalent to $\mx^\top \vw = \ve_1$, which follows from \eqref{eq:regress_weights}. Next, letting $\mc \defeq \diag{\vc}$, if $\vw \in \calW_{k + 1}(\vc)$ then $\mx^\top \mc^{-1}\vw = \0_{k + 1}$, which would imply the false statement
\[\Par{\mx^\top \mc^{-1}\mr\mx}\Par{\mx^\top \mr\mx}^{-1}\ve_1 = \0_{k + 1}.\]
\end{proof}

\textbf{Concrete choices of $\vc$, $\mr$.} We now specify several choices of $\vc, \mr$ used in our experiments. We first describe several natural grids for $\vc \in (0, 1)^\ell$, taking $\vc_1 \le \vc_2 \le \ldots \le \vc_\ell$ for simplicity. Fix a ratio $\rho > 1$ between the largest and smallest sample sizes, and let
\[\vt_i \defeq \frac{i - 1}{\ell - 1},\quad i \in [\ell],\]
and define three grids over $i \in [\ell]$:
\begin{equation}\label{eq:grids_def}\begin{aligned}
\vs_i^{\mathrm{unif}} \defeq 1 + (\rho - 1)\vt_i, \quad
\vs_i^{\mathrm{UI}}  \defeq \frac{1}{1 - (1 - \rho^{-1})\vt_i},\quad
\vs_i^{\mathrm{CL}} \defeq \frac{1}{1 - (1 - \rho^{-1})\frac{1-\cos(\pi\vt_i)}{2}}.
\end{aligned}
\end{equation}
The three choices in \eqref{eq:grids_def} correspond to taking uniformly spaced $\vs_i$, uniformly spaced $\vs_i^{-1}$, and $\vs_i^{-1}$ placed according to the \emph{Chebyshev-Lobatto} nodes.
For each $\vs_i^\star$ in \eqref{eq:grids_def} where $\star \in \{\mathrm{unif},\mathrm{UI},\mathrm{CL}\}$, we define corresponding sampling coefficients $\vc_i^\star$ by taking $\vc_i^\star = \frac{\vs_i^\star}{\sum_{j \in [\ell]} \vs_j^\star}$.

Regarding $\mr$, perhaps the simplest choice is to take $\mr = \id_\ell$, in which case we call
\begin{equation}\label{eq:wols}\vw^{\mathrm{OLS}} \defeq \mx\Par{\mx^\top\mx}^{-1}\ve_1\end{equation}
the \emph{polynomial OLS} weights corresponding to $\vc$. On the other hand, applying Theorem~\ref{thm:first-order-plugin-bias-variance} at each sample size $\vn_i \approx \vc_i N$ suggests that the variance of each estimate $\vy_i$ is proportional to $\frac 1 {\vc_i}$. This suggests reweighting the regression problem to favor lower-variance summands. We call
\begin{equation}\label{eq:wva}\vw^{\mathrm{VA}} \defeq \mc\mx\Par{\mx^\top\mc\mx}^{-1}\ve_1,\end{equation}
where $\mr = \mc \defeq \diag{\vc}$, the \emph{variance-aware} weights corresponding to $\vc$.

\subsection{Sharpness of the fixed-order bias}
\label{ssec:extrapolation-bias-lower-bound}

We conclude the section by showing that the bias bound in Theorem~\ref{thm:general-order-extrapolation} is sharp at every fixed order $k$. In particular, for a simple isotropic family, \emph{any} fixed extrapolation rule \eqref{eq:extrapolation-estimator} of exact order $k$ (i.e., $\vw \in \calW_k(\vc) \setminus \calW_{k + 1}(\vc)$) has bias $\gtrsim \frac{d^{k + 2}}{N^{k + 1}}$, matching the prediction of Theorem~\ref{thm:general-order-extrapolation}. By Lemma~\ref{lem:ols-va-exact-order}, this applies to both the OLS weights \eqref{eq:wols} and the variance-aware weights \eqref{eq:wva}.

We now describe our hard instance.
Fix a constant $D_0>0$ (our distance to be estimated), and consider two Gaussians with mean $\0_d$ and covariances
\begin{equation}
\label{eq:extrapolation-hard-family}
    \mb=\frac12\id_d,
    \quad
    \ma=a\id_d,
    \quad
    a\defeq
    \left(\sqrt{\frac12}+\sqrt{\frac{D_0}{d}}\right)^2.
\end{equation}
Then
$D=d_{\mathrm{BW}}^2(\ma,\mb)=D_0$, and $a\to \half$ as
$d\to\infty$. Both covariances satisfy
the normalization in Assumption~\ref{assum:normalization} for all sufficiently large $d$. In Appendix~\ref{app:isotropic-extrapolation-coefficients}, we prove the following technical lemma regarding the coefficients in Assumption~\ref{assum:polynomial-expectation-expansion} for the plug-in estimator.

\begin{restatable}{lemma}{SharpIsotropicCoefficients}
\label{lem:sharp-isotropic-coefficients}
Fix $k\ge1$ and $D_0 > 0$. In the setting of Model~\ref{model:one-side} induced by~\eqref{eq:extrapolation-hard-family}, let
$\va$ be the coefficient vector in
Proposition~\ref{prop:uniform-psd-expansion} at order $k$.
For all sufficiently large $d$,
\[
    \Abs{\va_j}
    \asymp_{k,D_0}
    d^{j+1},
    \text{ for all } 
    j\in[k].
\]
\end{restatable}

Using Lemma~\ref{lem:sharp-isotropic-coefficients}, we now prove our lower bound against extrapolation strategies. Note that Proposition~\ref{prop:extrapolation-bias-lower-bound} applies to the plug-in estimator \eqref{eq:plugin}, an order-$0$ extrapolation, as $1 \in \calW_0(c)$ for any $c$.

\begin{proposition}
\label{prop:extrapolation-bias-lower-bound}
Fix $k\ge 0$, $\ell \ge k + 1$, $D_0>0$, $\vc \in (0, 1]^\ell$, and
$\vw\in\calW_k(\vc)\setminus\calW_{k+1}(\vc)$.
In the setting of Model~\ref{model:one-side} induced by~\eqref{eq:extrapolation-hard-family}, there exist
constants $c,C>0$, depending only on $(k,\vc,\vw,D_0)$, such that,
for all sufficiently large $d$ and $N\ge Cd$, following the notation \eqref{eq:extrapolation-estimator},
\begin{equation*}
    c\frac{d^{k+2}}{N^{k+1}}
    \le
    \Abs{\bbE[\hD(\vn, \vw)]-D_0}
    \le
    C\frac{d^{k+2}}{N^{k+1}}.
\end{equation*}
\end{proposition}

\begin{proof}
Applying Proposition~\ref{prop:uniform-psd-expansion} at order $k+1$, the first $k$ terms cancel, so
\begin{equation*}
    \bbE[\hD(\vn, \vw)]-D_0
    =
    \frac{\va_{k+1}}{N^{k+1}}
    \sum_{i\in[\ell]}\frac{\vw_i}{\vc_i^{k+1}}
    +
    O_{k,\vc,\vw}\paren{\frac{d^{k+3}}{N^{k+2}}}.
\end{equation*}
Because $\vw \not\in \calW_{k +1}(\vc)$, the coefficient of $\va_{k + 1}$ in the above expression does not vanish.
The claim follows by plugging in the estimate on $\va_{k + 1}$ from Lemma~\ref{lem:sharp-isotropic-coefficients} and using our lower bound on $N$.
\end{proof}
\section{Relative Taylor Debiasing}
\label{sec:optimal-estimation}

The sample complexity of Theorem~\ref{thm:general-order-extrapolation} approaches the lower bound in Proposition~\ref{prop:gaussian-w2-lower-bound} as k increases, but remains suboptimal at every fixed order.  
In this section, we develop \emph{relative Taylor debiasing} (RTD), a new estimator for solving Problem~\ref{prob:FD_est} whose sample complexity matches the lower bound up to a constant factor.

At a high level, our strategy is based on the fact that the degree-$k$ correction term $Q_{k, \mc}(\mh)$ in \eqref{eq:order_k} can be identified with a symmetric multilinear form $\calQ_{k, \mc}[\mh^{\otimes k}]$. In our applications, we take $\mh = \mh_m = \frac 1 m \sum_{a \in [m]} \my_a$ to be a (whitened) sample covariance, and hence
\[Q_{k, \mc}(\mh_m) = \frac 1 {m^k} \sum_{a_1, \ldots, a_k \in [m]} \calQ_{k, \mc}[\my_{a_1},\ldots, \my_{a_k}].\]
In Lemma~\ref{lem:wishart-collision-expansion}, we tracked repeated indices amongst the $\{a_i\}_{i \in [k]}$, and applied Wishart moment bounds to bound surviving terms. Our RTD estimator instead removes terms with repeated indices directly, by estimating each multilinear slot using independent group averages.  

In Section~\ref{ssec:optimal-direct-debiasing}, we construct the covariance estimator and bound its bias and variance. In Section~\ref{ssec:optimal-efficient-evaluation}, we give its evaluation algorithm and prove that RTD solves Problem~\ref{prob:FD_est} using $n \asymp \frac d {\eps^2}$ samples.

We require a standard helper fact.

\begin{fact}[Chapter 4, \cite{Vershynin_2026}]\label{fact:pilot}
Let $\{\vX_i\}_{i \in [n]} \simiid\calN(\vmu,\msig)$ and let
$\widehat\msig$ be the empirical covariance. For every fixed
$\rho\in(0,1)$, if $n\gtrsim_\rho d$, then with probability at least
$0.99$,
\[
    (1-\rho)\widehat\msig
    \preceq
    \msig
    \preceq
    (1+\rho)\widehat\msig.\]
\end{fact}

Our RTD strategy leverages Fact~\ref{fact:pilot} by using $O(d)$ pilot samples to first construct a constant-factor approximation to $\mb$ in Model~\ref{model:one-side}, used to center our Taylor expansion-based estimator.

\subsection{Estimating $\BW$ via high-order debiasing}
\label{ssec:optimal-direct-debiasing}

Recall from Theorem~\ref{thm:first-order-plugin-bias-variance} that among the two components of $\FID$, i.e., the mean-based term $\norm{\vmu - \vnu}_2^2$ and covariance-based term $\BW^2(\ma, \mb)$, the plug-in estimator for the mean-based term already achieves the optimal $\approx \frac d {\eps^2}$ sample complexity for Problem~\ref{prob:FD_est}. Therefore, in this subsection, we focus on developing the covariance-based component of our RTD estimator.

To begin, throughout the section we let $\tmb \in \PSD^{d \times d}$ be a precomputed matrix satisfying
\begin{equation}\label{eq:tmbd_def}(1 - \rho) \tmb \preceq \mb \preceq (1 + \rho)\tmb,\end{equation}
for a small universal constant $\rho \in (0, \half)$ to be chosen later. In Theorem~\ref{thm:rtd-estimation}, we take $\tmb$ to be the output of Fact~\ref{fact:pilot} from $O(d)$ pilot samples, which incurs a small constant failure probability.

We next introduce some notation used in this section: let
\begin{equation}
\label{eq:optimal-relative-localization}
    \me
    \defeq
    \widetilde\mb^{-1/2}
    (\mb-\widetilde\mb)
    \widetilde\mb^{-1/2},
    \quad
    \widetilde\mc
    \defeq
    \paren{
        \widetilde\mb^{1/2}
        \ma
        \widetilde\mb^{1/2}
    }^{1/2}.
\end{equation}
When $\mb$ is singular, we overload notation and let $\me$ be defined only on the span of $\mb$.
Then under \eqref{eq:tmbd_def}, 
$\norm{\me}_{\mathrm{op}}\le\rho$ and
$\norms{\widetilde\mb}_{\mathrm{op}},
\norms{\widetilde\mc}_{\mathrm{op}}\le 2$. Thus Proposition~\ref{prop:trace-resolvent-expansion} with $\mh \gets \me$, $\mb \gets \tmb$ gives
\begin{equation}
\label{eq:optimal-local-bures-series}
    d_{\mathrm{BW}}^2(\ma,\mb)
    =
    d_{\mathrm{BW}}^2(\ma,\widetilde\mb)
    +
    \inprod{\widetilde\mb-\widetilde\mc}{\me}
    +
    \sum_{k \ge 2}
    Q_{k,\widetilde\mc}(\me).
\end{equation}
Note that the constant term $\BW^2(\ma, \tmb)$ is explicit. 
Our estimator for \eqref{eq:optimal-local-bures-series} uses independent groups of samples to estimate the next $L$ terms of the series \eqref{eq:optimal-local-bures-series}, truncating the rest.

We now state our estimator in more detail. Assume that we have samples $\{\vX_a\}_{a \in [m]} \sim \calN(\0_d, \mb)$,\footnote{Given two samples from $\calN(\vnu, \mb)$ as in Model~\ref{model:one-side}, their difference is $\sim \calN(\0_d, 2\mb)$, and hence downscaling their difference by $\sqrt{2}$ gives a fresh sample from $\calN(\0_d, \mb)$. We omit this distinction for brevity in this subsection.} and define the whitened sample covariance matrices
\begin{equation}\label{eq:y_def}\my_a \defeq \tmb^{-1/2} \vX_a\vX_a^\top \tmb^{-1/2} - \id_d,\quad a \in [m].\end{equation}
Thus, the $\my_a$ are independent and satisfy
$\bbE\my_a=\me$. 

For each degree $2 \le k \le L$, we produce an unbiased estimate of $Q_{k, \tmc}(\me)$ as follows. We first partition $[m]$ into $k$ disjoint groups
$G_{k,1},\ldots,G_{k,k}$ satisfying $|G_{k,i}|\asymp \frac m k$ for all $i \in [k]$, and set
\begin{equation}
\label{eq:cross-fit-group-means}
    \bmy_{k,i}
    \defeq
    \frac1{|G_{k,i}|}
    \sum_{j\in G_{k,i}}\my_j,
    \text{ for all } i\in[k].
\end{equation}
Note that different degrees $k$ reuse the same samples $\{\my_a\}_{a \in [m]}$, but the group means $\{\bmy_{k, i}\}_{i \in [k]}$ are independent. Our variance bound requires handling dependences between degrees.

Finally, we are ready to define the covariance component of our RTD estimator: letting
\begin{equation}
\label{eq:cross-fit-q}
    \widehat Q_k
    \defeq
    \mathcal Q_{k,\widetilde\mc}
    \Brack{
        \bmy_{k,1},
        \ldots,
        \bmy_{k,k}
    },
\end{equation}
where $\calQ_{k, \tmc}$ is the multilinear form corresponding to $Q_{k, \tmc}$ from \eqref{eq:order_k},
we let
\begin{equation}
\label{eq:cross-fit-bures-estimator}
    \widehat D_{\mathrm{BW},L}
    \defeq
    d_{\mathrm{BW}}^2(\ma,\widetilde\mb)
    +
    \frac1m
    \sum_{a \in [m]}
    \inprod{\widetilde\mb-\widetilde\mc}{\my_a}
    +
    \sum_{k=2}^L
    \widehat Q_k.
\end{equation}
Our final RTD estimator simply uses the plug-in estimator $\norm{\vmu - \hvnu_n}_2^2$ from Theorem~\ref{thm:first-order-plugin-bias-variance}, and adds the estimator \eqref{eq:cross-fit-bures-estimator}. By independence of group means and multilinearity,
\begin{equation}\label{eq:expect_dbwl}\E\Brack{ \widehat D_{\mathrm{BW},L}} = \BW^2(\ma, \tmb) + \inprod{\tmb - \tmc}{\me} + \sum_{k = 2}^L Q_{k, \tmc}(\me).\end{equation}
Comparing with \eqref{eq:optimal-local-bures-series}, our estimator has two sources of error: a bias due to the truncated tail terms $k \ge L + 1$, and the variance of $\hD_{\mathrm{BW}, L}$ itself.
Our main technical result bounds these error sources.

\begin{proposition}
\label{prop:cross-fit-statistical}
Let $\tmb$ satisfy \eqref{eq:tmbd_def}, and assume $\{\my_a\}_{a \in [m]}$ follow \eqref{eq:y_def}. If $m \gtrsim d + L$, then
\[
    \Abs{
        \bbE[
            \widehat D_{\mathrm{BW},L}
        ]
        -
        d_{\mathrm{BW}}^2(\ma,\mb)
    }
    \lesssim
    d\rho^{L+1},\quad
    \Var\left[
        \widehat D_{\mathrm{BW},L}
    \right]
    \lesssim
    \frac dm.
\]
\end{proposition}

The bias bound in Proposition~\ref{prop:cross-fit-statistical} is a direct corollary of Proposition~\ref{prop:trace-resolvent-expansion}, so we focus on proving the variance bound. After centering each group mean $\bmy_{k, i}$ around $\me$, the expansion of $\hQ_k - Q_{k, \tmc}(\me) = \calQ_{k, \tmc}[\bmy_{k, 1}, \ldots, \bmy_{k, k}] - \calQ_{k, \tmc}[\me, \ldots, \me]$ consists of terms of the form
\begin{equation}\label{eq:mixed_tensor}\calQ_{k, \tmc}[\cdot^{\otimes j}, \me^{\otimes k - j}]\end{equation}
applied to independent mean-zero matrices. To apply Lemma~\ref{lem:multilinear-hs-second-moment} to these tensors, we require understanding their Frobenius norms, so we provide such a bound in Lemmas~\ref{lem:symmetrized-trace-word-hs} and~\ref{lem:optimal-partial-contraction}.

\begin{lemma}
\label{lem:symmetrized-trace-word-hs}
Let $\mb_1,\ldots,\mb_j\in\R^{d\times d}$, and define a symmetric $j$-linear form $\calM: (\R^{d \times d})^j \to \R$ by
\begin{equation}
\label{eq:symmetrized-trace-word}
    \calM(\mx_1,\ldots,\mx_j)
    \defeq
    \frac1{j!}
    \sum_{\pi\in S_j}
    \Tr\paren{
       \mx_{\pi(1)}
        \mb_1\cdots
        \mx_{\pi(j)}\mb_j
    },
\end{equation}
where $S_j$ is the set of permutations $\pi: [j] \to [j]$.
Then letting $(d)_j \defeq d(d+1)\ldots(d + j - 1)$,
\begin{equation}
\label{eq:symmetrized-trace-word-hs}
    \normf{\calM}^2
    \le
    \frac{(d)_j}{(j-1)!}
    \prod_{\ell \in [j]}
    \norm{\mb_\ell}_{\mathrm{op}}^2.
\end{equation}
\end{lemma}
\begin{proof}
We set up two pieces of helper notation. Let $C_j \subseteq S_j$ be the set of $j$-\emph{cycles}, i.e., permutations consisting of a single cycle. Also, for a permutation $\pi \in S_j$, define $c_\pi \in S_j$ by $c_\pi(i) = \pi(\pi^{-1}(i) + 1)$ where indices are modulo $j$, i.e., the element of $\pi$ after $i$. For example, if $\pi = \{3, 2, 4, 1\}$, $c_\pi(4) = 1$.

Next, let $\mg_1,\ldots,\mg_j$ be independent $d\times d$ matrices with
i.i.d.\ entries $\sim \calN(0, 1)$, so that
\begin{equation}
\label{eq:gaussian-probing-hs}
    \normf{\calM}^2
    =
    \bbE\Abs{\calM(\mg_1,\ldots,\mg_j)}^2.
\end{equation}
Also, for each permutation $\pi\in S_j$, write
\begin{equation*}
    \calM_\pi
    \defeq
    \Tr\paren{\mg_{\pi(1)}
        \mb_1\cdots
        \mg_{\pi(j)}\mb_j
    }.
\end{equation*}
Thus, expanding \eqref{eq:gaussian-probing-hs} gives
\begin{equation}\label{eq:cross_terms}
\normf{\calM}^2 = \frac 1 {(j!)^2} \sum_{\pi, \sigma} \E\Brack{\calM_\pi\calM_\sigma}.
\end{equation}
We next bound $\E[\calM_\pi \calM_\sigma]$ for fixed $\pi, \sigma \in S_j$. Expanding traces entrywise,
\[\calM_\pi = \sum_{(a, b) \in [d]^j \times [d]^j} \prod_{r \in [j]} [\mg_r]_{a_r b_r} [\mb_{\pi^{-1}(r)}]_{b_r a_{c_\pi(r)}}.\]
Thus, when expanding $\calM_\pi \calM_\sigma$ and taking expectations, all terms vanish except those which share all indices in $a, b$ due to the independent Gaussian entries. Therefore,
\begin{align*}
\E\Brack{\calM_\pi\calM_\sigma} = \sum_{a \in [d]^j} \prod_{r \in [j]} \Brack{\mb_{\pi^{-1}(r)}^\top \mb_{\sigma^{-1}(r)}}_{a_{c_\pi(r)}a_{c_\sigma(r)}} ,
\end{align*}
where we collapsed the $b$ indices. Next, let $\mh_r \defeq \mb^\top_{\pi^{-1}(r)}\mb_{\sigma^{-1}(r)}$ and $\rho \defeq c_\pi^{-1}c_\sigma$, and re-index $x_r \defeq a_{c_\pi(r)}$. The above display then becomes
\[\E\Brack{\calM_\pi\calM_\sigma} = \sum_{x \in [d]^j} \prod_{r \in [j]} [\mh_r]_{x_rx_{\rho(r)}}.\]
This sum factors based on the cycle structure of $\rho$. Indeed, if $\rho = c_\tau$ is exactly a $j$-cycle, then the above display is $\Tr(\mh_{\tau(1)}\mh_{\tau(2)}\ldots\mh_{\tau(j)})$, and otherwise the terms factor independently. Thus,
\[\E\Brack{\calM_\pi\calM_\sigma} = \prod_{(r_1, r_2, \ldots, r_\ell) \text{ a cycle in } \rho} \Tr\Par{\mh_{r_1}\ldots\mh_{r_\ell}}.\]
Now using $|\Tr(\mh_{r_1}\ldots\mh_{r_\ell})| \le d\prod_{i \in [\ell]}\normsop{\mh_{r_i}}$ on each term above, we have
\[\Abs{\E[\calM_\pi\calM_\sigma]} \le d^{|\cyc(\rho)| } \prod_{r \in [j]} \normop{\mh_r} \le d^{|\cyc(\rho)| } \prod_{r \in [j]} \normop{\mb_r}^2, \]
where $\cyc(\rho)$ is the set of cycles that partition $\rho$. Plugging this back into \eqref{eq:cross_terms} then gives
\begin{align*}\normf{\calM}^2 &\le \frac {1} {(j!)^2} \prod_{r \in [j]} \normop{\mb_r}^2 \sum_{\pi, \sigma \in S_j} d^{|\cyc(c_\pi^{-1}c_\sigma)|} \\
&= \frac {1} {((j - 1)!)^2} \prod_{r \in [j]} \normop{\mb_r}^2 \sum_{c, c' \in C_j} d^{|\cyc(c^{-1}c')|} \\
&\le \frac{1}{(j - 1)!} \prod_{r \in [j]} \normop{\mb_r}^2 \sum_{\pi \in S_j} d^{|\cyc(\pi)|}.\end{align*}
The second line used that each cycle $c_\pi$ corresponds to $j$ distinct permutations $\pi$ (creating a factor of $j^2$), and the third line used that for each $c \in C_j$, we have $\sum_{c' \in C_j} d^{|\cyc(c^{-1}c')|} \le \sum_{\pi \in S_j} d^{|\cyc(\pi)|}$, and there are $(j - 1)!$ choices of the cycle $c$. The conclusion follows from the standard identity
\[\sum_{\pi \in S_j} d^{|\cyc(\pi)|} = (d)_j.\]
\end{proof}

We next apply Lemma~\ref{lem:symmetrized-trace-word-hs} to the specific setting of the tensors \eqref{eq:mixed_tensor}.

\begin{lemma}
\label{lem:optimal-partial-contraction}
Under the event \eqref{eq:tmbd_def}, define $\me, \tmc$ as in \eqref{eq:optimal-relative-localization}, let $k \ge 2$, and let $j \in [k]$. Then letting $\calQ_{k, \tmc}$ be the symmetric $k$-linear form associated with \eqref{eq:order_k},
\[
    \normf{
        \mathcal Q_{k,\widetilde\mc}
        [\,\cdot^{\otimes j},\me^{\otimes k-j}]
    }
    \lesssim
    \rho^{k-j}
    \paren{
        \frac{(d)_j}{(j-1)!}
    }^{1/2}.
\]
\end{lemma}

\begin{proof}
For all $t>0$, we define for convenience
\begin{equation*}
    \mr_t\defeq(t\id_d+\widetilde\mc^2)^{-1},
    \quad
    \md_t\defeq\widetilde\mc\mr_t^{1/2}.
\end{equation*}
These matrices are all functions of $\tmc$, so they all commute. Our starting point is manipulating the expression for $\calQ_{k, \tmc}$ into a form to which Lemma~\ref{lem:symmetrized-trace-word-hs} applies. First, from \eqref{eq:order_k} we have
\begin{equation*}
    Q_{k,\widetilde\mc}(\mx)
    =
    \frac{2(-1)^k}{\pi}
    \int_0^\infty
    t^{1/2}
    \Tr\paren{
        \mr_t
        \paren{\md_t\mx\md_t}^k
    }\mathrm dt.
\end{equation*}
The expression above is obtained by setting $\mx_1, \ldots, \mx_k$ all equal in the following $k$-linear form:
\[(\mx_1, \ldots, \mx_k) \to \frac{2(-1)^k}{\pi}
    \int_0^\infty
    t^{1/2}
    \Tr\paren{
        \mx_1 \md_t^2 \mx_2 \ldots \md_t^2 \mx_k
    \md_t^2 \mr_t}\mathrm dt.\]
Therefore, to symmetrize the $k$-linear form, we explicitly have
\begin{equation}
\label{eq:polarized-Q-resolvent}
\begin{aligned}
    \mathcal Q_{k,\widetilde\mc}[\mx_1,\ldots,\mx_k] =
    \frac{2(-1)^k}{\pi k!}
    \sum_{\pi \in S_k}
    \int_0^\infty
    t^{1/2}
    \Tr\paren{
        \mx_{\pi(1)} \md_t^2 \mx_{\pi(2)} \ldots \md_t^2 \mx_{\pi(k)}
    \md_t^2 \mr_t
    }\mathrm dt.
\end{aligned}
\end{equation}
The expression \eqref{eq:polarized-Q-resolvent} is almost compatible with Lemma~\ref{lem:symmetrized-trace-word-hs}, except the free positions can change depending on which $\mx_r$ copies are fixed to $\me$. For each $t > 0$ and possible set of free positions $P \in \binom{[k]}{j}$, we define a symmetric $j$-multilinear form
\[\calM_{t, P}[\mx_1, \ldots, \mx_j] = \frac 1 {j!} \sum_{\pi \in S_j} \Tr\Par{\mx_{\pi(1)} \mb_{1, t, P} \mx_{\pi(2)} \mb_{2, t, P} \ldots \mx_{\pi(j)} \mb_{j, t, P}},\]
such that the $\mb_{r, t, P}$ consist of intermediate copies of $\md_t$, $\mr_t$, and $\me$ from the integrand \eqref{eq:polarized-Q-resolvent}, depending on the positions $P$. The only important property of these matrices is that
\begin{equation}\label{eq:opnorm_bound}\prod_{r \in [j]} \normop{\mb_{r, t, P}} \le \rho^{k - j}\normop{\mr_t\md_t^2},\end{equation}
because there are $k - j$ copies of $\me$ with $\normsop{\me} \le \rho$, and $\normsop{\md_t} \le 1$ so we eliminate $2(k - 1)$ copies of it. We next regroup the permutation sum in \eqref{eq:polarized-Q-resolvent} to apply Lemma~\ref{lem:symmetrized-trace-word-hs}:
\[\mathcal Q_{k,\widetilde\mc}[\mx_1,\ldots,\mx_j, \me^{\otimes k - j}] =
    \frac{2(-1)^k}{\pi}
    \int_0^\infty
    t^{1/2} \frac 1 {\binom k j} \sum_{P \in \binom{[k]}{j}} \calM_{t, P}[\mx_1, \ldots, \mx_j]\mathrm dt.\]
    To see where the coefficient of $\binom{k}{j}^{-1}$ comes from, observe that 
    \[\frac{(k - j)!}{k!} = \frac 1 {j!\binom k j},\]
    where the factor of $(k - j)!$ arises because this is the number of different labelings of the positions occupied by the $\me$, corresponding to the same multilinear form.
    Another way to see this is that the total mass on multilinear forms is $1$ in both expressions. Finally, Lemma~\ref{lem:symmetrized-trace-word-hs} and \eqref{eq:opnorm_bound} yield
    \begin{equation}\label{eq:last_integral}\normf{\mathcal Q_{k,\widetilde\mc}[\cdot^{\otimes j}, \me^{\otimes k - j}]} \le \rho^{k - j}\Par{\frac{(d)_j}{(j - 1)!}}^{1/2} \int_0^\infty t^{1/2} \normop{\mr_t\md_t^2}\dd t.\end{equation}
    Observe that since $\tmc \in \PSD^{d \times d}$ and $\normsop{\tmc} \in [0, 2]$, all eigenvalues of $\tmc$ are in $[0, 2]$, so
    \[t^{1/2}\normsop{\mr_t \md_t^2} = t^{1/2} \normop{\mr_t^2 \tmc^2} \le \sup_{s \in [0, 2]}  \frac{s^2t^{1/2}}{(t + s^2)^2} \le \min\Brace{\frac 1 {4\sqrt{t}}, \frac 4 {t^{1.5}}},\]
    where we used $(t + s^2)^2 \ge \max\{4ts^2, t^2\}$.
    The conclusion follows because the integral in \eqref{eq:last_integral} is bounded by a universal constant after substituting the above display.
\end{proof}

\begin{remark}
The three factors in Lemma~\ref{lem:optimal-partial-contraction} have separate
origins. The factor $\rho^{k-j}$ comes from the $k-j$ copies
of $\me$, the $(d)_j^{1/2}$
is the cost of the $j$ free matrix slots from Lemma~\ref{lem:symmetrized-trace-word-hs}, and
$(j-1)!^{-1/2}$ is the gain from full symmetrization of the multilinear forms. This additional factorial
saving ends up saving a factor of $L \approx \log d$ from the sample complexity of RTD.
\end{remark}

\begin{proof}[Proof of Proposition~\ref{prop:cross-fit-statistical}]
By comparing \eqref{eq:optimal-local-bures-series} with \eqref{eq:expect_dbwl}, the bias magnitude is
\begin{equation*}
    \Abs{
        \sum_{k>L}
        Q_{k,\widetilde\mc}(\me)
    } \le \frac{Q_{2, \tmc}(\me)\normop{\me}^{L - 1}}{1 -\normop{\me}}
    \lesssim
    \norm{\me}_{\mathrm F}^2
    \frac{\normop{\me}^{L - 1}}{1-\rho}
    \lesssim
    d\rho^{L+1},
\end{equation*}
where we used Proposition~\ref{prop:trace-resolvent-expansion}, \eqref{eq:2c_bound}, and $\normf{\me}^2 \le d\normop{\me}^2 \le d\rho^2$ in the three inequalities. 

The rest of the proof bounds the variance. Write
\[\bmy_{k, a} = \me + \mxi_{k, a},\]
where each $\mxi_{k, a}$ is the average of $|G_{k,a}| = \Theta(\frac m k)$ i.i.d.\ draws of $\my_a - \me$, equal in distribution to
\[\vz\vz^\top - \Par{\id_d + \me},\quad \vz \sim \calN(\0_d, \id_d + \me).\]
Therefore, for every fixed $\mm \in \Sym^{d \times d}$, standard Gaussian moment calculations show
\begin{align*}\E\Brack{\inprod{\mm}{\mxi_{k, a}}^2 } &= \frac 1 {|G_{k, a}|} \E_{\vz \sim \calN(\0_d, \id_d + \me)}\Brack{\inprod{\mm}{\vz\vz^\top - (\id_d + \me)}^2} \\
&= \frac{2}{|G_{k, a}|} \Tr\Par{\mm\Par{\id_d + \me}\mm\Par{\id_d + \me}} \lesssim \frac{k}{m} \normf{\mm}^2.
\end{align*}
This verifies Lemma~\ref{lem:multilinear-hs-second-moment} with $v=\frac {Ck} m$ for a universal constant $C$.
Next, expanding~\eqref{eq:cross-fit-q} multilinearly,
\begin{equation}\label{eq:variance_term}
    \widehat Q_k-Q_{k,\widetilde\mc}(\me)
    =
    \sum_{\varnothing\ne J\subseteq[k]}
    \mathcal Q_{k,\widetilde\mc}
    [
        \mxi_{k,a}:a\in J;
        \me:a\notin J
    ].
\end{equation}
In other words, there is a term for each nonempty subset $J \subseteq [k]$, corresponding to free indices where $\calQ_{k, \tmc}$ acts on the independent matrices $\{\mxi_{k, a}\}_{a \in J}$. Lemmas~\ref{lem:multilinear-hs-second-moment} and~\ref{lem:optimal-partial-contraction} then imply
\begin{equation}\label{eq:one_subset_bound}
\begin{aligned}\Par{\E\Brack{\mathcal Q_{k,\widetilde\mc}
    \Brack{
        \mxi_{k,a}:a\in J;
        \me:a\notin J
    }^2} }^{1/2} &\lesssim \Par{\frac{Ck}{m}}^{j/2} \normf{\calQ_{k, \tmc}[\cdot^{\otimes j}, \me^{\otimes k - j}]} \\
    &\lesssim \Par{\frac{Ck}{m}}^{j/2} \rho^{k-j}\Par{\frac{(d)_j}{(j-1)!}}^{1/2}.
    \end{aligned}
    \end{equation}
We next combine these terms over all $\emptyset \neq J \subseteq [k]$. We repeatedly use the inequality
\begin{equation}\label{eq:cs_var}\E\Brack{\Par{\sum_r U_r}^2} = \sum_{r,s} \E[U_rU_s]\le \sum_{r, s} \sqrt{\E[U_r^2]\E[U_s^2]}= \Par{\sum_r \sqrt{\E[U_r^2]}}^2\end{equation}
for any collection of (possibly dependent) random variables $\{U_r\}_{r}$. First, at a fixed $k$,
\begin{equation}\label{eq:fixed_k}
\E\Brack{\Par{\hQ_k - Q_{k, \tmc}(\me)}^2}^{1/2} \lesssim \sum_{j \in [k]} \binom{k}{j}  \Par{\frac{Ck}{m}}^{j/2} \rho^{k-j}\Par{\frac{(d)_j}{(j-1)!}}^{1/2}
\end{equation}
which follows from \eqref{eq:one_subset_bound} and \eqref{eq:cs_var} applied to each subset $\emptyset \neq J \subseteq [k]$, using that there are $\binom{k}{j}$ such subsets with $|J| = j$. Next, applying \eqref{eq:cs_var} once again across different $k$,
\begin{align*}
\Var\Brack{\sum_{k = 2}^L \hQ_k}^{1/2} &= \E\Brack{\Par{\sum_{k = 2}^L\hQ_k - Q_{k, \tmc}(\me)}^2}^{1/2} \\
&\lesssim \sum_{j \in [L]} \Par{\frac{C}{m}}^{j/2} \Par{\frac{(d)_j}{(j-1)!}}^{1/2} \Par{\sum_{k = \max\{2, j\}}^L \binom{k}{j} k^{j/2} \rho^{k-j}}.
\end{align*}
Next, an elementary calculation shows that if $\rho \le \frac 1 {16}$, then 
\[\sum_{k = \max\{2, j\}}^L \binom{k}{j} k^{j/2} \rho^{k-j} \le \sum_{k = \max\{2, j\}}^\infty \binom{k}{j} k^{j/2} \rho^{k-j} \le 2^j \sqrt{j!}\]
so combining the above two displays and $(d)_j \le d(d + L)^{j - 1}$,
\begin{align*}\Var\Brack{\sum_{k = 2}^L \hQ_k}^{1/2} &\lesssim \sum_{j \in [L]} \Par{\frac {4C} m}^{j/2}\Par{\frac{(d)_j}{(j-1)!}}^{1/2} \sqrt{j!} \\
&\lesssim \sqrt{\frac d m} \sum_{j \in [L]} \sqrt{j} \Par{\frac{4C(d + L)}{m}}^{(j-1)/2}.
\end{align*}
Finally, taking $m \gtrsim d + L$ sufficiently large gives
\[\sum_{j \in [L]} \sqrt{j} \Par{\frac{4C(d + L)}{m}}^{(j-1)/2} \lesssim 1,\]
so we have shown that $\Var[\sum_{k = 2}^L \hQ_k] \lesssim \frac d m$. The conclusion follows from
\[\Var\Brack{\frac 1 m \sum_{a \in [m]} \inprod{\tmb - \tmc}{\my_a}} \lesssim \frac 1 m \normf{\tmb - \tmc}^2 \lesssim \frac d m,\]
and combining the linear and high-order terms in \eqref{eq:cross-fit-bures-estimator} with $\Var[X + Y] \le 2\Var[X] + 2\Var[Y]$.
\end{proof}

\subsection{Efficient implementation}
\label{ssec:optimal-efficient-evaluation}

We now discuss the efficient computation of the estimates
\[\hQ_k = \mathcal Q_{k,\mc}
    \paren{
        \bmy_{k,1},
        \ldots,
        \bmy_{k,k}
    },\quad 2 \le k \le L,\]
required by our estimator $\hD_{\mathrm{BW}, L}$ in \eqref{eq:cross-fit-bures-estimator}. Lemma~\ref{lem:implement_homogeneous} first shows how to compute one homogeneous term $\calQ_{k, \mc}[\my^{\otimes k}]$ via an explicit recursion from \cite{del2018taylor}, and then Lemma~\ref{lem:implement_polarization} shows how to compute mixed products through $\calQ_{k, \mc}$ via polarization, at a $2^k$ overhead over Lemma~\ref{lem:implement_homogeneous}. 

We give a brief introduction to the strategy in Lemma~\ref{lem:implement_homogeneous}, following Theorem 1.1, \cite{del2018taylor}. Our strategy is based on using \emph{Sylvester equations}, a classical tool from control theory, to recursively compute coefficients of the Taylor expansion of the matrix square root. To see the relevance to our setting, recall from Lemma~\ref{lem:first_order_expand} and Proposition~\ref{prop:trace-resolvent-expansion} that $Q_{k, \mc}(\mh)$ is the homogeneous degree-$k$ term in the Taylor expansion of $\Phi_{\mc}(\mh)$, defined as
\[\Phi_{\mc}(\mh) = \inprod{\mc}{\mh} + 2\Tr(\mc) - 2\Tr\Par{\Par{\mc\Par{\id_d + \mh}\mc}^{1/2}}.\]
The constant and linear terms in the Taylor expansion of $\Phi_{\mc}$ vanish, which we show in the proof of Lemma~\ref{lem:implement_homogeneous}. 
To extract the degree-$k \ge 2$ term, we set $\mh = z\mv$ for a fixed $\mv \in \Sym^{d \times d}$, so 
\[\Phi_{\mc}(z\mv) = \sum_{k \ge 2} z^k Q_{k, \mc}(\mv). \]
Our task thus reduces to understanding the expansion of
\begin{equation}\label{eq:xv_notation}
\begin{aligned}
\mx(z, \mv) \defeq \Par{\mc (\id_d + z\mv)\mc}^{1/2} = \mc + \sum_{k \ge 1} \mx_k(\mv) z^k,
\end{aligned}
\end{equation}
where $\mx_k(\mv)$
is the $k^{\text{th}}$ Taylor coefficient of $\mx(z, \mv)$ in $z$, holding $\mv$ as fixed. Fortunately, a computationally efficient algorithm for evaluating $\mx_k(\mv)$ was provided by \cite{del2018taylor}, based on recursively matching coefficients of $\mx(z, \mv)^2 = \mc (\id_d + z\mv)\mc$. In the scalar case, this recursion is simply
\begin{align*}c\sqrt{1 + zv} = c + \sum_{k \ge 1} x_k(v) z^k &\iff c^2(1 + zv) = \Par{c + \sum_{k \ge 1} x_k(v) z^k}^2 \\
&\iff 2cx_k(v) = -\sum_{r \in [k-1]} x_r(v) x_{k-r}(v),\text{ for all } k \ge 2.
\end{align*}

\begin{lemma}\label{lem:implement_homogeneous}
Let $k \ge 2$, $\mc \in \PSD^{d \times d}$, $\mv \in \Sym^{d \times d}$, define $Q_{k,\mc}(\mv)$ as in \eqref{eq:order_k}, and define $\mx_k(\mv)$ as in \eqref{eq:xv_notation}. Then, $Q_{k, \mc}(\mv) = -2\Tr(\mx_k(\mv))$, and can be computed in time $O(k^2 d^\omega)$.
\end{lemma}
\begin{proof}
First, suppose that $\mc$ is positive definite. Then the expansion \eqref{eq:xv_notation} is justified because the square root is analytic on regions bounded away from $0$. Squaring both sides of \eqref{eq:xv_notation} gives
\[\mc(\id_d + z\mv)\mc = \Par{\mc + \sum_{k \ge 1} \mx_k(\mv)z^k}^2,\]
and comparing the nonzero coefficients in $z$ implies the following recursion holds:
\begin{equation}\label{eq:sylvester}
\begin{aligned}\mc\mx_1(\mv) + \mx_1(\mv)\mc &= \mc\mv\mc, \\
\mc\mx_r(\mv) + \mx_r(\mv)\mc &= -\sum_{s \in[r-1]} \mx_s(\mv)\mx_{r-s}(\mv),\text{ for all } 2 \le r \le k.
\end{aligned}
\end{equation}
This recursion is also exactly the specialization of Theorem 1.1, \cite{del2018taylor} to $\ma \gets \mc^2$ and $\mh \gets \mc\mv\mc$. Next, by writing $\Phi_{\mc}(z\mv)$ from Lemma~\ref{lem:first_order_expand} in two different ways,
\begin{align*}
\Phi_{\mc}(z\mv) = z\inprod{\mc}{\mv} - 2\Par{\Tr(\mx(z, \mv)) - \Tr(\mc)} = \sum_{k \ge 2} z^k Q_{k, \mc}(\mv),
\end{align*}
it immediately follows from our definition of $\mx_k(\mv)$ in \eqref{eq:xv_notation} that $Q_{k, \mc}(\mv) = -2\Tr(\mx_k(\mv))$. To see that the other terms cancel, the constant term is $-2\Tr(\mc) + 2\Tr(\mc) = 0$, and the linear term is
\begin{equation}\label{eq:linear_cancel}z\inprod{\mc}{\mv} - 2z\Tr(\mx_1(\mv)) = z\Par{\Tr(\mc\mv) - 2\Tr(\mx_1(\mv))}.\end{equation}
Now if $\mc$ has eigendecomposition $\muu\mlam\muu^\top$ where $\mlam = \diag{\vlam}$, let $\tmv \defeq \muu^\top \mv\muu$ and $\tmx_1(\mv) \defeq \muu^\top\mx_1(\mv)\muu$. Then the diagonal of the first line of \eqref{eq:sylvester}, after changing basis, reads
\begin{equation}\label{eq:sylvester_basis}\mlam \tmx_1(\mv) + \tmx_1(\mv)\mlam = \mlam \tmv\mlam \implies 2\vlam_j [\tmx_1(\mv)]_{jj} =\vlam_j^2 \tmv_{jj} \text{ for all } j \in [d].\end{equation}
Dividing by $\vlam_j$ and summing over all $j \in [d]$ shows $2\Tr(\tmx_1(\mv)) = \Tr(\mlam \tmv)$, and then changing basis (which does not affect the trace) shows that \eqref{eq:linear_cancel} indeed vanishes. We note at this point that the entire proof extends straightforwardly to singular $\mc$, by restricting to its span.

Finally, we discuss efficient computation. We first spend $O(d^\omega)$ time forming the eigendecomposition of $\mc$. As justified in \eqref{eq:sylvester_basis}, each Sylvester equation $\mc\mx + \mx\mc = \mr$ is solved entrywise by
\[\tmx_{ij} = \frac{\tmr_{ij}}{\vlam_i + \vlam_j},\]
after changing basis to $\tmx = \muu^\top \mx\muu$, $\tmr = \muu^\top \mr\muu$, where the entry is $0$ if $\vlam_i = \vlam_j = 0$. Thus, we work in the $\muu$ basis and apply the above recursion $k$ times in total to \eqref{eq:sylvester}. Forming each of the matrices $\mr$ on the right-hand side of \eqref{eq:sylvester} takes at most $O(kd^\omega)$ time, giving the claim.
\end{proof}

We now use the multilinearity of $\calQ$ for our extension to general mixed inputs, using Lemma~\ref{lem:implement_homogeneous}. Pseudocode for the algorithm in Lemma~\ref{lem:implement_polarization} is provided in Algorithm~\ref{alg:evaluate-q} for convenience, where we define the helper operation, for an input $\mr \in \Sym^{d \times d}$ and diagonal $\mlam = \diag{\vlam} \in \PSD^{d \times d}$,
\[\big[\mathsf{Sylv}_{\mlam}(\mr)\big]_{ij}
    \leftarrow
    \begin{cases}
        \dfrac{\mr_{ij}}{\vlam_i+\vlam_j},
        & \vlam_i+\vlam_j>0,\\[6pt]
        0,
        & \vlam_i=\vlam_j=0,
    \end{cases}
    \text{ for all } i, j\in[d].\]

\begin{lemma}\label{lem:implement_polarization}
Let $k \ge 2$, $\mc \in \PSD^{d \times d}$, and let $\calQ_{k, \mc}$ be the symmetric $k$-linear form such that $\calQ_{k, \mc}[\mv^{\otimes k}] = Q_{k, \mc}(\mv)$ for all $\mv \in \Sym^{d \times d}$, following \eqref{eq:order_k}. Then, 
\[\calQ_{k, \mc}\Brack{\mv_1, \ldots, \mv_k}\]
can be computed in time $O(2^k k^2 d^\omega)$, for any $\{\mv_i \in \Sym^{d \times d}\}_{i \in [k]}$.
\end{lemma}
\begin{proof}
By multilinearity of $\calQ_{k, \mc}$, the following \emph{polarization identity} holds:
\[\calQ_{k, \mc}\Brack{\mv_1, \ldots, \mv_k} = \frac 1 {k!} \sum_{S \subseteq [k]} (-1)^{k - |S|} Q_{k, \mc}\Par{\sum_{i \in S} \mv_i}. \]
This identity, a consequence of the principle of inclusion-exclusion, is Eq.\ (7) in \cite{Thomas14}, and is attributed therein to \cite{MazurO34}.
We then use Lemma~\ref{lem:implement_homogeneous} to evaluate each of the $2^k$ terms.
\end{proof}

\begin{algorithm}[t]
\DontPrintSemicolon
\caption{$\mathsf{EvaluateQ}(\mc,\mv_1,\ldots,\mv_k)$}
\label{alg:evaluate-q}
\KwIn{$k \geq 2$, $\mc \in \Sym_{\succeq \mzero}^{d\times d}$,
and $\{\mv_i \in \Sym^{d\times d}\}_{i\in[k]}$}
\KwOut{$\calQ_{k,\mc}[\mv_1,\ldots,\mv_k]$}

$(\muu,\mlam) \leftarrow$ eigendecomposition of $\mc$,
where $\mlam=\diag{\vlam}$\;

$\tmv_i \leftarrow \muu^\top \mv_i \muu$ for all $i\in[k]$\;

$q \leftarrow 0$\;

\For{$S\subseteq[k]$}{
    $\mv_S \leftarrow \sum_{i\in S}\tmv_i$\;

    $\mx_1 \leftarrow
    \mathsf{Sylv}_{\mlam}\!\left(\mlam\mv_S\mlam\right)$\;

    \For{$r=2,\ldots,k$}{
        $\mr_r \leftarrow
        -\sum_{s\in[r-1]}\mx_s\mx_{r-s}$\;

        $\mx_r \leftarrow \mathsf{Sylv}_{\mlam}(\mr_r)$\;
    }

    $q \leftarrow
    q + (-1)^{k-|S| + 1}\big(2\Tr(\mx_k)\big)$\;
}

\Return{$\frac q {k!}$}\;
\end{algorithm}

We are now ready to give our main result on sample-optimal FID estimation. Explicitly, the RTD estimator first uses $n_0 = O(d)$ pilot samples to form an empirical covariance $\tmb$, and then outputs
\begin{equation}\label{eq:rtd_def}
\norm{\vmu - \hvnu_{n - n_0}}_2^2 + \hD_{\mathrm{BW}, L},
\end{equation}
where $\hvnu_{n - n_0}$ is the empirical mean \eqref{eq:empirical}, $\hD_{\mathrm{BW}, L}$ is the estimator \eqref{eq:cross-fit-bures-estimator}, and both $\hvnu_{n - n_0}$ and $\hD_{\mathrm{BW}, L}$ are computed using a shared pool of samples, after excluding the pilot samples. Consistently with how $\hD_{\mathrm{BW}, L}$ assumes samples from $\calN(\0_d, \mb)$ (which we can simulate using two draws from $\calN(\vnu, \mb)$, see discussion before \eqref{eq:y_def}), we set $m \defeq \frac{n - n_0}{2}$ in our analysis of $\hD_{\mathrm{BW}, L}$.

\begin{theorem}\label{thm:rtd-estimation}
Under Model~\ref{model:one-side} and Assumption~\ref{assum:normalization}, the estimator \eqref{eq:rtd_def} with
$L =  \lceil \frac{\log (d/\eps)}{\log (1/\rho)} \rceil + C$
for a sufficiently small $\rho > 0$ and sufficiently large $C \in \N$ (independent of $d$, $\eps$), solves Problem~\ref{prob:FD_est} if $n \gtrsim \frac d {\eps^2}$. Moreover, if $n = \Theta(\frac{d}{\eps^2})$, the estimator can be computed in time 
\[O\Par{\frac{d^{\omega + \gamma}}{\eps^{2 + \gamma}}},\]
where $\gamma > 0$ depends only on $\rho$, and can be taken to be arbitrarily small.
\end{theorem}
\begin{proof}
We begin by discussing correctness. The proof of Theorem~\ref{thm:first-order-plugin-bias-variance} shows that with probability $\ge \frac 7 8$, the deviation of $\norm{\vmu - \hvnu_{n - n_0}}_2^2$ from $\norm{\vmu - \vnu}_2^2$ is $\le \frac \eps 3$. Next, 
with probability $\ge 0.99$, Fact~\ref{fact:pilot} implies \eqref{eq:tmbd_def} holds with $n_0 \gtrsim_\rho d$. Under this event, taking $L$ as stated bounds the bias of $\hD_{\mathrm{BW}, L}$ with respect to $\BW^2(\ma, \mb)$ by $\frac \eps 3$, using Proposition~\ref{prop:cross-fit-statistical}. Finally, since $m = \frac{n - n_0}{2} = \Omega(n)$, Chebyshev's inequality and Proposition~\ref{prop:cross-fit-statistical} imply that $\hD_{\mathrm{BW}, L}$ deviates from its mean by $\le \frac \eps 3$ with probability $\ge \frac 7 8$. Taking a union bound over all three events proves that \eqref{eq:rtd_def} solves Problem~\ref{prob:FD_est}.

Next, for the runtime, the empirical mean $\hvnu_{n - n_0}$ takes $O(nd)$ time to compute, and the pilot covariance $\tmb$ can be formed in time $O(d^\omega)$ because $n_0 = O(d)$. We can also whiten all of the samples by $\tmb^{-1/2}$ in time $O(nd^{\omega - 1})$, and therefore form all of the group means $\bmy_{k, i}$ in \eqref{eq:cross-fit-group-means}, for a fixed $k$, in time $O(nd^{\omega - 1} + kd^\omega)$. The cost of forming all group means across all $k$ is thus at most $O(nd^{\omega - 1}L^2)$. All of these steps fit within the stated runtime budget.

It remains to bound the cost of forming all of the $\hQ_k$ used to compute $\hD_{\mathrm{BW}, L}$. Given the group means, the dominant cost comes from summing Lemma~\ref{lem:implement_polarization} across all $2 \le k \le L$, i.e.,
\[\sum_{k = 2}^L O\Par{2^k k^2 d^\omega} = O\Par{2^L d^\omega \log^2\Par{\frac d \eps}}.\]
Choosing a sufficiently small constant $\rho$ then leads to the runtime claim.
\end{proof}

\section{Gaussian Controlled Experiments}
\label{sec:gaussian_exp}
Previous theoretical analysis has established tight error characterizations for all methods considered. We now use Gaussian experiments to validate the tightness of our analysis and the lower-bound constructions. We first verify the hard isotropic case proposed in Section~\ref{ssec:extrapolation-bias-lower-bound}, then show that the predicted dimension dependence appears more broadly across a range of Gaussian families.

\textbf{Experimental setup.} We use Model~\ref{model:one-side} with known covariance $\ma$ and known zero means $\vmu = \vnu = \vzero_d$. We write $d$ for the dimension and $N$ for the total number of sampled observations. We denote the estimators by FID-OLS$_k$, FID-VALE$_k$, and FID-RTD for FID estimation, where $k$ is the extrapolation order. We omit the FID- prefix for other settings and in comparisons spanning multiple embeddings.
All extrapolation methods in Gaussian experiments follow the regression schedule in~\cite{chong2019effectively}, using 15 sample counts uniformly spaced over $[5000,N]$ up to integer rounding; RTD uses $m_0=0.8N$ pilot observations and $m=N-m_0$ correction observations. We also provide a detailed setup of these methods to Appendix~\ref{app:gaussian_exp_extra}.

\textbf{The hard isotropic instance.} We use the same scaled identity covariance family as in the bias lower-bound construction~\eqref{eq:extrapolation-hard-family}. This tests whether the residual bias predicted by the theory remains visible at finite sample sizes, even in this simple setting. In Table~\ref{tab:hard-isotropic}, increasing the extrapolation order reduces the bias, but the errors of OLS$_1$, OLS$_2$, and OLS$_3$ remain much larger than their trial standard deviations. RTD's mean error is comparable to its trial standard deviation. Thus, the residual extrapolation bias identified by our bounds is already visible for isotropic Gaussians.

\begin{table}[!ht]
\centering
\small
\caption{Isotropic experiment with $D_0=5$, $d=5000$, and $N=100{,}000$. Mean and standard deviation are reported across 50 trial sample trajectories.}
\label{tab:hard-isotropic}
\begin{tabular}{@{}lrrrrr@{}}
\toprule
 & Empirical & OLS$_1$ & OLS$_2$ & OLS$_3$ & RTD \\
\midrule
Mean estimate & 37.8610 & -9.1300 & 6.9197 & 4.7519 & 4.9033 \\
Signed error  & 32.8610 & -14.1300 & 1.9197 & -0.2481 & -0.0967 \\
Standard Deviation  & 0.0139 & 0.0316 & 0.0290 & 0.0447 & 0.1164 \\
\bottomrule
\end{tabular}
\end{table}
\textbf{Random Gaussian families.} We next test the predicted dimension dependence on random covariance pairs. Figure~\ref{fig:controlled-random-error} reports the error of 5 different estimators on a family of random Gaussian covariance pairs at four fixed sample budgets. At each fixed $N$, the empirical estimator's error grows approximately as $d^2$, while OLS$_1$ exhibits approximately $d^3$ growth at larger dimensions, matching the theoretically predicted bias scalings. RTD follows a trend close to $d^{1/2}$, as predicted by its optimal error bound in Section~\ref{sec:optimal-estimation}. OLS$_2$ and OLS$_3$ also show trends close to $d^{1/2}$ over the tested grid because their bias decays with higher powers of $N^{-1}$ and is not obvious over the tested range. Additional results for other random families, as well as rank-deficient and ill-conditioned cases, appear in Appendix~\ref{app:gaussian_exp_extra} and show broadly consistent qualitative trends.
\begin{figure}[!ht]
    \centering
    \includegraphics[width=\textwidth]{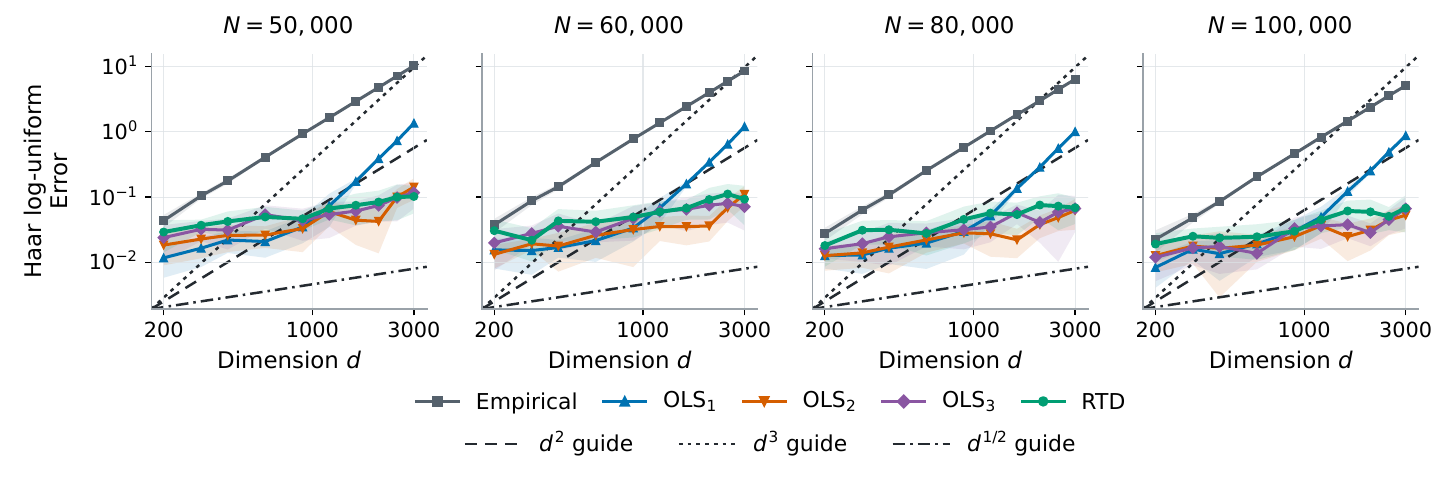}
    \caption{Estimation error versus dimension for Haar log-uniform covariances. At each $(d,N)$, we compute the median absolute error over five trials for each of ten independent population pairs. Curves show the mean of these ten medians, with shaded bands indicating $\pm1$ sample SD.}
    \label{fig:controlled-random-error}
\end{figure}

\section{ImageNet Experiments}
\label{sec:imagenet_exp}
In this section, we evaluate the practical accuracy and sample efficiency of RTD and extrapolation methods on ImageNet.
Extrapolations involve several tunable hyperparameters, so we begin with hyperparameter selection on the development pool. We then compare RTD and optimized VALE against the empirical estimator and FID$_\infty$ under standard 50K-sample setup. Finally, we reduce the sample budget to assess which estimators maintain robust performance with fewer samples.

\textbf{Experimental setup.} We evaluate FID estimation on ImageNet using seven generator checkpoints spanning GAN, diffusion, and autoregressive models~\cite{sauer2022stylegan,karras2024edm,zheng25ddo,tian2024visual}, with Inception-V3, DINOv2, and CLIP embeddings~\cite{heusel2017gans,stein2023exposing,jayasumana2024rethinking}. For each checkpoint, 550K pictures are generated in total: the first 300K form a development pool for proxy construction and ablation, and the remaining 250K is evenly split into five held-out blocks. Details are provided in Appendix~\ref{app:imagenet-setup}.

\textbf{Reference proxies.} Since population distances are unknown, we evaluate sensitivity to three reference proxies computed on the full 300K development pool: the empirical plug-in estimate $Q_c^{\mathrm{plug}}$, FID$_\infty$ $Q_c^\infty$, and RTD $Q_c$, for each checkpoint--embedding pair $c$. The latter two are medians of five randomized estimates on the same pool. Figure~\ref{fig:imagenet-proxy-example} compares StyleGAN-XL-256 across Inception-V3, DINOv2, and CLIP over 30K--300K samples. The empirical estimates remain noticeably higher, whereas RTD and FID$_\infty$ broadly agree across the tested sample budgets. Additional results for the remaining checkpoints and all three embeddings appear in Appendix~\ref{app:imagenet-reference-sensitivity}.

\newcommand{\imagenetProxySwatch}[2]{%
    \begingroup
    \setbox0=\hbox{\includegraphics{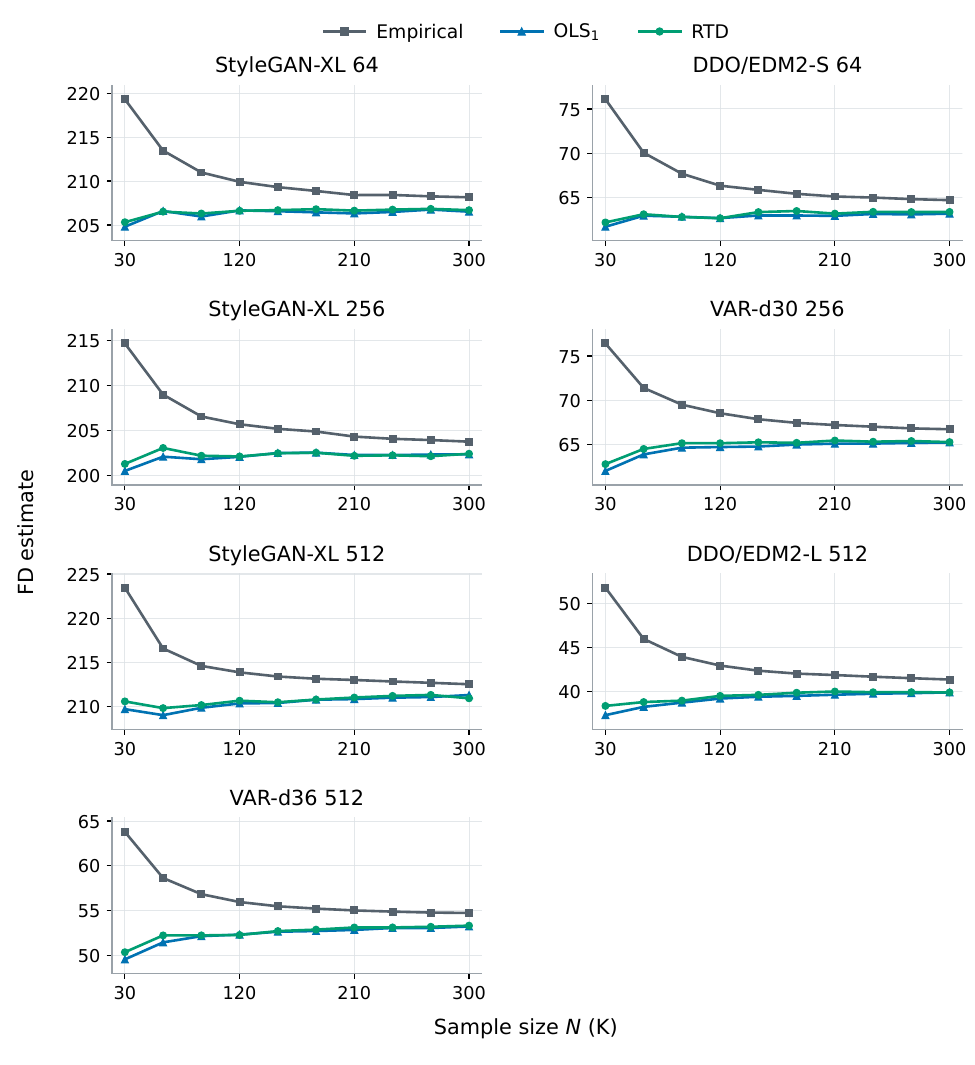}}%
    \edef\imagenetProxySample{\noexpand\includegraphics[width=1.8em,trim=\the\dimexpr\wd0*#1/1000\relax\space \the\dimexpr\ht0*960/1000\relax\space \the\dimexpr\wd0*#2/1000\relax\space \the\dimexpr\ht0*20/1000\relax,clip]{figures/imagenet_proxy_consistency_dinov2.pdf}}%
    \raisebox{-0.2ex}{\imagenetProxySample}%
    \endgroup
}
\newcommand{\imagenetProxyLegend}{%
    \begingroup\footnotesize\sffamily
    \imagenetProxySwatch{325}{620}\,Empirical\hspace{1.4em}%
    \imagenetProxySwatch{510}{435}\,FID$_\infty$\hspace{1.4em}%
    \imagenetProxySwatch{650}{295}\,RTD%
    \endgroup
}
\begin{figure}[!ht]
\centering
\begingroup
\imagenetProxyLegend\par\vspace{2pt}
\begin{minipage}[t]{0.32\textwidth}
\centering
{\small Inception-V3\par}
\setbox0=\hbox{\includegraphics{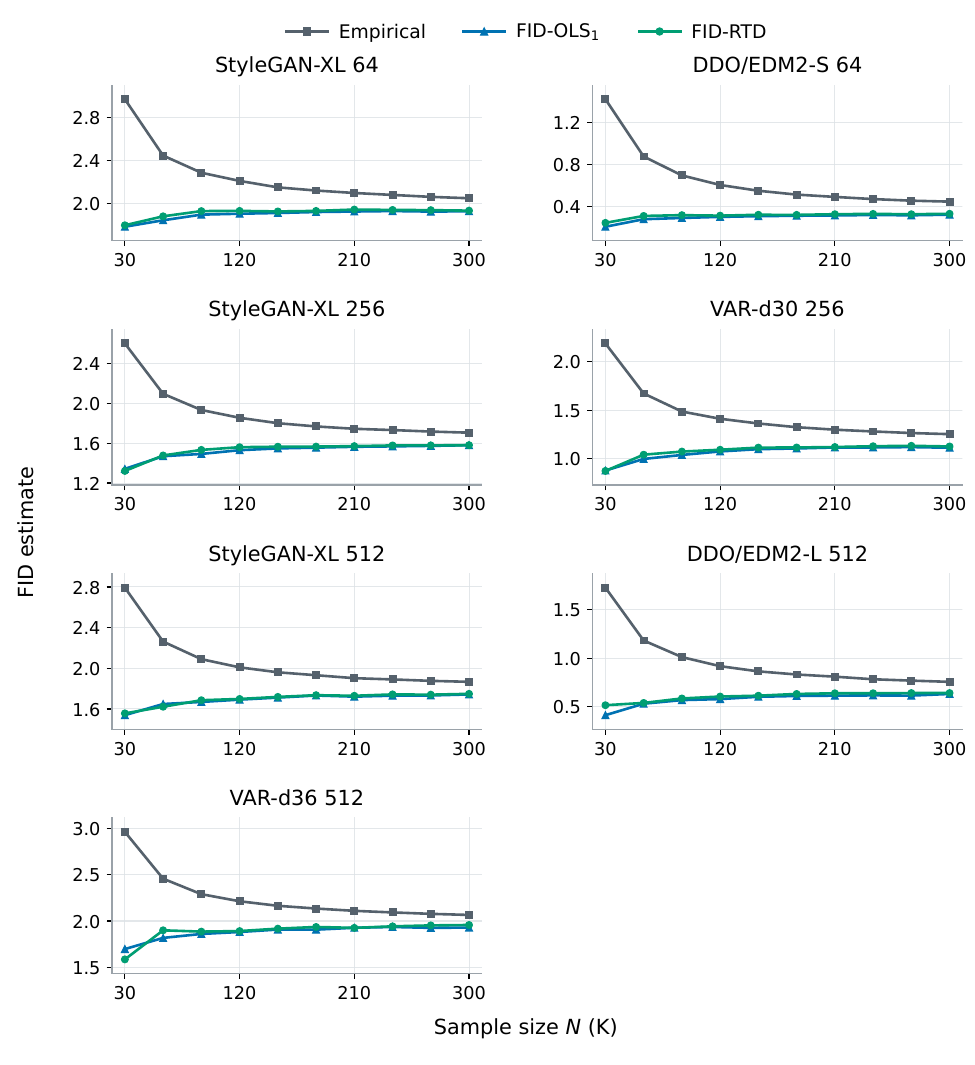}}
\edef\proxyexamplegraphic{\noexpand\includegraphics[width=\noexpand\linewidth,trim=\the\dimexpr\wd0*5/100\relax\space \the\dimexpr\ht0*51/100\relax\space \the\dimexpr\wd0*49/100\relax\space \the\dimexpr\ht0*31/100\relax,clip]{figures/imagenet_proxy_consistency_inception.pdf}}
\proxyexamplegraphic
\end{minipage}\hfill
\begin{minipage}[t]{0.32\textwidth}
\centering
{\small DINOv2\par}
\setbox0=\hbox{\includegraphics{figures/imagenet_proxy_consistency_dinov2.pdf}}
\edef\proxyexamplegraphic{\noexpand\includegraphics[width=\noexpand\linewidth,trim=\the\dimexpr\wd0*5/100\relax\space \the\dimexpr\ht0*51/100\relax\space \the\dimexpr\wd0*49/100\relax\space \the\dimexpr\ht0*31/100\relax,clip]{figures/imagenet_proxy_consistency_dinov2.pdf}}
\proxyexamplegraphic
\end{minipage}\hfill
\begin{minipage}[t]{0.32\textwidth}
\centering
{\small CLIP\par}
\setbox0=\hbox{\includegraphics{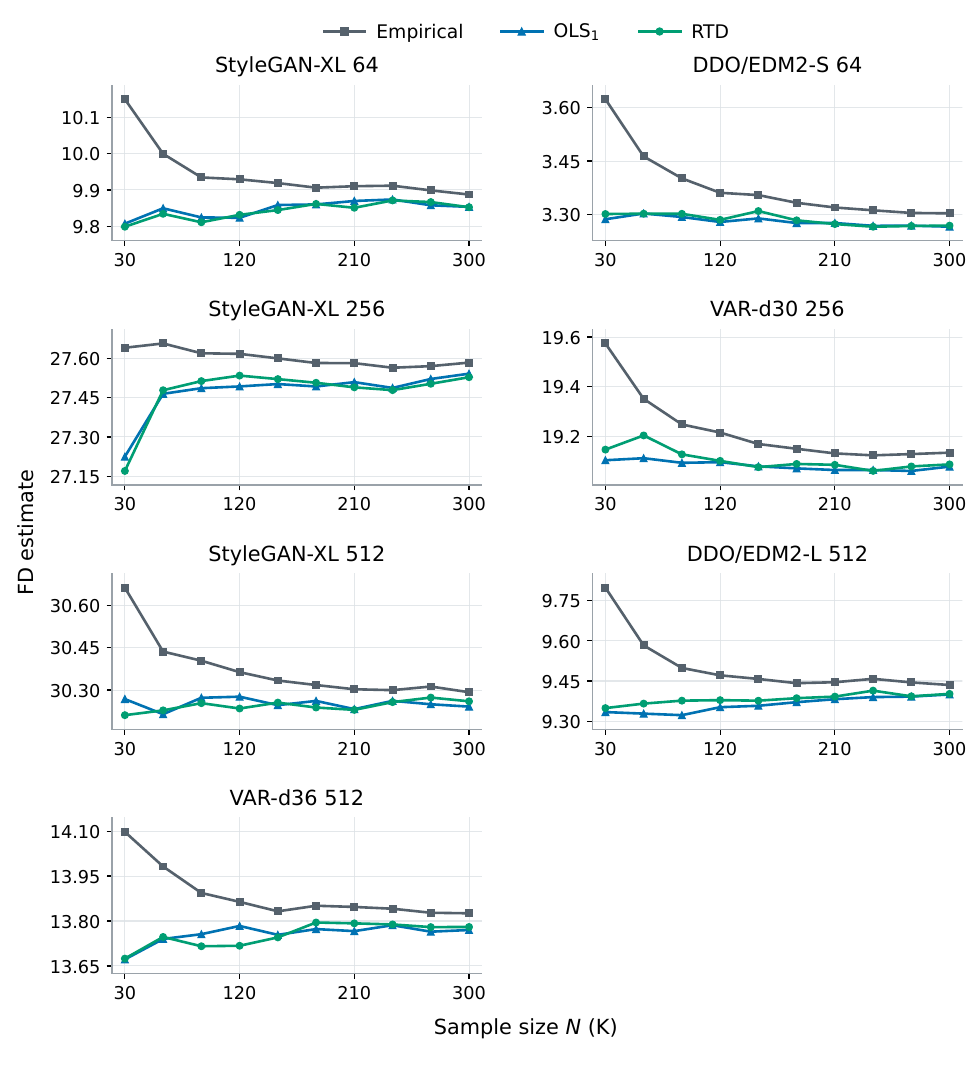}}
\edef\proxyexamplegraphic{\noexpand\includegraphics[width=\noexpand\linewidth,trim=\the\dimexpr\wd0*5/100\relax\space \the\dimexpr\ht0*51/100\relax\space \the\dimexpr\wd0*49/100\relax\space \the\dimexpr\ht0*31/100\relax,clip]{figures/imagenet_proxy_consistency_clip.pdf}}
\proxyexamplegraphic
\end{minipage}%
\endgroup
\caption{Sample-budget dependence for StyleGAN-XL at $256\times256$ resolution with Inception-V3 (left), DINOv2 (middle), and CLIP (right). Horizontal axes show $N$ (thousands); vertical axes show FID estimates. FID$_\infty$ and RTD report medians of five estimates.}
\label{fig:imagenet-proxy-example}
\end{figure}

\textbf{Utility.} Table~\ref{tab:imagenet-proxy-median-error} reports held-out errors against all three proxies to assess sensitivity to this choice. 
Because $Q_c^\infty$ and $Q_c$ closely agree in all settings we considered, the ablation studies and reduced-budget experiments use the RTD proxy $Q_c$. At the standard 50K budget, RTD achieves the lowest mean error on Inception across all three proxies. Optimized low-order VALE is particularly favorable for DINOv2 and CLIP, with VALE$_2$ leading on DINOv2 across all proxies. This pattern is consistent with the bias--variance tradeoff: in lower dimensions, low-order extrapolation can control bias while keeping variance small. The VALE configurations are selected on the development pool (Table~\ref{tab:imagenet-optimal-schedule}).

\textbf{Ablation of extrapolation setup.} Table~\ref{tab:imagenet-ablation-summary} quantifies the gains from optimizing extrapolation on 50K development-pool trials. Averaged over orders $k=1,2,3$, variance-aware weighting reduces the mean RMSE of OLS by $17.0$--$26.1\%$ across the three embeddings; optimizing the regression-point schedule and count yields a further $5.6$--$11.7\%$ reduction. Detailed comparisons, selected configurations, and the complete regression grid appear in Appendix~\ref{app:imagenet-ablation}.

\textbf{Computation.} Table~\ref{tab:imagenet-computation-time} reports estimator runtimes for held-out 50K trials in our computational environment (Appendix~\ref{ssec:computational-environment}). All reported mean runtimes are $\le 2$ seconds. RTD has the highest runtime for Inception and DINOv2, while the cost of VALE$_k$ depends on its selected regression grid. Notably, these estimation costs are negligible compared with the cost of sample generation.

\textbf{Reducing the sample budget.} We next test whether these accuracy gains translate into lower sample-generation costs, varying $N$ from 10K to 50K with $n_{\min} = 5$K for extrapolations. On Inception, FID-VALE$_2$ achieves the lowest mean proxy error among the evaluated estimators at 10K (Table~\ref{tab:imagenet-fd-sample-budget-heldout}). Its error of $0.1368$ is comparable to FID$_\infty$ at 50K ($0.1352$), but it only uses one fifth as many generated samples. RTD performs best at 50K, but its error increases sharply at 10K. Results for DINOv2 and CLIP appear in Appendix~\ref{app:imagenet-sample-budget}.

\textbf{Takeaways.} We briefly conclude with some recommendations for practitioners. Our experiments confirmed that the plug-in estimator exhibits substantial finite-sample error. At the standard 50K sample budget, based on our findings, our recommendation is to use RTD for more accurate evaluation. In sample-constrained settings, we find VALE$_2$ with 10K samples to be a robust choice.
\begin{table}[!ht]
\centering
\footnotesize
\setlength{\tabcolsep}{4pt}
\caption{Held-out utility and sensitivity to the proxy at $N=50\mathrm{K}$. For each checkpoint and proxy, we take the median absolute error over five held-out trials; entries report means $\pm$ sample SD.}
\label{tab:imagenet-proxy-median-error}
\label{tab:imagenet-methods-comparison}
\begin{tabular}{@{}llccc@{}}
\toprule
Embedding & Estimator & \shortstack{Empirical proxy\\$Q_c^{\mathrm{plug}}$} & \shortstack{FID$_\infty$ proxy\\$Q_c^\infty$} & \shortstack{RTD proxy\\$Q_c$} \\
\midrule
Inception & Empirical & $0.5108 \pm 0.0264$ & $0.6389 \pm 0.0267$ & $0.6287 \pm 0.0268$ \\
 & FID$_\infty$ & $0.2531 \pm 0.0456$ & $0.1250 \pm 0.0423$ & $0.1352 \pm 0.0471$ \\
 & VALE$_1$ & $0.2410 \pm 0.0406$ & $0.1129 \pm 0.0369$ & $0.1231 \pm 0.0412$ \\
 & VALE$_2$ & $0.2224 \pm 0.0372$ & $0.0944 \pm 0.0347$ & $0.1046 \pm 0.0356$ \\
 & VALE$_3$ & $0.2328 \pm 0.0383$ & $0.1048 \pm 0.0345$ & $0.1150 \pm 0.0376$ \\
 & RTD & $\mathbf{0.1972 \pm 0.0422}$ & $\mathbf{0.0692 \pm 0.0418}$ & $\mathbf{0.0794 \pm 0.0416}$ \\
\midrule
DINOv2 & Empirical & $5.5063 \pm 0.2975$ & $6.9704 \pm 0.3749$ & $6.9433 \pm 0.2378$ \\
 & FID$_\infty$ & $3.0186 \pm 0.3095$ & $1.5545 \pm 0.3459$ & $1.5816 \pm 0.3008$ \\
 & VALE$_1$ & $2.9861 \pm 0.3162$ & $1.5220 \pm 0.3335$ & $1.5492 \pm 0.2806$ \\
 & VALE$_2$ & $\mathbf{2.6939 \pm 0.4000}$ & $\mathbf{1.2409 \pm 0.3962}$ & $\mathbf{1.2570 \pm 0.3987}$ \\
 & VALE$_3$ & $2.8119 \pm 0.5786$ & $1.4883 \pm 0.6377$ & $1.4678 \pm 0.5907$ \\
 & RTD & $2.7139 \pm 0.4797$ & $1.3454 \pm 0.4310$ & $1.3573 \pm 0.3539$ \\
\midrule
CLIP & Empirical & $0.1851 \pm 0.0636$ & $0.2296 \pm 0.0656$ & $0.2255 \pm 0.0623$ \\
 & FID$_\infty$ & $0.0838 \pm 0.0523$ & $0.0645 \pm 0.0423$ & $0.0652 \pm 0.0416$ \\
 & VALE$_1$ & $0.0858 \pm 0.0408$ & $0.0633 \pm 0.0394$ & $0.0651 \pm 0.0384$ \\
 & VALE$_2$ & $\mathbf{0.0787 \pm 0.0565}$ & $0.0640 \pm 0.0255$ & $0.0657 \pm 0.0336$ \\
 & VALE$_3$ & $0.1004 \pm 0.0378$ & $0.0803 \pm 0.0425$ & $0.0735 \pm 0.0367$ \\
 & RTD & $0.0831 \pm 0.0463$ & $\mathbf{0.0527 \pm 0.0303}$ & $\mathbf{0.0553 \pm 0.0304}$ \\
\bottomrule
\end{tabular}
\end{table}
\begin{table}[!ht]
\centering
\footnotesize
\setlength{\tabcolsep}{4pt}
\caption{Successive ablations on 50K development-pool trials. RMSE is averaged over seven checkpoints and orders $k=1,2,3$; parentheses show reductions from preceding row.}
\label{tab:imagenet-ablation-summary}
\begin{tabular}{@{}lccc@{}}
\toprule
Configuration & Inception & DINOv2 & CLIP \\
\midrule
FID$_\infty$, OLS$_2$, OLS$_3$ & $0.0468$ & $1.0702$ & $0.1121$ \\
$+$ VALE weighting & $0.0346\;(\downarrow26.1\%)$ & $0.8696\;(\downarrow18.7\%)$ & $0.0930\;(\downarrow17.0\%)$ \\
$+$ Regression strategy & $\mathbf{0.0308}\;(\downarrow11.2\%)$ & $\mathbf{0.7679}\;(\downarrow11.7\%)$ & $\mathbf{0.0878}\;(\downarrow5.6\%)$ \\
\bottomrule
\end{tabular}
\end{table}

\begin{table}[!ht]
\centering
\footnotesize
\setlength{\tabcolsep}{2.5pt}
\caption{Estimator runtime (ms) for 50K trials, measured by mean $\pm$ sample SD over all 35 runs.}
\label{tab:imagenet-computation-time}
\begin{tabular}{@{}lrrrrrr@{}}
\toprule
Embedding & Empirical & FID$_\infty$ & VALE$_1$ & VALE$_2$ & VALE$_3$ & RTD \\
\midrule
\shortstack[l]{Inception}
& $64.9 \pm 0.2$
& $915.3 \pm 3.0$
& $123.1 \pm 0.3$
& $890.2 \pm 2.1$
& $1181.1 \pm 3.6$
& $1976.7 \pm 5.3$ \\

\shortstack[l]{DINOv2}
& $25.9 \pm 0.1$
& $368.8 \pm 1.1$
& $50.3 \pm 0.3$
& $358.9 \pm 1.1$
& $478.0 \pm 1.5$
& $559.3 \pm 2.0$ \\

\shortstack[l]{CLIP}
& $19.2 \pm 0.1$
& $272.6 \pm 1.8$
& $262.8 \pm 2.4$
& $354.9 \pm 7.7$
& $362.3 \pm 2.4$
& $217.7 \pm 12.2$ \\
\bottomrule
\end{tabular}
\end{table}
\begin{table}[!ht]
\centering
\footnotesize
\setlength{\tabcolsep}{2pt}
\caption{Sample-budget sensitivity for Inception. Entries report the mean $\pm$ sample SD across seven checkpoints of the median absolute error over five held-out trials against the RTD proxy.}
\label{tab:imagenet-fd-sample-budget-heldout}
\begin{tabular}{@{}rrrrrrr@{}}
\toprule
$N$ & Empirical & FID$_\infty$ & VALE$_1$ & VALE$_2$ & VALE$_3$ & RTD \\
\midrule
10K & $3.588\pm0.134$ & $0.211\pm0.058$ & $0.223\pm0.073$ & $\mathbf{0.137}\pm0.053$ & $1.014\pm0.579$ & $101.592\pm23.510$ \\
20K & $1.724\pm0.060$ & $0.159\pm0.054$ & $0.153\pm0.048$ & $\mathbf{0.107}\pm0.053$ & $0.172\pm0.067$ & $0.490\pm0.119$ \\
30K & $1.107\pm0.032$ & $0.144\pm0.030$ & $0.137\pm0.028$ & $0.108\pm0.050$ & $\mathbf{0.100}\pm0.046$ & $0.149\pm0.059$ \\
40K & $0.804\pm0.026$ & $0.141\pm0.033$ & $0.130\pm0.033$ & $\mathbf{0.112}\pm0.040$ & $0.123\pm0.038$ & $0.118\pm0.044$ \\
50K & $0.629\pm0.027$ & $0.135\pm0.047$ & $0.123\pm0.041$ & $0.105\pm0.036$ & $0.115\pm0.038$ & $\mathbf{0.079}\pm0.042$ \\
\bottomrule
\end{tabular}
\end{table}

\section*{Acknowledgements}
We thank the NSF AI Institute for Foundations of Machine Learning (IFML) for their support of this project, and the Texas Advanced Computing Center (TACC) for providing computing resources.

\section*{AI Disclosure}
\phantomsection
\addcontentsline{toc}{section}{AI Disclosure}

The authors initially formulated the FID estimation problem and derived a $O(\frac{d^2}{\eps^2})$ sample complexity bound for the empirical plug-in estimator. We then consulted GPT 5.6 Sol Pro on whether the upper bound could be improved, and through interactive conversations, discovered the RTD strategy in Section~\ref{sec:optimal-estimation} with LLM assistance. Motivated by  prior work on $\mathrm{FID}_\infty$ \cite{chong2019effectively}, the authors in parallel proposed the extrapolation framework in Section~\ref{sec:extrapolation} (along with empirical improvements, e.g., the variance-aware choice of regression weights). The authors finally developed a unified analysis establishing tight bounds for all estimators studied. The manuscript was written solely by the authors, who take full responsibility for the organization and presentation of all results.

\bibliography{refs}
\bibliographystyle{alpha}

\newpage
\appendix
\section{Information-Theoretic Lower Bound}\label{app:lb}

In this section, we prove our information-theoretic lower bound on FID estimation. Our proof is based on Le Cam's two-point method, and we recall a standard statement for convenience.

\begin{lemma}[Theorem 2.2(i), \cite{Tsybakov09}]\label{lem:lecam}
Let $P_0, P_1$ be probability distributions on the same measurable space $\Omega$, and let $\phi: \Omega \to \{0, 1\}$ be a (possibly randomized) function. Then,
\begin{equation}\label{eq:lecam}\Pr_{\omega \sim P_0}\Brack{\phi(\omega) = 1} + \Pr_{\omega \sim P_1}\Brack{\phi(\omega) = 0} \ge 1 - \TV\Par{P_0, P_1}.\end{equation}
\end{lemma}

\restatelb*
\begin{proof}
Let $\tau = \frac{4\eps}{d}$, and define
\begin{equation*}
    P \defeq \calN\paren{\vzero_d, \frac{1}{8}\id_d}, \quad Q_0 \defeq \calN\paren{\vzero_d, \frac{1}{2}\id_d}, \quad Q_1 \defeq \calN\paren{\vzero_d,\paren{\frac{1}{2} + \tau}\id_d}.
\end{equation*}
Since $\eps \le \frac 1 8$, we have $0<\tau\le \half$, so $(P, Q_0)$ and $(P, Q_1)$ both satisfy Assumption~\ref{assum:normalization}. Also,
\begin{equation*}
    D_0\defeq \FID(P,Q_0)
    =
    d\paren{\sqrt{\frac18}-\sqrt{\frac12}}^2
    =
    \frac d8, 
    \quad
    D_1\defeq \FID(P, Q_1) = d\paren{\sqrt{\frac18}-\sqrt{\frac12+\tau}}^2,
\end{equation*}
so
\begin{align*}
    D_1 - D_0
    &=
    4\eps
    -\frac{d}{\sqrt2}
    \paren{
        \sqrt{\frac12+\tau}
        -\sqrt{\frac12}
    }=
    4\eps\left(
        1-
        \frac{1}{
            \sqrt{1+2\tau}+1
        }
    \right)
    > 2\eps,
\end{align*}
where the strict inequality follows from $\tau>0$. Therefore, an algorithm that solves Problem~\ref{prob:FD_est} under Model~\ref{model:one-side} can distinguish between draws from $Q_0^{\otimes n}$ and $Q_1^{\otimes n}$ with probability $\ge \frac 2 3$, i.e., it provides a test $\phi$ such that when $P_0 \gets Q_0^{\otimes n}$ and $P_1 \gets Q_1^{\otimes n}$, the left-hand side of \eqref{eq:lecam} is at most $\frac 1 3 + \frac 1 3$.

On the other hand, the KL divergence between the two Gaussian distributions is
\begin{align*}
    \KL(Q_0\Vert Q_1)
    &=
    \frac d2\left(
        \frac{1}{1 + 2\tau}
        -1
        +\log(1 + 2\tau)
    \right) \le \frac d4\paren{\frac{8\eps}{d}}^2 = \frac{16\eps^2}{d},
\end{align*}
where we used 
$\frac{1}{1+x}-1+\log(1+x)\le \frac{x^2}{2}$.
Therefore, by tensorization and Pinsker's inequality,
\begin{equation*}
    \TV(Q_0^{\otimes n},Q_1^{\otimes n})
    \le
    \sqrt{
        \frac12
        \KL(Q_0^{\otimes n}\Vert Q_1^{\otimes n})
    }
    \le
    \sqrt{
        \frac{8n\eps^2}{d}
    }.
\end{equation*}
Thus, if $n\le \frac{d}{128\eps^2}$, then $\TV(Q_0^{\otimes n},Q_1^{\otimes n}) \le \frac{1}{4}$, so Lemma~\ref{lem:lecam} gives $\frac 2 3 \ge 1 - \frac 1 4$, a contradiction.
\end{proof}
\section{Additional Results for Extrapolation}

\subsection{Subsampling from a shared pool}
\label{ssec:shared-pool-extrapolation}

Theorem~\ref{thm:general-order-extrapolation} (and in general, our estimator construction \eqref{eq:extrapolation-estimator}) assumes that different extrapolation levels are
computed from independent samples. The original FID$_\infty$ procedure in \cite{chong2019effectively}
instead generates a single pool of activations and obtains
estimates by repeatedly subsampling this common pool. In this section, we extend our analysis to the shared pool setting.

We follow the
shuffle-and-prefix rule in Algorithm~1 of \cite{chong2019effectively}:
reshuffle the pool independently at each level and take the required number
of observations.\footnote{Section~4.3 of \cite{chong2019effectively} uses the
phrase ``with replacement,'' but Algorithm~1 and the
\href{https://github.com/mchong6/FID_IS_infinity/blob/master/score_infinity.py}{official implementation}
shuffle the pool and take a prefix, which samples without replacement within
each level.}
Our analysis assumes an i.i.d.\ Gaussian pool and does not cover the QMC
sampling variant. We now show that shared-pool subsampling leaves the
marginal bias and variance at each level unchanged, and only introduces
cross-level dependence through the overlap between the subsamples.

Let $\{\vX_j\}_{j \in [N]}$ be the shared sample pool. For each level $i$, draw a
uniform subset $I_i\subseteq[N]$ of size $\vn_i=\vc_iN$, independent of the sample pool and independently
across levels. Thus sampling is
without replacement within each level, and subsets from different levels
may overlap. Note that unlike in Section~\ref{sec:extrapolation}, the $\vc_i$ need not sum to $1$, and hence the $\vn_i$ may not sum to $N$.

Write $\calI \defeq (I_1,\ldots,I_\ell)$, and let $\hat D_{I_i}$
denote the empirical Fr\'echet estimator computed from the samples indexed
by $I_i$. We first give a bound on each estimator used in the extrapolation.

\begin{proposition}
\label{prop:shared-pool-marginal}
For every level $i \in [\ell]$, conditional on the subset $I_i$, we have for some $|\va_1| \lesssim d^2$,
\[
    \bbE[\hat D_{I_i}\mid I_i]-D
    =
    \frac{\va_1}{\vn_i}
    +
    O\paren{\frac{d^3}{\vn_i^2}},
    \quad
    \operatorname{Var}\Brack{\hat D_{I_i}\mid I_i}
    \lesssim
    \frac{d}{\vn_i}
    +
    \frac{d^2}{\vn_i^2}.
\]
Moreover, the latter variance bound admits the following refinement.
Writing $\vX_a=\vnu+\mb^{1/2}\vg_a$ with
$\vg_a\simiid\calN(\vzero,\id_d)$ and $\mk\defeq\mb-\mc$, define
\[
    \psi(\vg)
    \defeq
    2(\vnu - \vmu)^\top\mb^{1/2}\vg
    +
    \vg^\top\mk\vg
    -
    \Tr(\mk).
\]
Then, conditional on $I_i$,
\begin{equation}
\label{eq:shared-pool-influence-expansion}
    \hat D_{I_i}
    -
    \bbE[\hat D_{I_i}\mid I_i]
    =
    \frac1{\vn_i}
    \sum_{a\in I_i}\psi(\vg_a)
    +
    \mathcal R_i,
    \qquad
    \bbE[\mathcal R_i^2\mid I_i]
    \lesssim
    \frac{d^2}{\vn_i^2},
\end{equation}

\end{proposition}

\begin{proof}
Since $I_i$ is sampled independently of the common pool, conditional on $I_i$, the collection $\{X_a\}_{a\in I_i}$ consists of $\vn_i$ i.i.d.\ draws. The marginal bias and variance bounds therefore follow directly from Proposition~\ref{prop:uniform-psd-expansion} and Theorem~\ref{thm:first-order-plugin-bias-variance}. It remains only to establish the refined bound \eqref{eq:shared-pool-influence-expansion}.

As in the proof of
Theorem~\ref{thm:first-order-plugin-bias-variance}, the centered
fluctuation of $\hat D_{I_i}$ has three components: the sample-mean
fluctuation, the linear covariance fluctuation, and the nonlinear Bures remainder. We consider the three terms separately. Let $\bar\vg_i\defeq \frac 1 {\vn_i}\sum_{a\in I_i}\vg_a$. The centered sample-mean contribution is
\[
    \frac1{\vn_i}
    \sum_{a\in I_i}
    2(\vnu - \vmu)^\top\mb^{1/2}\vg_a
    +
    r_i^{\mathrm{mean}},
    \quad
    r_i^{\mathrm{mean}}
    \defeq
    \twonorm{\mb^{1/2}\bar\vg_i}^2
    -
    \frac{\Tr(\mb)}{\vn_i}.
\]
Since $\bar\vg_i\sim\calN(\vzero_d,\frac 1 {\vn_i}\id_d)$, Gaussian moment identities give $\bbE[(r_i^{\mathrm{mean}})^2\mid I_i] =\frac 2 {\vn_i^2} \Tr(\mb^2) \lesssim \frac d{\vn_i^2}$. 

For the covariance  contribution, let $\widehat \mw_i$ be the unbiased sample covariance of $\{\vg_a\}_{a\in I_i}$ and write $\mh_i\defeq\widehat \mw_i-\id_d$. Comparing $\widehat \mw_i$ with the raw
second moment gives
\[
    \inprod{\mk}{\mh_i}
    =
    \frac1{\vn_i}
    \sum_{a\in I_i}
    \paren{
        \vg_a^\top\mk\vg_a-\Tr(\mk)
    }
    +
    \rho_i,
\]
where
\[
    \rho_i
    \defeq
    \frac1{\vn_i(\vn_i-1)}
    \sum_{a\in I_i}
    \vg_a^\top\mk\vg_a
    -
    \frac{\vn_i}{\vn_i-1}
    \bar\vg_i^\top\mk\bar\vg_i.
\]
As in the proof of Theorem~\ref{thm:first-order-plugin-bias-variance}, independence of the sample mean and sample covariance yields $\bbE[\rho_i\mid I_i]=0$ and $\bbE[\rho_i^2\mid I_i] \lesssim\frac 1 {\vn_i^2}\norm{\mk}_{\mathrm F}^2 \lesssim \frac d {\vn_i^2}$. Finally, define the centered nonlinear remainder by
\[
    r_i^{\mathrm{nl}}
    \defeq
    \Phi_{\mc}(\mh_i)
    -
    \bbE[\Phi_{\mc}(\mh_i)\mid I_i].
\]
The same argument as in the proof of \eqref{eq:apply_pi} gives
$\bbE[(r_i^{\mathrm{nl}})^2\mid I_i]\lesssim \frac{d^2}{\vn_i^2}$.

Collecting the three contributions then gives
\eqref{eq:shared-pool-influence-expansion} with
$\mathcal R_i
\defeq
r_i^{\mathrm{mean}}+\rho_i+r_i^{\mathrm{nl}}$, and
$\bbE[\mathcal R_i^2\mid I_i]\lesssim \frac{d^2}{n_i^2}$.
\end{proof}

Proposition~\ref{prop:shared-pool-marginal} shows that the potential interaction across levels mainly depends on the term $\psi(\vg_a)$. Computing the cross-level covariance therefore reduces to counting the collisions $I_i\cap I_j$.
\begin{proposition}
\label{prop:shared-pool-covariance}
Conditional on $\calI$, letting $V \defeq \Var[\psi(\vg)]$, we have
\begin{equation}
\label{eq:shared-pool-level-covariance}
    \operatorname{Cov}
    \left(
        \hat D_{I_i},
        \hat D_{I_j}
        \mid\calI
    \right)
    =
    V
    \frac{|I_i\cap I_j|}{\vn_i \vn_j}
    +
    O_{\mathbf c}
    \left(
        \frac{d^{3/2}}{N^{3/2}}
        +
        \frac{d^2}{N^2}
    \right).
\end{equation}

For the same subsampling scheme and any fixed 
$\vw\in\R^\ell$ satisfying $\sum_{i\in [\ell]}\vw_i=1$, we have
\begin{equation}
\label{eq:shared-pool-weighted-variance}
    \operatorname{Var}
    \left[
        \sum_{i\in [\ell]}
        \vw_i\hat D_{I_i}
        \;\middle|\;\calI
    \right]
    =
    \Gamma_{\mathrm{pool}}(\mathbf w,\calI)
    \frac{V}{N}
    +
    O_{\mathbf c,\mathbf w}
    \left(
        \frac{d^{3/2}}{N^{3/2}}
        +
        \frac{d^2}{N^2}
    \right),
\end{equation}
where 
\begin{equation}
\label{eq:expected-shared-pool-factor}
    \Gamma_{\mathrm{pool}}(\mathbf w,\calI)
    \defeq
    N\sum_{i,j \in [\ell]}
    \vw_i\vw_j
    \frac{|I_i\cap I_j|}{n_i n_j}, 
    \quad 
    \bbE_{\calI}
    \Gamma_{\mathrm{pool}}(\mathbf w,\calI)
    =
    1+
    \sum_{i\in[\ell]}
    \vw_i^2
    \left(
        \frac1{c_i}-1
    \right).
\end{equation}
The unconditional variance is given by
\eqref{eq:shared-pool-weighted-variance} with
$\Gamma_{\mathrm{pool}}(\mathbf w,\calI)$ replaced by its expectation.
\end{proposition}

\begin{proof}
Until we average over $\calI$ below, all covariances and second moments
are conditional on $\calI$. The marginal bounds remain valid under
this conditioning because the subsets are independent of the observations.
Write the leading term in
\eqref{eq:shared-pool-influence-expansion} as
$L_i\defeq \frac 1 {\vn_i}\sum_{a\in I_i}\psi(\vg_a)$. Conditional on $\calI$,
independence of the observations in the common pool gives
\[
    \operatorname{Cov}(L_i,L_j)
    =
    \frac1{\vn_i \vn_j}
    \sum_{a\in I_i}
    \sum_{b\in I_j}
    \operatorname{Cov}
    \left(
        \psi(\vg_a),\psi(\vg_b)
    \right)
    =
    V\frac{|I_i\cap I_j|}{\vn_i \vn_j}.
\]
Indeed, every term with $a\neq b$ vanishes, while each shared observation
contributes one copy of $V$.

By Proposition~\ref{prop:shared-pool-marginal},
$\bbE[L_i^2]\lesssim_{\mathbf c} \frac dN$ and
$\bbE[\mathcal R_i^2]\lesssim_{\mathbf c}\frac {d^2}{N^2}$. Hence the 
Cauchy--Schwarz inequality gives
\[|\operatorname{Cov}(L_i,\mathcal R_j)|
\lesssim_{\mathbf c}\frac{d^{3/2}}{N^{3/2}},\]
with the same bound for
$\operatorname{Cov}(\mathcal R_i,L_j)$, while
$|\operatorname{Cov}(\mathcal R_i,\mathcal R_j)|
\lesssim_{\mathbf c} \frac{d^2}{N^2}$. This proves
\eqref{eq:shared-pool-level-covariance}.

For the weighted estimator, summing
\eqref{eq:shared-pool-level-covariance} against $\vw_i\vw_j$ gives
\eqref{eq:shared-pool-weighted-variance}. Equivalently, its leading
fluctuation can be written as
\[
    \sum_{i \in [\ell]} \vw_iL_i
    =
    \sum_{a\in[N]} q_a\psi(\vg_a),
    \quad
    q_a
    \defeq
    \sum_{i:a\in I_i}\frac{\vw_i}{\vn_i},
\]
so that
$N\sum_aq_a^2=\Gamma_{\mathrm{pool}}(\mathbf w,\calI)$. Independently sampled subsets satisfy $\bbE|I_i\cap I_j|=\frac{\vn_i\vn_j}{N}$ for $i\neq j$, whereas $|I_i\cap I_i|=\vn_i$. Therefore
\begin{align*}
    \bbE\Gamma_{\mathrm{pool}}
    &=
    \sum_{i \in [\ell]}\frac{\vw_i^2}{\vc_i}
    +
    \sum_{i\neq j}\vw_i\vw_j =
    \sum_i\frac{\vw_i^2}{\vc_i}
    +
    1-\sum_{i \in [\ell]} \vw_i^2,
\end{align*}
where the last equality uses $\sum_{i \in [\ell]}\vw_i=1$, proving \eqref{eq:expected-shared-pool-factor}.
Finally, $\bbE[\hat D_{I_i}\mid\calI]=\bbE[\hat D_{\vn_i}]$ depends only
on the fixed sample size $\vn_i$, not on the selected indices. The law of
total variance therefore gives the stated unconditional variance by
averaging \eqref{eq:shared-pool-weighted-variance} over $\calI$.
\end{proof}
Table~\ref{tab:shared-pool-gamma} reports this factor for extrapolation orders $k\in\{1,2\}$. All schedules span $5$K--$50$K samples. For each order, we compare ordinary least squares (OLS) with variance-aware regression (VA), and use either $J=8$, or $J=15$ points across different point selection strategies.

\begin{table}[H]
\centering
\small
\caption{Expected common-pool variance factor
$\bbE\Gamma_{\mathrm{pool}}$ for different extrapolation methods.}
\label{tab:shared-pool-gamma}
\begin{tabular}{@{}llrrrr@{}}
\toprule
Family
& Order
& Points $J$
& \shortstack[r]{Uniform\\in $n$}
& \shortstack[r]{Uniform\\in $\frac 1n$}
& \shortstack[r]{Chebyshev\\in $\frac 1n$} \\
\midrule

OLS & $k=1$ & 8
    & 1.507 & 2.148 & 1.681 \\
OLS & $k=1$ & 15
    & \textbf{1.342}$^\dagger$ & 1.743 & 1.407 \\
\cmidrule(lr){1-6}
VA & $k=1$ & 8
    & 1.218 & 1.231 & 1.190 \\
VA & $k=1$ & 15
    & 1.141 & 1.225 & 1.123 \\
\cmidrule(lr){1-6}
VA & $k=2$ & 8
    & 1.778 & 1.567 & 1.784 \\
VA & $k=2$ & 15
    & 1.470 & 1.543 & 1.472 \\
\bottomrule
\end{tabular}
\vspace{1mm}
\parbox{0.92\linewidth}{\footnotesize
$^\dagger$ Original FID$_\infty$: first-order OLS using $15$ sample sizes
uniformly spaced in $n$ between $5$K and $50$K.}
\end{table}
\subsection{Proof of Lemma~\ref{lem:sharp-isotropic-coefficients}}
\label{app:isotropic-extrapolation-coefficients}

We give a short proof of Lemma~\ref{lem:sharp-isotropic-coefficients}, our coefficient lower bound from Section~\ref{ssec:extrapolation-bias-lower-bound}, which applies to the instance \eqref{eq:extrapolation-hard-family} of zero-mean Gaussians with isotropic covariances, reproduced here for convenience:
\[\mb=\frac12\id_d,
    \quad
    \ma=a\id_d,
    \quad
    a\defeq
    \left(\sqrt{\frac12}+\sqrt{\frac{D_0}{d}}\right)^2.\]
We use the following standard consequence of the Marchenko--Pastur law
\cite{marchenko1967distribution,bai2010spectral}.
\begin{lemma}[Appendix B, \cite{montanaro2007distinguishability}]
\label{lem:mp-square-root-moment}
Let 
$\mw_m\sim\Wish(\frac 1 m \id_d,m)$, so that $\bbE\mw_m=\id_d$.
If $d,m\to\infty$ with $\frac d m \to\gamma\in[0,1)$, then
\[
    \frac1d\bbE\Tr\mw_m^{1/2}
    \to
    1-\frac12\phi(\gamma),
\]
where
\begin{equation}
\label{eq:mp-curvature-expansion}
    \phi(\gamma)
    \defeq
    \sum_{j\ge1}\alpha_j\gamma^j,
    \quad
    \alpha_j
    \defeq
    \frac{
        \binom{2j-2}{j-1}
        \binom{2j}{j}
    }{
        2^{4j-2}j(j+1)
    }
    >0.
\end{equation}
\end{lemma}

\SharpIsotropicCoefficients*

\begin{proof}
Throughout let $m \defeq n-1$, let $\{\vX_i\}_{i \in [n]} \sim \calN(\0_d, \mb)$ as in Model~\ref{model:one-side}, and let
\[\hvnu_n \defeq \frac 1 n \sum_{i \in [n]} \vX_i,\quad \ms_m \defeq \frac 1 {m} \sum_{i \in [n]}(\vX_i - \hvnu_n)(\vX_i - \hvnu_n)^\top\]
denote the empirical mean and scaled covariance. As shown in \eqref{eq:plugin} and Fact~\ref{fact:rescaling}, in the setting of \eqref{eq:extrapolation-hard-family}, $\ms_m \overset{\dd}{=} \half \mw_m$ where $\mw_m \sim \Wish(\frac 1 m\id_d, m)$, and the plug-in estimator is
\[\hdd_n = \norm{\hvnu_n}_2^2 + \BW^2\Par{a\id_d, \ms_m},\quad \BW^2\Par{a\id_d, \half \mw_m} = ad + \Tr(\ms_m) - 2\sqrt{a}\Tr(\ms_m^{1/2}). \]
Since $D_0 = \BW^2(a\id_d, \half \id_d)$, $\E\Tr(\ms_m) = \frac d 2$, and $\E[\norms{\hvnu_n}_2^2] = \frac d {2n}$, we can directly compute
\[
\E\Brack{\hdd_n} - D_0 = \frac{d}{2n} + \sqrt{2a}d\Par{1 - \frac 1 d \E \Tr \mw_m^{1/2}}.
\]
Then by plugging in Lemma~\ref{lem:mp-square-root-moment} and taking $d,n \to \infty$ with $\frac d n \to \gamma$,
\[\frac1 d \Par{\E\Brack{\hdd_n} - D_0 } \to \half \phi(\gamma), \]
where $\phi(\gamma) = \sum_{j \ge 1}\alpha_j \gamma^j$ is defined in \eqref{eq:mp-curvature-expansion}. Here we used that $a \to \half$ as $d \to \infty$.

On the other hand, whenever $n \gtrsim_k d$, Proposition~\ref{prop:uniform-psd-expansion} gives
\[\frac 1d \Par{\E\Brack{\hdd_n} - D_0}  = \sum_{j \in [k]} A_j(d)\Par{\frac d n}^j + O_k\Par{\Par{\frac d n}^{k + 1}}, \]
for coefficients $\{A_j \defeq \frac{\va_j}{d^{j + 1}}\}_{j \in [k]}$ with $|A_j| \lesssim_k 1$. We write $A_j(d)$ as a function of $d$ to disambiguate, because Proposition~\ref{prop:uniform-psd-expansion} applies to a different distribution at each $d$ (as $\ma$ is a function of $d$).

It remains to identify the coefficients. The vectors $(A_1(d),\ldots,A_k(d))$ are uniformly bounded. Consider any convergent subsequence with limit $(A_1^\star,\ldots,A_k^\star)$. For any sufficiently small fixed $\gamma > 0$, take $n = \lceil d/\gamma\rceil$ and let $d \to \infty$ along this subsequence. The preceding limit and expansion give
\[\Abs{\sum_{j \in [k]} A_j^\star \gamma^j - \half \phi(\gamma)} \lesssim_k \gamma^{k + 1}.\]
Since this holds for every sufficiently small $\gamma > 0$, the Taylor expansion of $\phi$ yields $A_j^\star = \half\alpha_j$ for all $j \in [k]$. Thus all subsequential limits coincide, so $A_j(d) \to \half\alpha_j > 0$. Recalling $\va_j = d^{j + 1}A_j(d)$ proves the claim.
\end{proof}


\section{Additional Details and Results for Gaussian Experiments}
\label{app:gaussian_exp_extra}
This appendix provides implementation details, covariance constructions, and additional results for the Gaussian experiments in Section~\ref{sec:gaussian_exp}. All methods share the same $N$-sample pool within each population pair and trial. Figures~\ref{fig:controlled-spiked-bulk-error}--\ref{fig:stress-ill-conditioned} use the same sampling and aggregation conventions as Figure~\ref{fig:controlled-random-error}.

\paragraph{Estimator implementation.}
OLS$_k$, with $k\in\{1,2,3\}$, fits a degree-$k$ polynomial in $\frac 1n$ to the same $J=15$ plug-in scores and reports the intercept at $\frac 1n =0$; $k=1$ corresponds to $\mathrm{FID}_{\infty}$~\cite{chong2019effectively}. The sample counts are uniformly spaced in the window $[n_{\min}, N]$ up to integer rounding, where we use $n_{\min}=5000$:
\[
    n_j=n_{\min}+\left\lfloor\frac{(j-1)(N-n_{\min})}{14}\right\rfloor,
    \qquad j=1,\ldots,15.
\]

For RTD, we choose the truncation order from both $d$ and $m$ using
\begin{equation}
\label{eq:experimental-truncation}
    \eps_{\mathrm{design}}=\min\{1,\sqrt{d/m}\},
    \qquad
    L=\max\left\{2,
        \left\lceil\log_4(d/\eps_{\mathrm{design}})\right\rceil
    \right\}.
\end{equation}

\paragraph{Covariance families.}
We consider five Gaussian covariance families with different spectra and relative eigenspace geometries. For the first three families, the two eigenbases are drawn independently from the Haar measure on the orthogonal group \cite{Mezz06}. For the last two, we randomize the relative rotation angles within two-dimensional planes while keeping the population distance fixed at $D=5$. The five constructions are as follows.
\begin{itemize}
\item \emph{Haar log-uniform} (Figure~\ref{fig:controlled-random-error}).
For each covariance, the eigenvalues are sampled independently according to
\[
\log \lambda_i \sim \operatorname{Unif}(\log 0.05,\log 1).
\]

\item \emph{Spiked-bulk} (Figure~\ref{fig:controlled-spiked-bulk-error}).
For each covariance, $10\%$ of the eigenvalues are sampled uniformly from $[0.7,1]$, while the remaining $90\%$ are sampled uniformly from $[0.1,0.3]$.

\item \emph{Rotated Toeplitz} (Figure~\ref{fig:controlled-toeplitz-error}).
The two covariances have the spectra of the Toeplitz matrices
\[
[0.4^{|i-j|}]_{i,j=1}^{d}
\qquad\text{and}\qquad
[0.7^{|i-j|}]_{i,j=1}^{d},
\]
respectively, with each spectrum rescaled so that its largest eigenvalue is one.

\item \emph{Rank-deficient} (Figure~\ref{fig:stress-rank}).
For $q\in\{0.25,0.50,0.75\}$ and $r=qd$, both covariances are rank-$r$ orthogonal projectors: each has $r$ eigenvalues equal to one and $d-r$ eigenvalues equal to zero. Their supports differ by principal angles $\theta_i$ in $s=\min(r,d-r)$ two-dimensional planes, chosen so that
\begin{equation*}
2\sum_{i=1}^{s}(1-\cos\theta_i)=5.
\end{equation*}

\item \emph{Ill-conditioned} (Figure~\ref{fig:stress-ill-conditioned}).
For $\kappa\in\{10^2,10^4,10^6\}$, let $\lambda=\kappa^{-1}$. For even $d$, both covariances have $\frac d2$ eigenvalues equal to one and $\frac d2$ eigenvalues equal to $\lambda$. The covariance $\ma$ consists of repeated blocks $\operatorname{diag}(1,\lambda)$, while $\mb$ is obtained by rotating each block through an angle $\theta_i$, chosen so that
\begin{equation*}
2\sum_{i=1}^{d/2}\left[
(1+\lambda)
-\sqrt{(1+\lambda)^2-(1-\lambda)^2\sin^2\theta_i}
\right]=5.
\end{equation*}
These pairs are full rank and noncommuting, with condition number exactly $\kappa$.
\end{itemize}

\newcommand{\gaussianFamilyCrop}[2]{%
    \begingroup
    \setbox0=\hbox{\includegraphics{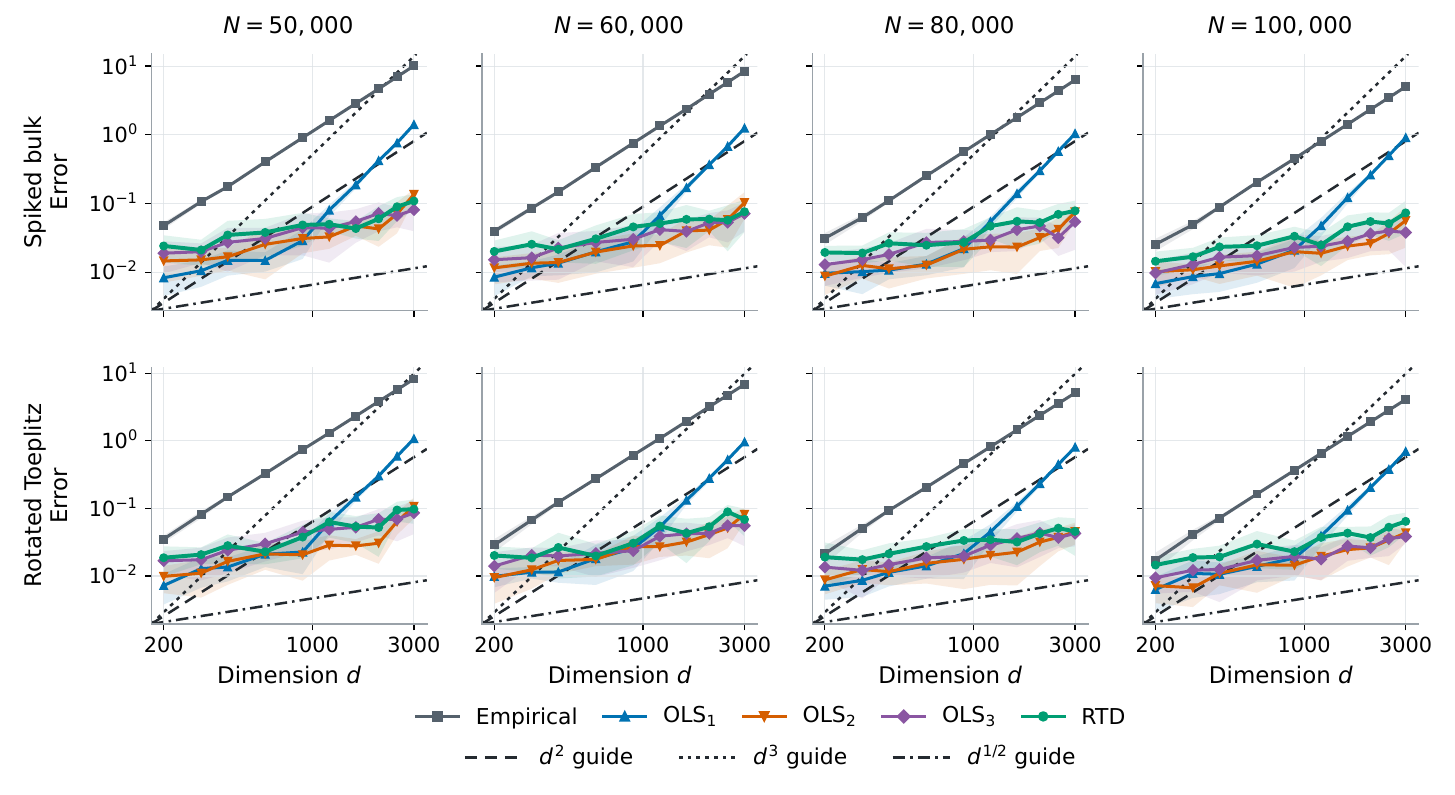}}%
    \edef\gaussianFamilyGraphic{\noexpand\includegraphics[width=\noexpand\textwidth,trim=0pt \the\dimexpr\ht0*#1/1000\relax\space 0pt \the\dimexpr\ht0*#2/1000\relax,clip]{figures/gaussian_dimension_other_families.pdf}}%
    \gaussianFamilyGraphic
    \endgroup
}

\begin{figure}[!ht]
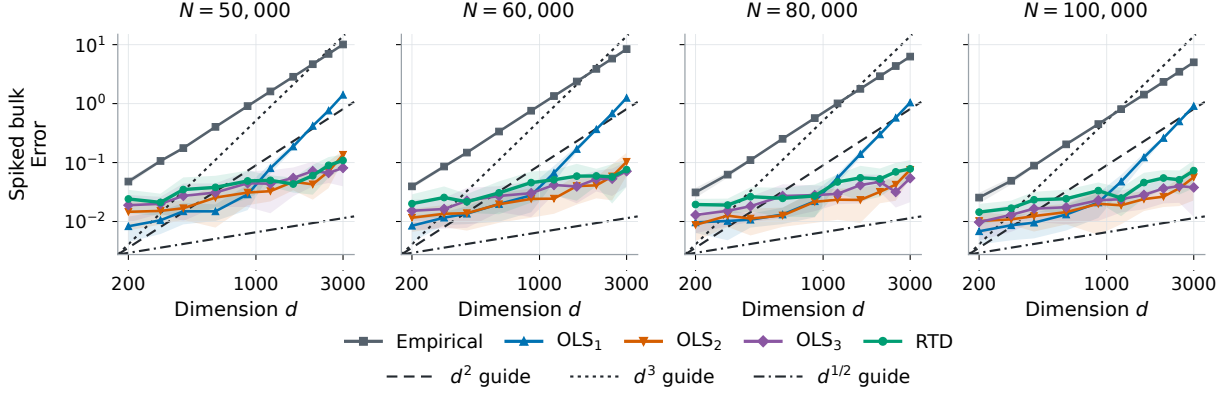

    \centering
    \gaussianFamilyCrop{585}{0}\par\nointerlineskip
    \gaussianFamilyCrop{130}{793}\par\nointerlineskip
    \gaussianFamilyCrop{20}{875}
    \caption{Estimation error versus dimension for the spiked-bulk family at four fixed sample budgets.}
    \label{fig:controlled-spiked-bulk-error}
    \label{fig:controlled-other-random-error}
\end{figure}

\begin{figure}[!ht]
    \centering
    \gaussianFamilyCrop{945}{0}\par\nointerlineskip
    \gaussianFamilyCrop{130}{445}\par\nointerlineskip
    \gaussianFamilyCrop{20}{875}
    \caption{Estimation error versus dimension for the rotated Toeplitz family at fixed sample budgets.}
    \label{fig:controlled-toeplitz-error}
\end{figure}

\begin{figure*}[!ht]
    \centering
    \includegraphics[width=\textwidth]{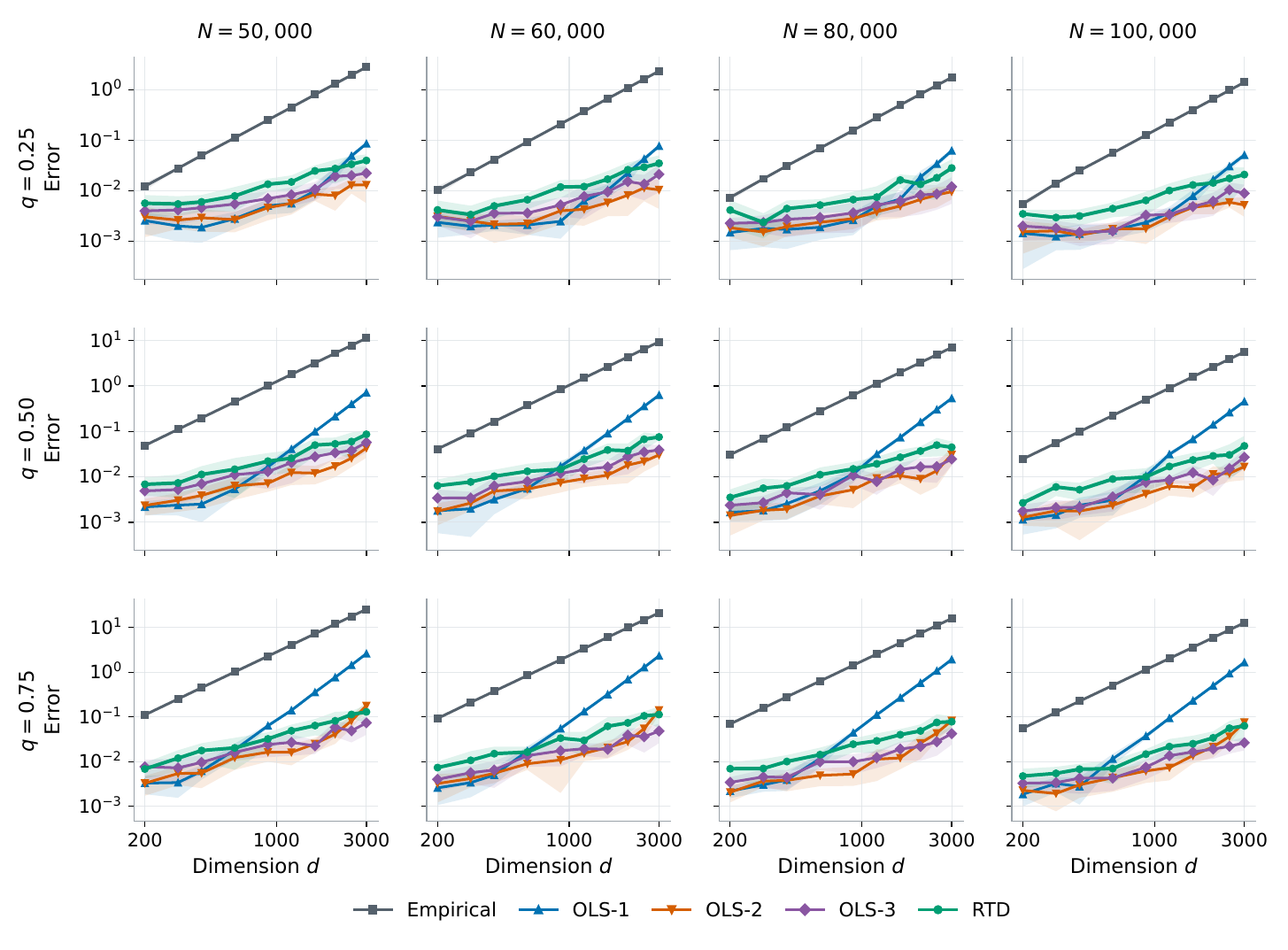}
    \caption{Estimation error for singular covariances with $D=5$.}
    \label{fig:stress-rank}
\end{figure*}

\begin{figure*}[!ht]
    \centering
    \includegraphics[width=\textwidth]{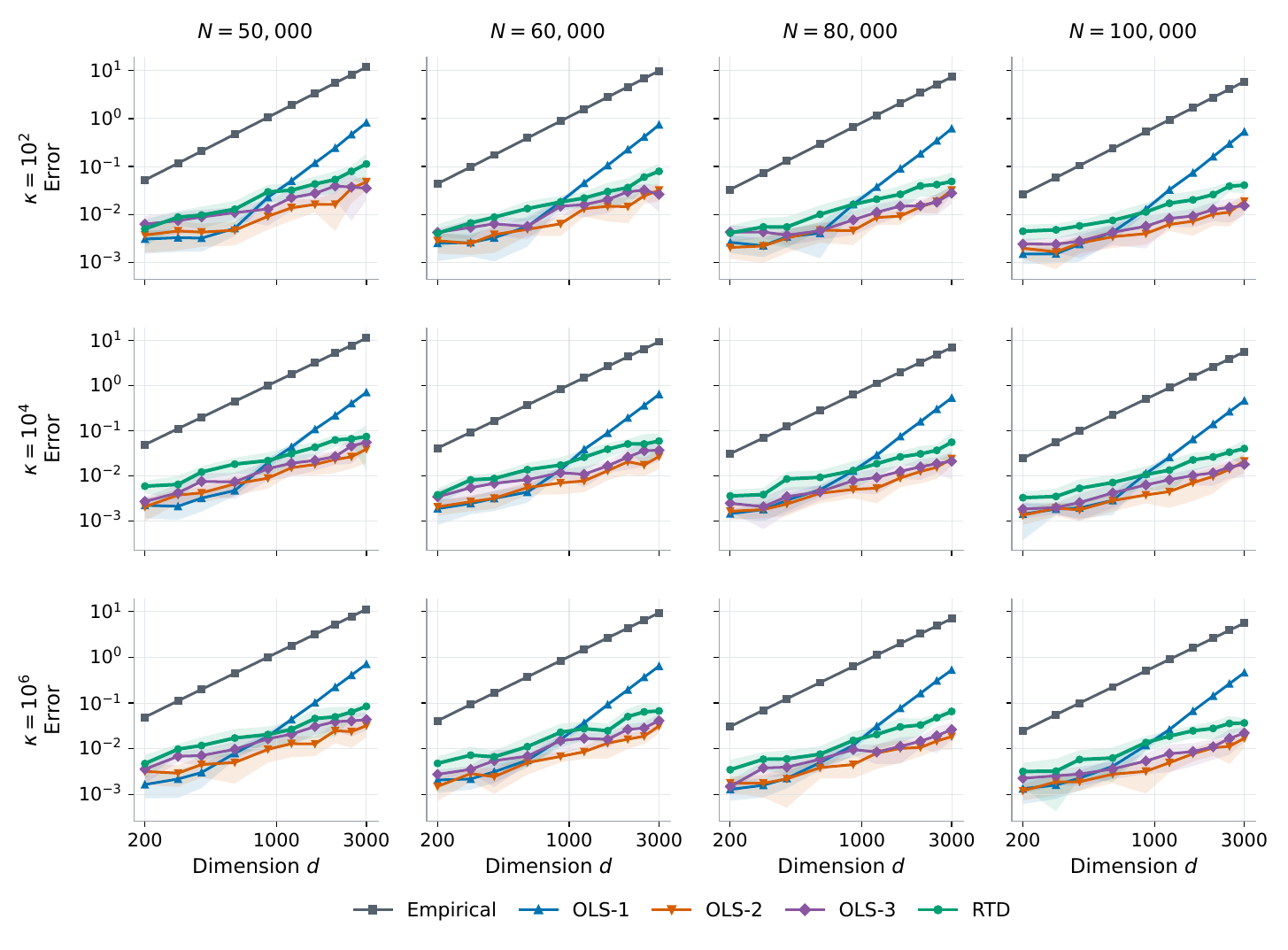}
    \caption{Estimation error for ill-conditioned covariances with $D=5$.}
    \label{fig:stress-ill-conditioned}
\end{figure*}

\section{Implementation and Reference Checks for the Experiments}
\label{app:experimental-details}

\subsection{Computational Environment}\label{ssec:computational-environment}
The experiments use GPU compute nodes equipped with NVIDIA GH200 Grace Hopper superchips and 72-core Arm Grace CPUs. The covariance and Taylor-coefficient matrix computations use CuPy float64; the small extrapolation regressions use NumPy float64.

\subsection{ImageNet experimental setup}
\label{app:imagenet-setup}

\paragraph{Model and embedding configurations.}
Table~\ref{tab:imagenet-model-detail} lists the seven generator checkpoints, their native image resolutions, and their full-300K RTD reference proxies. The checkpoints span GAN, diffusion, and autoregressive models. For each checkpoint, we use Inception-V3 ($d=2{,}048$), DINOv2 ViT-L/14 ($d=1{,}024$), and CLIP ViT-L/14 ($d=768$). Inception-V3 provides the representation used in classical FID, while DINOv2 and CLIP provide alternative feature embeddings.

\begin{table}[h]
\centering
\small
\caption{Model checkpoints and full-300K RTD proxy values.}
\label{tab:imagenet-model-detail}
\setlength{\tabcolsep}{4pt}
\begin{tabular}{@{}>{\raggedright\arraybackslash}p{0.38\linewidth}rrrrr@{}}
\toprule
Model & Resolution & \shortstack[r]{Inception\\ $Q_c$} & \shortstack[r]{DINOv2\\ $Q_c$} & \shortstack[r]{CLIP\\ $Q_c$} & \shortstack[r]{Published\\FID} \\
\midrule
StyleGAN-XL~\cite{sauer2022stylegan} & 64 & 1.936 & 206.696 & 9.852 & 1.51 \\
DDO/EDM2-S~\cite{zheng25ddo} & 64 & 0.333 & 63.356 & 3.270 & 0.97 \\
StyleGAN-XL~\cite{sauer2022stylegan} & 256 & 1.581 & 202.413 & 27.528 & 2.30 \\
VAR-d30~\cite{tian2024visual} & 256 & 1.127 & 65.289 & 19.089 & 1.73 \\
StyleGAN-XL~\cite{sauer2022stylegan} & 512 & 1.748 & 210.961 & 30.260 & 2.41 \\
DDO/EDM2-L~\cite{zheng25ddo} & 512 & 0.641 & 39.898 & 9.402 & 1.21 \\
VAR-d36~\cite{tian2024visual} & 512 & 1.957 & 53.301 & 13.780 & 2.63 \\
\bottomrule
\end{tabular}
\end{table}
\paragraph{Details on estimators.}
In the ImageNet experiments, we conduct some preprocessing procedure for raw feature embedding: given generated features $\vX_1,\ldots,\vX_N$, we compute the full-sample mean $\bar{\vX}_N$ and the Helmert contrasts
\[
    \bar{\vX}_N=\frac1N\sum_{j=1}^N\vX_j,
    \qquad
    \vz_i=\sqrt{\frac{i}{i+1}}
    \left(\frac1i\sum_{j=1}^i\vX_j-\vX_{i+1}\right),
    \quad i=1,\ldots,N-1.
\]
Under the Gaussian model, the contrasts are independent $\calN(\0_d,\mb)$ observations and are independent of $\bar{\vX}_N$. We use the first $m_0=\lfloor0.6N\rfloor$ contrasts to form the pilot covariance
\(
    \tmb=\frac1{m_0}\sum_{i=1}^{m_0}\vz_i\vz_i^\top,
\)
and the remaining $m=N-1-m_0$ contrasts to compute the covariance estimate $\widehat D_{\mathrm{BW},L}$ in \eqref{eq:cross-fit-bures-estimator}, with $L$ chosen by \eqref{eq:experimental-truncation} using this value of $m$.

\paragraph{Reported statistics.}
For the ablation study, we draw $T=10$ random 50K-sample trials for each (model, embedding) pair and compute the absolute center error $|\overline D_c-Q_c|$, sample SD, and RMSE, where $\overline D_c$ is the mean estimate over the ten trials. For each embedding, we report the mean and sample SD of each statistic across the seven generators. The trial SD measures variability across sampling trials, whereas the SD accompanying each aggregate statistic measures variability across generators. The held-out error statistic is defined below.

\paragraph{Held-out comparison.}
We evaluate Empirical, FID$_\infty$, VALE$_1$, VALE$_2$, VALE$_3$, and RTD on the five held-out 50K blocks per checkpoint. The VALE$_k$ configurations are selected on the development pool as specified in Table~\ref{tab:imagenet-optimal-schedule}. FID$_\infty$ follows~\cite{chong2019effectively} without quasi-Monte Carlo, using first-order OLS on 15 points uniformly spaced in $n$ from 5K to 50K. For a reference proxy $R_c\in\{Q_c^{\mathrm{plug}},Q_c^\infty,Q_c\}$, we compute $e_c(R_c)=\operatorname{median}_{t=1,\ldots,5}|\widehat D_{c,t}-R_c|$ and report its mean $\pm$ sample SD across the seven checkpoints within each embedding. Table~\ref{tab:imagenet-proxy-median-error} reports all three proxy choices; the reduced-budget experiments in Tables~\ref{tab:imagenet-fd-sample-budget-heldout} and~\ref{tab:imagenet-sample-budget-additional} use $R_c=Q_c$.

\subsection{Consistency and sensitivity of the ImageNet reference proxy}
\label{app:imagenet-reference-sensitivity}

This subsection provides the complete results addressing the reference-proxy sensitivity discussed in Section~\ref{sec:imagenet_exp}. Figures~\ref{fig:imagenet-proxy-consistency-inception}--\ref{fig:imagenet-proxy-consistency-clip} supplement Figure~\ref{fig:imagenet-proxy-example} with results for the remaining six generator configurations across all three embeddings.

Table~\ref{tab:imagenet-proxy-agreement} quantifies the change in reference-proxy values when replacing RTD with another estimator, reporting mean absolute differences across seven generator configurations. At 300K, FID$_\infty$ and RTD give similar proxy values, whereas the empirical proxy differs substantially. However, we believe this larger discrepancy is because of the substantial finite-sample bias of the empirical estimator.

\begin{table}[!ht]
\centering
\small
\caption{Mean absolute differences from RTD at the same sample budget across seven generators.}
\label{tab:imagenet-proxy-agreement}
\begin{tabular}{@{}lcrrr@{}}
\toprule
Estimator & $N$ & Inception & DINOv2 & CLIP \\
\midrule
FID$_\infty$ & 30K & 0.0439 & 0.7538 & 0.0278 \\
FID$_\infty$ & 300K & 0.0102 & 0.1334 & 0.0082 \\
Empirical & 300K & 0.1178 & 1.4370 & 0.0405 \\
\bottomrule
\end{tabular}
\end{table}

\begin{figure}[!htbp]
\centering
\includegraphics[width=\textwidth]{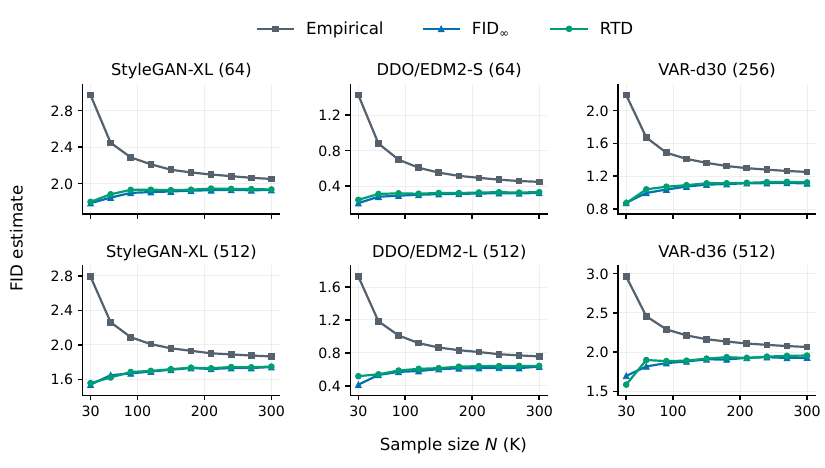}
\caption{Estimates across sample budgets for the six additional generators using Inception.}
\label{fig:imagenet-proxy-consistency-inception}
\end{figure}

\begin{figure}[!htbp]
\centering
\includegraphics[width=\textwidth]{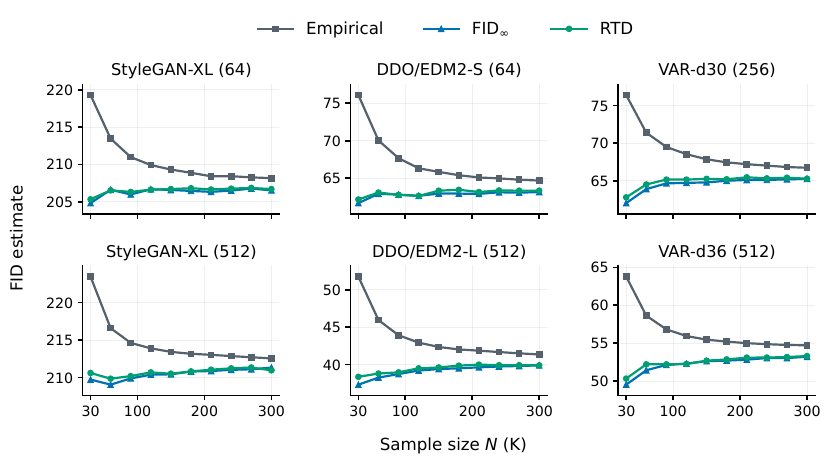}
\caption{Estimates across sample budgets for the six additional generators using DINOv2.}
\label{fig:imagenet-proxy-consistency-dinov2}
\end{figure}

\begin{figure}[!htbp]
\centering
\includegraphics[width=\textwidth]{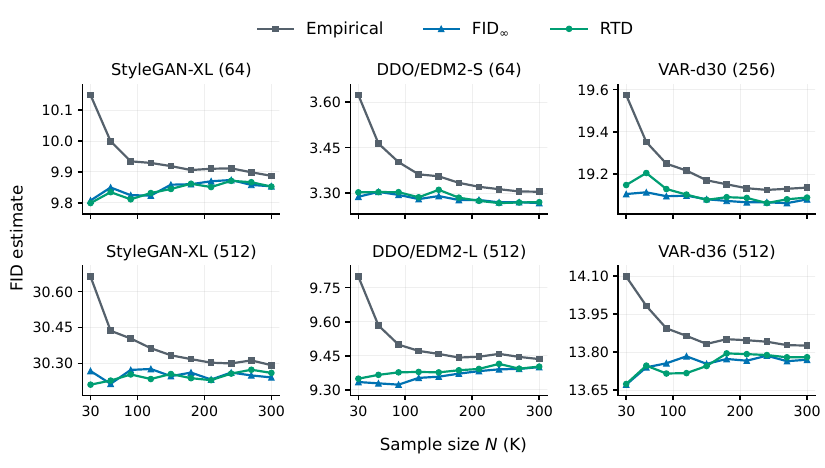}
\caption{Estimates across sample budgets for the six additional generators using CLIP.}
\label{fig:imagenet-proxy-consistency-clip}
\end{figure}

\subsection{Additional extrapolation ablations}
\label{app:imagenet-ablation}
In this subsection, we give a full grid record of the ablation result in Table~\ref{tab:imagenet-ablation-summary}. In particular, we examine the effect of polynomial order, variance-aware weighting, and regression-point assignment.

\paragraph{Polynomial order \(k\).}
We first compare OLS$_k$ at orders $k=1,\ldots,6$, using 15 points uniformly spaced in $n$ from 5K to 50K. Figure~\ref{fig:imagenet-order-sweep} summarizes the results. SD increases with order in all three embeddings, while center error can decrease at low orders but increases at high orders. Beyond $k=4$, SD and RMSE increase sharply. We therefore restrict subsequent comparisons to $k=1,\ldots,4$.

\begin{figure}[h]
\centering
\includegraphics[width=\textwidth]{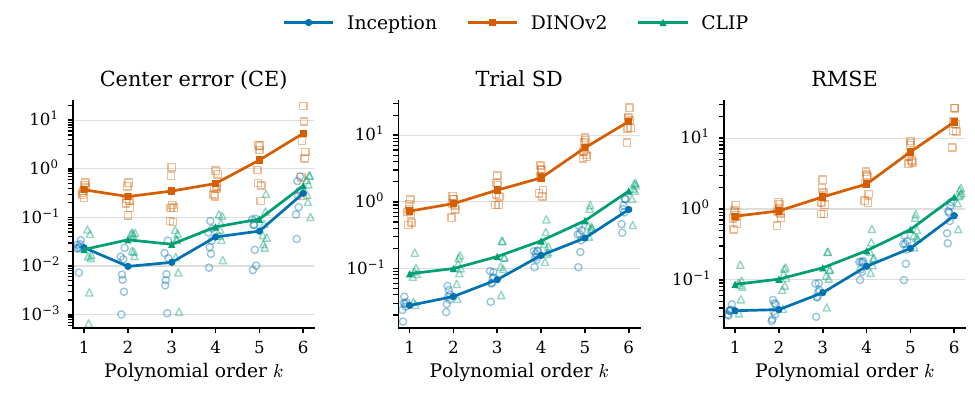}
\caption{OLS$_k$ order sweep. Hollow markers show the case-level statistics for seven generators; solid lines show their means. Markers are slightly offset horizontally for visibility.}
\label{fig:imagenet-order-sweep}
\end{figure}

\textbf{Variance-aware regression.}
We next compare OLS$_k$ and VALE$_k$ at matched orders $k=1,\ldots,4$ on the same 15-point uniform-in-$n$ grid. Table~\ref{tab:imagenet-VA-OLS} shows that VALE$_k$ reduces both mean trial SD and mean RMSE at all four orders in each embedding, and has lower RMSE than OLS$_k$ in almost all individual cases.
We therefore use VALE$_k$ for the selected extrapolation rules in subsequent comparisons, retaining FID$_\infty$ as the baseline.
\begin{table}[!htbp]
\centering
\footnotesize
\caption{Matched OLS$_k$--VALE$_k$ comparison. RMSE drop is the mean casewise percentage reduction; win rate counts cases with lower RMSE.}
\label{tab:imagenet-VA-OLS}
\setlength{\tabcolsep}{2pt}
\begin{tabular}{@{}lcrrrrrr@{}}
\toprule
& & \multicolumn{2}{c}{SD} & \multicolumn{2}{c}{RMSE} & & \\
\cmidrule(lr){3-4}\cmidrule(lr){5-6}
Embedding & $k$ & OLS$_k$ & VALE$_k$ & OLS$_k$ & VALE$_k$ & \shortstack[r]{RMSE drop\\(\%)} & \shortstack[r]{Win\\rate} \\
\midrule
Inception & 1 & $0.0277\pm0.0068$ & $0.0233\pm0.0065$ & $0.0366\pm0.0044$ & $0.0289\pm0.0048$ & 21.21 & 7/7 \\
 & 2 & $0.0378\pm0.0109$ & $0.0289\pm0.0092$ & $0.0380\pm0.0101$ & $0.0291\pm0.0081$ & 23.48 & 7/7 \\
 & 3 & $0.0678\pm0.0202$ & $0.0469\pm0.0115$ & $0.0659\pm0.0206$ & $0.0458\pm0.0115$ & 28.62 & 7/7 \\
 & 4 & $0.1569\pm0.0301$ & $0.1224\pm0.0227$ & $0.1553\pm0.0328$ & $0.1212\pm0.0240$ & 21.71 & 7/7 \\
\midrule
DINOv2 & 1 & $0.7256\pm0.2274$ & $0.6771\pm0.1986$ & $0.7901\pm0.2031$ & $0.6994\pm0.1968$ & 11.34 & 6/7 \\
 & 2 & $0.9416\pm0.2215$ & $0.7628\pm0.1987$ & $0.9390\pm0.2329$ & $0.7707\pm0.2157$ & 18.29 & 7/7 \\
 & 3 & $1.5010\pm0.5945$ & $1.1414\pm0.4880$ & $1.4816\pm0.6439$ & $1.1388\pm0.5456$ & 23.87 & 7/7 \\
 & 4 & $2.2905\pm0.9358$ & $1.7048\pm0.7726$ & $2.2480\pm0.8685$ & $1.6848\pm0.7001$ & 25.59 & 7/7 \\
\midrule
CLIP & 1 & $0.0830\pm0.0485$ & $0.0757\pm0.0425$ & $0.0863\pm0.0409$ & $0.0790\pm0.0363$ & 8.24 & 6/7 \\
 & 2 & $0.0995\pm0.0440$ & $0.0805\pm0.0361$ & $0.1021\pm0.0399$ & $0.0838\pm0.0323$ & 17.64 & 7/7 \\
 & 3 & $0.1507\pm0.0821$ & $0.1163\pm0.0656$ & $0.1477\pm0.0760$ & $0.1161\pm0.0611$ & 21.61 & 7/7 \\
 & 4 & $0.2590\pm0.1455$ & $0.2006\pm0.1146$ & $0.2575\pm0.1355$ & $0.1989\pm0.1074$ & 23.11 & 7/7 \\
\bottomrule
\end{tabular}
\end{table}

\paragraph{Regression-point assignment.}
For VALE$_k$ at orders $k=1,\ldots,4$, we compare uniform spacing in $n$, uniform spacing in $\frac 1n$, and Chebyshev--Lobatto spacing in $\frac 1n$, with point counts $\ell\in\{k+1,10,15,20\}$ and endpoints fixed at 5K and 50K. Table~\ref{tab:imagenet-optimal-schedule} summarizes the configurations with lowest mean RMSE for each embedding and each order $k=1,2,3$ used in the held-out comparisons; Tables~\ref{tab:imagenet-va-full-grid}--\ref{tab:imagenet-va-full-grid-clip} report the complete results. Duplicated numbers may appear because the assignments share the same schedule under that case. 
\subsection{Additional sample-budget experiments}
\label{app:imagenet-sample-budget}

Under the same setup as the Inception experiment in Table~\ref{tab:imagenet-fd-sample-budget-heldout}, Table~\ref{tab:imagenet-sample-budget-additional} reports the results for DINOv2 and CLIP embeddings. For DINOv2, RTD has the lowest mean error at 20K--40K, but its error increases sharply at 10K. The same phenomenon also appears for Inception embeddings. This can be explained by the sensitivity to the pilot estimate: when $m_0 \to d$, the pilot may fail to stabilize high-order expansion terms. For the lower-dimensional CLIP embedding, VALE$_1$ or FID$_\infty$ has the lowest mean error at 10K--30K, while RTD performs best at 40K and 50K.
\begin{table}[!ht]
\centering
\small
\caption{Selected VALE$_k$ configurations.}
\label{tab:imagenet-optimal-schedule}
\begin{tabular}{@{}lclrr@{}}
\toprule
Embedding & Order $k$ & Schedule & Points $\ell$ & RMSE \\
\midrule
Inception & 1 & Endpoints & 2 & $0.0275\pm0.0045$ \\
 & 2 & Chebyshev in $\frac 1n$ & 15 & $0.0254\pm0.0079$ \\
 & 3 & Chebyshev in $\frac 1n$ & 20 & $0.0393\pm0.0111$ \\
\midrule
DINOv2 & 1 & Endpoints & 2 & $0.6675\pm0.1947$ \\
 & 2 & Chebyshev in $\frac 1n$ & 15 & $0.6871\pm0.1925$ \\
 & 3 & Chebyshev in $\frac 1n$ & 20 & $0.9491\pm0.3099$ \\
\midrule
CLIP & 1 & Uniform in $\frac 1n$ & 15 & $0.0742\pm0.0360$ \\
 & 2 & Uniform in $\frac 1n$ & 20 & $0.0791\pm0.0350$ \\
 & 3 & Uniform in $n$ & 20 & $0.1101\pm0.0376$ \\
\bottomrule
\end{tabular}
\end{table}

\begin{table}[h]
\centering
\small
\caption{FID-VALE$_k$ regression-point ablation for Inception.}
\label{tab:imagenet-va-full-grid}
\begin{tabular}{@{}ccrrr@{}}
\toprule
Order $k$ & Points $\ell$ & Uniform $n$ & Uniform $\frac 1n$ & Chebyshev $\frac 1n$ \\
\midrule
1 & 2 & $\mathbf{0.0275\pm0.0045}$ & $\mathbf{0.0275\pm0.0045}$ & $\mathbf{0.0275\pm0.0045}$ \\
1 & 10 & $0.0291\pm0.0060$ & $0.0324\pm0.0049$ & $0.0317\pm0.0058$ \\
1 & 15 & $0.0289\pm0.0048$ & $0.0336\pm0.0078$ & $0.0327\pm0.0042$ \\
1 & 20 & $0.0279\pm0.0057$ & $0.0340\pm0.0065$ & $0.0309\pm0.0035$ \\
\addlinespace[3pt]
2 & 3 & $0.0415\pm0.0171$ & $0.0371\pm0.0083$ & $0.0371\pm0.0083$ \\
2 & 10 & $0.0324\pm0.0115$ & $0.0307\pm0.0090$ & $0.0323\pm0.0115$ \\
2 & 15 & $0.0291\pm0.0081$ & $0.0310\pm0.0091$ & $\mathbf{0.0254\pm0.0079}$ \\
2 & 20 & $0.0280\pm0.0088$ & $0.0289\pm0.0076$ & $0.0287\pm0.0095$ \\
\addlinespace[3pt]
3 & 4 & $0.1149\pm0.0361$ & $0.0893\pm0.0241$ & $0.0697\pm0.0217$ \\
3 & 10 & $0.0573\pm0.0108$ & $0.0519\pm0.0115$ & $0.0553\pm0.0161$ \\
3 & 15 & $0.0458\pm0.0115$ & $0.0421\pm0.0113$ & $0.0511\pm0.0163$ \\
3 & 20 & $0.0399\pm0.0150$ & $0.0476\pm0.0127$ & $\mathbf{0.0393\pm0.0111}$ \\
\addlinespace[3pt]
4 & 5 & $0.3399\pm0.0834$ & $0.2297\pm0.0678$ & $0.1106\pm0.0582$ \\
4 & 10 & $0.1246\pm0.0231$ & $0.1113\pm0.0269$ & $0.0966\pm0.0267$ \\
4 & 15 & $0.1212\pm0.0240$ & $0.0771\pm0.0234$ & $0.0999\pm0.0310$ \\
4 & 20 & $\mathbf{0.0645\pm0.0202}$ & $0.0694\pm0.0195$ & $0.0707\pm0.0240$ \\
\bottomrule
\end{tabular}
\end{table}

\begin{table}[h]
\centering
\small
\caption{VALE$_k$ regression-point ablation for DINOv2.}
\label{tab:imagenet-va-full-grid-dinov2}
\begin{tabular}{@{}ccrrr@{}}
\toprule
Order $k$ & Points $\ell$ & Uniform $n$ & Uniform $\frac 1n$ & Chebyshev $\frac 1n$ \\
\midrule
1 & 2 & $\mathbf{0.6675\pm0.1947}$ & $\mathbf{0.6675\pm0.1947}$ & $\mathbf{0.6675\pm0.1947}$ \\
1 & 10 & $0.7062\pm0.2072$ & $0.7179\pm0.2516$ & $0.7403\pm0.2734$ \\
1 & 15 & $0.6994\pm0.1968$ & $0.7498\pm0.1934$ & $0.7276\pm0.2032$ \\
1 & 20 & $0.7186\pm0.2074$ & $0.7277\pm0.1825$ & $0.6941\pm0.1803$ \\
\addlinespace[3pt]
2 & 3 & $0.9551\pm0.4135$ & $0.9665\pm0.2765$ & $0.9665\pm0.2765$ \\
2 & 10 & $0.8848\pm0.1940$ & $0.8011\pm0.2110$ & $0.8931\pm0.2122$ \\
2 & 15 & $0.7707\pm0.2157$ & $0.8699\pm0.2057$ & $\mathbf{0.6871\pm0.1925}$ \\
2 & 20 & $0.8050\pm0.2533$ & $0.7990\pm0.3580$ & $0.6903\pm0.2252$ \\
\addlinespace[3pt]
3 & 4 & $2.3145\pm0.5352$ & $1.8733\pm0.5067$ & $1.3758\pm0.4335$ \\
3 & 10 & $1.0432\pm0.2702$ & $1.1186\pm0.4467$ & $1.1959\pm0.5001$ \\
3 & 15 & $1.1388\pm0.5456$ & $1.0401\pm0.3823$ & $0.9846\pm0.2740$ \\
3 & 20 & $0.9920\pm0.4023$ & $0.9774\pm0.4040$ & $\mathbf{0.9491\pm0.3099}$ \\
\addlinespace[3pt]
4 & 5 & $8.8701\pm1.6289$ & $4.5986\pm1.9862$ & $2.4009\pm0.8418$ \\
4 & 10 & $2.3577\pm0.5683$ & $1.9596\pm0.3583$ & $2.2620\pm0.6918$ \\
4 & 15 & $1.6848\pm0.7001$ & $1.6664\pm0.5619$ & $1.7821\pm0.6044$ \\
4 & 20 & $1.7618\pm0.4114$ & $\mathbf{1.3593\pm0.2658}$ & $1.7541\pm0.5591$ \\
\bottomrule
\end{tabular}
\end{table}
\clearpage
\begin{table}[h]
\centering
\small
\caption{VALE$_k$ regression-point ablation for CLIP.}
\label{tab:imagenet-va-full-grid-clip}
\begin{tabular}{@{}ccrrr@{}}
\toprule
Order $k$ & Points $\ell$ & Uniform $n$ & Uniform $\frac 1 n$ & Chebyshev $\frac 1 n$ \\
\midrule
1 & 2 & $0.0772\pm0.0333$ & $0.0772\pm0.0333$ & $0.0772\pm0.0333$ \\
1 & 10 & $0.0805\pm0.0309$ & $0.0785\pm0.0382$ & $0.0798\pm0.0333$ \\
1 & 15 & $0.0790\pm0.0363$ & $\mathbf{0.0742\pm0.0360}$ & $0.0786\pm0.0318$ \\
1 & 20 & $0.0790\pm0.0362$ & $0.0797\pm0.0373$ & $0.0780\pm0.0349$ \\
\addlinespace[3pt]
2 & 3 & $0.1129\pm0.0500$ & $0.0983\pm0.0361$ & $0.0983\pm0.0361$ \\
2 & 10 & $0.1025\pm0.0457$ & $0.0927\pm0.0329$ & $0.0918\pm0.0306$ \\
2 & 15 & $0.0838\pm0.0323$ & $0.0936\pm0.0426$ & $0.0829\pm0.0268$ \\
2 & 20 & $0.0857\pm0.0325$ & $\mathbf{0.0791\pm0.0350}$ & $0.0817\pm0.0367$ \\
\addlinespace[3pt]
3 & 4 & $0.3055\pm0.1352$ & $0.2116\pm0.0934$ & $0.1450\pm0.0703$ \\
3 & 10 & $0.1244\pm0.0826$ & $0.1147\pm0.0402$ & $0.1207\pm0.0593$ \\
3 & 15 & $0.1161\pm0.0611$ & $0.1169\pm0.0435$ & $0.1154\pm0.0460$ \\
3 & 20 & $\mathbf{0.1101\pm0.0376}$ & $0.1171\pm0.0477$ & $0.1125\pm0.0603$ \\
\addlinespace[3pt]
4 & 5 & $1.0286\pm0.5131$ & $0.4177\pm0.1928$ & $0.2422\pm0.0591$ \\
4 & 10 & $0.3041\pm0.1450$ & $0.2143\pm0.1208$ & $0.2046\pm0.1186$ \\
4 & 15 & $0.1989\pm0.1074$ & $0.1918\pm0.1220$ & $0.1886\pm0.0864$ \\
4 & 20 & $0.2069\pm0.1021$ & $\mathbf{0.1744\pm0.1053}$ & $0.1841\pm0.1130$ \\
\bottomrule
\end{tabular}
\end{table}

\begin{table}[!ht]
\centering
\footnotesize
\setlength{\tabcolsep}{2pt}
\caption{Sample-budget sensitivity for DINOv2 and CLIP, using the same error statistic and RTD reference proxy as Table~\ref{tab:imagenet-fd-sample-budget-heldout}. Bold entries have the lowest mean within each embedding and budget.}
\label{tab:imagenet-sample-budget-additional}
\begin{tabular}{@{}rrrrrrr@{}}
\toprule
$N$ & Empirical & FID$_\infty$ & VALE$_1$ & VALE$_2$ & VALE$_3$ & RTD \\
\midrule
\multicolumn{7}{l}{\textbf{DINOv2}} \\
10K & $40.557\pm1.051$ & $2.722\pm1.084$ & $\mathbf{2.273}\pm1.066$ & $4.705\pm1.863$ & $20.916\pm11.920$ & $20.571\pm4.880$ \\
20K & $19.545\pm0.540$ & $1.675\pm0.469$ & $1.482\pm0.475$ & $1.872\pm1.087$ & $3.934\pm1.491$ & $\mathbf{1.125}\pm0.397$ \\
30K & $12.574\pm0.245$ & $1.625\pm0.605$ & $1.581\pm0.503$ & $1.425\pm0.513$ & $1.761\pm0.995$ & $\mathbf{1.259}\pm0.399$ \\
40K & $9.067\pm0.219$ & $1.727\pm0.187$ & $1.661\pm0.245$ & $1.332\pm0.445$ & $1.345\pm0.322$ & $\mathbf{1.295}\pm0.315$ \\
50K & $6.943\pm0.238$ & $1.582\pm0.301$ & $1.549\pm0.281$ & $\mathbf{1.257}\pm0.399$ & $1.468\pm0.591$ & $1.357\pm0.354$ \\
\midrule
\multicolumn{7}{l}{\textbf{CLIP}} \\
10K & $1.214\pm0.217$ & $0.128\pm0.093$ & $\mathbf{0.100}\pm0.093$ & $0.243\pm0.141$ & $1.526\pm0.747$ & $0.128\pm0.114$ \\
20K & $0.612\pm0.115$ & $0.084\pm0.061$ & $\mathbf{0.063}\pm0.028$ & $0.156\pm0.103$ & $0.373\pm0.255$ & $0.102\pm0.064$ \\
30K & $0.398\pm0.089$ & $\mathbf{0.067}\pm0.058$ & $0.071\pm0.036$ & $0.115\pm0.080$ & $0.136\pm0.042$ & $0.096\pm0.052$ \\
40K & $0.286\pm0.083$ & $0.069\pm0.041$ & $0.066\pm0.032$ & $0.094\pm0.066$ & $0.097\pm0.063$ & $\mathbf{0.060}\pm0.017$ \\
50K & $0.226\pm0.062$ & $0.065\pm0.042$ & $0.065\pm0.038$ & $0.066\pm0.034$ & $0.073\pm0.037$ & $\mathbf{0.055}\pm0.030$ \\
\bottomrule
\end{tabular}
\end{table}

\end{document}